\documentclass{article}

    \PassOptionsToPackage{numbers, compress}{natbib}

\usepackage[final]{neurips_2025}

\usepackage[utf8]{inputenc} 
\usepackage[T1]{fontenc}    
\usepackage{hyperref}       
\usepackage{url}            
\usepackage{booktabs}       
\usepackage{amsfonts}       
\usepackage{nicefrac}       
\usepackage{microtype}      

\usepackage{multirow}
\usepackage{makecell}
\usepackage[table,xcdraw]{xcolor}

\usepackage{amsmath}
\usepackage{amssymb}
\usepackage{mathtools}
\usepackage{amsthm}
\usepackage{algorithm}
\usepackage{enumitem}
\setlist{nosep}
\usepackage{subfigure}

\theoremstyle{plain}
\newtheorem{theorem}{Theorem}[section]
\newtheorem{proposition}[theorem]{Proposition}

\theoremstyle{definition}
\newtheorem{definition}[theorem]{Definition}

\theoremstyle{remark}
\newtheorem{remark}[theorem]{Remark}
\usepackage[textsize=tiny]{todonotes}
\usepackage{graphicx}
\usepackage{enumitem}
\usepackage{titlesec}
\usepackage{cleveref}
\titleformat{\paragraph}
  [runin]              
  {\itshape}          
  {}                  
  {0em}               
  {}                  

\newcommand{\methodname}{\textsc{GraDate}}
\newcommand{\fullmethodname}{ \textsc{GraDate }(\textit{\underline{GRA}ph \underline{DAT}a s\underline{E}lector})}
\newcommand{\optmethodname}{\textsc{GREAT}}
\newcommand{\fulloptmethodname}
{\optmethodname\ (\textit{\underline{G}DD sh\underline{R}inkag\underline{E} via sp\underline{A}rse projec\underline{T}ion})}

\title{Graph Data Selection for Domain Adaptation:\\ A Model-Free Approach}

\author{%
  Ting-Wei Li \\
  University of Illinois \\
  Urbana-Champaign, IL USA \\
  \texttt{twli@illinois.edu} \\
  \And
  Ruizhong Qiu \\
  University of Illinois \\
  Urbana-Champaign, IL USA \\
  \texttt{rq5@illinois.edu} \\
  \And
  Hanghang Tong \\
  University of Illinois \\
  Urbana-Champaign, IL USA \\
  \texttt{htong@illinois.edu} \\
}

\begin{document}

\maketitle

\begin{abstract}
Graph domain adaptation (GDA) is a fundamental task in graph machine learning, with techniques like shift-robust graph neural networks (GNNs) and specialized training procedures to tackle the distribution shift problem. Although these \textit{model-centric} approaches show promising results, they often struggle with severe shifts and constrained computational resources.
To address these challenges, we propose a novel \textit{model-free} framework, \fullmethodname, that selects the best training data from the source domain for the classification task on the target domain. \methodname\ picks training samples without relying on any GNN model's predictions or training recipes, leveraging optimal transport theory to capture and adapt to distribution changes. \methodname\ is \textit{data-efficient}, scalable and meanwhile complements existing model-centric GDA approaches. Through comprehensive empirical studies on several real-world graph-level datasets and multiple covariate shift types, we demonstrate that \methodname\ outperforms existing selection methods and enhances off-the-shelf GDA methods with much fewer training data.
\end{abstract}

\section{Introduction}
\label{sec:intro}

Graphs have emerged as a fundamental data structure for representing complex relationships across diverse domains, from modeling molecular interactions in biological networks~\citep{huang2020skipgnn,bongini2022biognn,liu2023muse,zitnik2024current} to capturing user behaviors in recommendation systems~\citep{gao2022graph,he2020lightgcn,cheng2023multi,chizari2023bias,ying2018graph}. 
In graph-level classification tasks~\citep{xie2022semisupervised, xu2023kernel, zenggraph,hong2024label, liu2023survey}, where the goal is to categorize graph structures, the distribution shift between source and target domains poses significant challenges. While numerous graph neural network (GNN) methods have been proposed for graph domain adaptation (GDA)~\citep{wu2023non, dai2022graph, wu2020unsupervised, sun2019adagcn, liu2024rethinking}, they predominantly relies on model architecture design and training strategies, 
which are inherently \textit{model-dependent} and brittle. 
This model-centricity introduces practical challenges: (i) the need for extensive resources to train and validate different architectural variants and (ii) the ignorance of source data quality. 
To address these aforementioned issues, 
rather than relying on sophisticated GNN architectures or training procedures, we aim to answer a fundamental question: 

\quad \textit{How to select the most relevant source domain data, based on available validation data,  for better graph-level classification accuracy evaluated on the target domain ?} 

In this paper, we propose a \textit{model-free}
method, \fullmethodname, that selects a subset of important training data in the source domain independently of any specific GNN model design, making it both \textit{data-efficient} and versatile. \methodname\  reduces computational overhead and enables quick adaptation to unseen graph domains based on available validation data. 
Conceptually, \methodname\ first leverages Fused Gromov-Wasserstein (FGW) distance~\cite{vayer2020fused} to compare graph samples. We provide a theoretical justification to demonstrate FGW's unique advantage over multi-layer GNNs for graph comparison. Then, FGW is used as a building block to measure the dataset-level distance between training and validation sets, which is termed as Graph Dataset Distance (GDD). 
Through theoretical analysis, we demonstrate that the domain generalization gap between source and target domain is upper-bounded by GDD, which naturally motivates us to minimize this GDD between training and validation sets by identifying the optimal subset of training graph data.   
At the core of \methodname\ lies our novel optimization procedure, \fulloptmethodname, that interleaves between (i) optimal transport-based distribution alignment with gradient updates on training sample weights and (ii) projection of training sample weights to sparse probability simplex. This dual-step process systematically increases the weights of beneficial training samples while eliminating the influence of detrimental samples that could harm generalization performance. 

By extensive experiments on six real-world graph-level datasets and two types of covariate shifts, we first show that \methodname\ significantly outperforms existing data-efficient selection methods. When coupled with vanilla GNNs that are only trained on its selected data, \methodname\ even surpasses state-of-the-art GDA methods. Intriguingly, the practical implications of \methodname\ extend beyond mere data selection. By operating independently of model architecture, \methodname\ can seamlessly complement existing off-the-shelf GDA methods, further enhancing their performance through better data curation while significantly improving the \textit{data-efficiency}.

In summary, our main contributions in this paper are as follows:

\begin{enumerate}[leftmargin=1.5em]

    \item \textbf{Theoretical justification of FGW distance.} We show in Theorem~\ref{thm:3.1-fgw} that the output distance of multi-layer GNNs is upper-bounded by Fused Gromov–Wasserstein distance~\cite{vayer2020fused} between graphs, which motivates us to use FGW as a building block to compare graphs. 
    \item \textbf{Novel graph dataset distance formulation.} Through Theorem~\ref{thm:3.3-gap}, we prove that the graph domain generalization gap is upper-bounded by a notion of Graph Dataset Distance (GDD).
    \item \textbf{A strong model-free graph selector.} We introduce \fullmethodname\ as the first \textit{model-free} method tailored for domain shift problem for graph-level classification tasks, complementing the predominant model-centric GDA methods (see Section~\ref{ssec:as-graph-selector}).
    \item \textbf{A data-efficient and powerful GDA method.} We show that trained with data selected by \methodname, even vanilla GNN can beat sophisticated GDA baselines (see Section~\ref{ssec:as-gda-model}).
    \item \textbf{A universal GDA model enhancer.} We demonstrate that the data selected by \methodname\ can be combined with GDA methods to further enhance their performance and efficiency (see Section~\ref{ssec:as-gda-enhancer}). 
\end{enumerate}

The rest of the paper is organized as follows. Section~\ref{sec:prelim} introduces needed background knowledge. Section~\ref{sec:method} details our definition of Graph Dataset Distance (GDD) and proposed \methodname, followed by experimental results in Section~\ref{sec:exp}. Section~\ref{sec:related-work} presents related works. Finally, the conclusion is provided in Section~\ref{sec:conclusion}.

\vspace{-1em}
\section{Preliminaries}
\label{sec:prelim}

In this section, we briefly introduce optimal transport and graph optimal transport, which are fundamental background related to our proposed method. We also provide the formal problem formulation of graph domain adaptation (GDA) in Appendix~\ref{appendix:prelim-gda}.

\subsection{Optimal Transport}
\label{prelim:otdd}


Optimal transport (OT)~\citep{kantorovich1942translocation} defines a distance between probability distributions. 
It is defined as follows: given  cost function  $d(\cdot, \cdot): \mathcal{X}\times \mathcal{X} \rightarrow \mathbb{R}^{\geq 0}$ and probability distributions $\mathbf{p}, \mathbf{q} \in \mathcal{P}(\mathcal{X})$, where $\mathcal{X}$ is a metric space, 
$\text{OT}(\mathbf{p}, \mathbf{q}, d) \triangleq \min_{\pi \in \Pi(\mathbf{p}, \mathbf{q})} \int_{\mathcal{X} \times \mathcal{X}} d(x, x') \pi(x, x') \, \text{d}x \text{d}x'$, where $\Pi(\mathbf{p}, \mathbf{q})$
is the set of couplings 
with marginals $\mathbf{p}$ and $\mathbf{q}$.  For supervised learning scenarios, we consider the empirical measures: $\mathbf{p} = \frac{1}{n} \sum_{i=1}^n \delta_{x_i}$ and $\mathbf{q} = \frac{1}{m} \sum_{j=1}^m \delta_{x'_j}$, where $\delta$ is the Dirac delta function. With the pairwise cost matrix $\mathbf{M} = [d(x_i, x'_j)]_{ij}$, we can re-formulate the OT problem as $\text{OT}(\mathbf{p}, \mathbf{q}, d) = \text{OT}(\mathbf{p}, \mathbf{q}, \mathbf{M}) \triangleq \min_{\pi \in \Pi(\mathbf{p}, \mathbf{q})} \sum_{i=1}^n \sum_{j=1}^m \mathbf{M}_{ij} \pi_{ij}$.

\paragraph{Optimal Transport Dataset Distance (OTDD).}
\citet{alvarez2020geometric} construct a distance metric between \textit{datasets}, where each dataset is represented as a collection of feature-label pairs $z = (x, y) \in \mathcal{Z} (=\mathcal{X} \times \mathcal{Y})$.
The authors propose \textit{label-specific distributions}, which can be seen as distributions over features $X$ of data samples with a specific label $y$, i.e. $\alpha_y(X) \triangleq P(X | Y = y)$. The metric on the space of feature-label pairs can thus be defined as a combination of \textit{feature distance} and \textit{label distance}: $d_{\mathcal{Z}}((x, y), (x', y')) \triangleq \left(d_{\mathcal{X}}(x, x')^r + c \cdot d(\alpha_y, \alpha_{y'})^r\right)^{1/r}  $,
where $d_\mathcal{X}$ is a metric on $\mathcal{X}$, $d(\alpha_y, \alpha_{y'})$ is the \textit{label distance} between distributions of features associated with labels $y$ and $y'$, $r \geq 1$ is the order of the distances and  $c \geq 0$ is a pre-defined weight parameter. Consider two datasets $D_s = \{ z^s_i:(x^s_i, y^s_i)  \}_{i \in [n_1]}$, $D_t = \{ z^t_i:(x^t_i, y^t_i)  \}_{i \in [n_2]}$ and their corresponding uniform empirical distributions $\mathbf{p}, \mathbf{q} \in \mathcal{Z}$, where  $\mathbf{p} = \sum_{i=1}^{n_1} \frac{1}{n_1} \delta_{z^s_i}$ and $\mathbf{q} = \sum_{i=1}^{n_2} \frac{1}{n_2} \delta_{z^t_i}$,  OTDD between $D_s$ and $D_t$ is computed as: $\text{OTDD}(D_s,D_t) 
     = \text{OT} (\mathbf{p}, \mathbf{q}, d_\mathcal{Z})
     = \text{OT} (\mathbf{p}, \mathbf{q}, \mathbf{M})
     = \min_{\pi \in \Pi(\mathbf{p}, \mathbf{q})}\mathbb E_{(z^s,z^t)\sim\pi}[d_\mathcal Z(z^s,z^t)]$,
where $\Pi(\mathbf{p}, \mathbf{q})$ is the set of valid couplings and $\mathbf{M} = [d_\mathcal{Z}(z^s_i, z^t_j)]_{ij}$ is the pairwise cost matrix. 

\subsection{Graph Optimal Transport}
\label{prelim:fgw-distance}

The Fused Gromov-Wasserstein (FGW) distance~\cite{vayer2020fused} integrates the Wasserstein distances~\cite{rubner2000earth} and the Gromov-Wasserstein distances~\cite{sturm2023space, memoli2011gromov}. 
Formally, a graph $\mathcal{G}$ with $n$ nodes can be represented as a distribution over vectors in a $d$-dimensional space, where $d$ is the dimension of the node feature. The \textit{features} are represented as $\mathbf{X} \in \mathbb{R}^{n \times d}$ and the \textit{structure} can be summarized in an adjacency matrix $\mathbf{A} \in \mathbb{R}^{n \times n}$. We further augment $\mathcal{G}$ a probability distribution $\mathbf{p} \in \mathbb{R}^n$ over nodes in the graph, where $\sum_{i=1}^n \mathbf{p} _i = 1$ and $\mathbf{p}_i \geq 0, \forall i \in [n]$. 
To compute FGW distance between two attributed graphs ($\mathcal{G}_1 = \{  
\mathbf{A}_1 , \mathbf{X}_1, \mathbf{p}_1  \}$ and $\mathcal{G}_2 = \{ 
\mathbf{A}_2 , \mathbf{X}_2, \mathbf{p}_2  \}$), 
we use pairwise feature distance as inter-graph distance matrix and adjacency matrices as intra-graph similarity matrices. 
The FGW distance is defined as the solution of the following optimization problem: $\textnormal{FGW}_\alpha(\mathcal{G}_1, \mathcal{G}_2)
    \triangleq \textnormal{FGW}_\alpha([\| \mathbf{X}_1[i]-\mathbf{X}_2[j]  \|_r]_{ij}, \mathbf{A}_1, \mathbf{A}_2, \mathbf{p}_1,\mathbf{p}_2, \alpha) \notag$
\vspace{-1.5em}
\begin{center}
\resizebox{0.85\linewidth}{!}{
\begin{minipage}{\linewidth}
\begin{align*}
\label{eq:prelim-linearfgw}
&= \left(\min_{\boldsymbol\pi\in\Pi(\mathbf p_1,\mathbf p_2)} \sum_{i,j,k,l} (1-\alpha) \| \mathbf{X}_1[i]-\mathbf{X}_2[j]  \|_2^r + \alpha |\mathbf{A}_1 (i,k)- \mathbf{A}_2 (j,l)|^r \boldsymbol\pi(i,j)\boldsymbol\pi(k,l) \right)^{\frac{1}{r}},
\end{align*}
\end{minipage}
}
\end{center}
where $\Pi(\mathbf p_1,\mathbf p_2) \triangleq \{\boldsymbol\pi |\boldsymbol\pi \textbf{1}_{n_2} = \mathbf{p}_1, \boldsymbol\pi^T \textbf{1}_{n_1} = \mathbf{p}_2, \boldsymbol\pi \geq 0\}$ is the collection of feasible couplings between $\mathbf{p}_1$ and $\mathbf{p}_2$,  $\alpha \in [0,1]$ acts as a trade-off parameter, and $r\ge1$ is the order of the distances.

\section{Methodology}
\label{sec:method}
\vspace{-0.5em}

In this section, we propose our framework, \textit{\underline{GRA}ph \underline{DAT}a s\underline{E}lector}, abbreviated as \methodname. 
In Section~\ref{ssec:graph-dataset-distance}, we first introduce Theorem~\ref{thm:3.1-fgw} to motivate our use of LinearFGW~\cite{nguyen2023linear} for graph distance computation. After that, we define the Graph Dataset Distance (GDD), which
measures the discrepancy between graph sets. In Section~\ref{ssec:graph-data-selection}, we provide Theorem~\ref{thm:3.3-gap} to bound the domain generalization gap using GDD and then
propose a GDD minimization problem that is solved by a novel optimization algorithm, \fulloptmethodname.
Finally, we introduce \methodname\ that combines GDD computation and \optmethodname\ to select the most important subset of the training data to solve graph domain adaptation problem (definition is detailed in Appendix~\ref{appendix:ssec:graph-da}).





\label{ssec:linearfgw}

\subsection{Graph Dataset Distance (GDD): A Novel Notion to Compare Graph Datasets}
\label{ssec:graph-dataset-distance}

\subsubsection{FGW Distance For Graph Comparison} 

\paragraph{Challenges in Graph Distance Computation.}
To find the optimal samples that can achieve better performance on the target domain, we first need an efficient way that can accurately capture the discrepancy among graphs. To achieve this, most methods rely on Graph Neural Networks (GNNs)~\cite{wu2020unsupervised, dai2022graph, sun2019adagcn, yin2024dream} to obtain meaningful representations of these structured objects. However, these approaches face critical limitations: (i) high computational complexity 
to train on full training set and (ii) sensitivity to GNN hyper-parameters. 
To address these drawbacks, we propose to use FGW distance~\cite{vayer2020fused} for replacing GNNs. As demonstrated in the following Theorem~\ref{thm:3.1-fgw}, FGW offers provable advantages that make it more suitable than GNNs to compare attributed graphs.





\begin{theorem}
\label{thm:3.1-fgw}
Given two graphs $\mathcal G_1=(\mathbf A_1,\mathbf X_1)$ and $\mathcal G_2=(\mathbf A_2,\mathbf X_2)$, for a $k$-layer graph neural network (GNN) $f$ with ReLU activations, under regularity assumptions in Appendix~\ref{appendix:ass:fgw}, we have 
\begin{align}
d_{\textnormal W}(f(\mathcal G_1),f(\mathcal G_2))
\le C\cdot\textnormal{FGW}_\beta(\mathcal G_1,\mathcal G_2),
\end{align}
where $d_{\textnormal W}$ denotes the $r$-Wasserstein distance, $C$ and $\beta$ are constants depending on GNN $f$, regularity constants and $k$. 
\begin{proof}
The proof is in Appendix~\ref{appendix:thm3.1-fgw-proof}.
\end{proof}
\end{theorem}

\paragraph{Implication of Theorem~\ref{thm:3.1-fgw}.} 
We have the following two main insights: (i) since the theorem holds for any possible $k$-layer GNNs, with ReLU activations, FGW is provably able to capture the differences between attributed graphs in a way that upper-bounds the discrepancy between learned GNN representations; (ii) to the best of our knowledge, this is the first time that FGW is proved to be the distance metric that enjoys this theoretical guarantee, and hence we adopt it as the major basis for our graph data selection method.

\paragraph{Practical Consideration.}
We utilize LinearFGW~\cite{nguyen2023linear}  as an efficient approximation of FGW distance. Formally, LinearFGW defines a distance metric $d_\text{LinearFGW} (\cdot,\cdot)$ where $d_\text{LinearFGW} (\mathcal{G}_i, \mathcal{G}_j)$ is the LinearFGW distance between any pair of graphs $\mathcal{G}_i, \mathcal{G}_j$. LinearFGW offers an approximation to FGW with linear time complexity with respect to the number of training graphs. While we omit the details of LinearFGW here for brevity, they can be found in Appendix~\ref{appendix:linearfgw} and Algorithm~\ref{alg:linearfgw}.

\subsubsection{Graph-Label Distance} 
With FGW as a theoretically grounded metric for \textit{model-free} comparison between individual attributed graphs, we further extend it to compare sets of labeled graphs across domains, which aids our domain adaptation process. Inspired by OTDD~\citep{alvarez2020geometric} (detailed in Section~\ref{prelim:otdd}), 
we extend the original definition of \textit{label distance} to incorporate distributions in graph subsets $S, S'$, namely $\alpha_y^{S}$ and $\alpha_{y'}^{S'}$. 
Given a set of labeled graphs $\mathcal{D} = \{\mathcal{G}_i, y_i\}_{i=1}^N$ and label set $\mathcal{Y}$, we formulate the \textit{graph-label distance} between label $y \in \mathcal{Y}$ in graph subset $S = \{\mathcal{G}^S_i, y^S_i\}_{i \in |S|} \subseteq \mathcal{D}$ and label $y' \in \mathcal{Y}$ in graph subset $S' = \{\mathcal{G}^{S'}_j, y^{S'}_j\}_{j \in |S'|} \subseteq \mathcal{D}$
as follows:
\begin{equation}
\label{eq:method-label-distance}
     d(\alpha_y^S, \alpha_{y'}^{S'}) = \text{OT}(\mathbf{p}^S_y, \mathbf{q}^{S'}_{y'}, d_{\text{LinearFGW}} ) ,
\end{equation}
where $\mathbf{p}^S_y = \frac{1}{|{{i:y^S_i = y}}|} \sum_{{i:y^S_i = y}}\delta_{\mathcal{G}^S_i}$,  $\mathbf{q}^{S'}_{y'} = \frac{1}{|{{j:y^{S'}_j = y}}|} \sum_{{j:y^{S'}_j = y}}\delta_{\mathcal{G}^{S'}_j}$ are label-specific uniform empirical measures with the distance metric measured by LinearFGW. Intuitively, we collect graphs in subset $S$ with label $y$ and graphs in subset $S'$ with label $y'$  as distributions. Then, we compute the optimal transport distance between these distributions and define the distance as \textit{graph label distance}. 

\subsubsection{Graph Dataset Distance (GDD)}
Building upon the aformentioned graph-label distance, we then propose the notion of Graph Dataset Distance (GDD), which compares two graph subsets at a dataset-level. Specifically, based on Equation~(\ref{eq:method-label-distance}), we can define a distance metric $d^c_{g\mathcal{Z}}$ between graph subsets $S, S'$:
\resizebox{\linewidth}{!}{
\begin{minipage}{\linewidth}
\begin{align}
\label{eq:method-gfgw-distance}
    d^c_{g\mathcal{Z}}((\mathcal{G}^S_i, y^{S}_i), (\mathcal{G}^{S'}_j, y^{S'}_j)) 
    &= d_\text{LinearFGW}(\mathcal{G}^S_i, \mathcal{G}^{S'}_j) + c \cdot  d(\alpha_y^S, \alpha_{y'}^{S'}),
\end{align}
\end{minipage}
}

which is a combination of LinearFGW distance and graph-label distance with a weight parameter $c \geq 0$ balancing the importance of two terms. GDD can thus be defined as:
\begin{equation}
\label{eq:method-gdd-distance}
    \text{GDD}(\mathcal{D}^S, \mathcal{D}^{S'}) = \text{OT}(\mathbf{p}^S, \mathbf{q}^{S'}, d^c_{g\mathcal{Z}}), 
\end{equation}
where $\mathbf{p}^S = \frac{1}{|S|} \sum_{i \in [|S|]} \delta_{(\mathcal{G}_i^S, y^S_i)}$ and  
$\mathbf{q}^{S'} = \frac{1}{|S'|} \sum_{j \in [|S'|]} \delta_{(\mathcal{G}_j^{S'}, y^{S'})}$ are 
uniform empirical measures. We summarize the computation of GDD in Appendix~\ref{appendix:gdd-computation} and Algorithm~\ref{alg:gdd-computation}.

\begin{remark}
\label{remark:1}
If we set $c=0$, GDD will omit the label information and only consider the distributional differences of graph data themselves, which matches the setting of unsupervised GDA problem. 
\end{remark}

\subsection{\methodname: A Model-Free Graph Data Selector}
\label{ssec:graph-data-selection}




\subsubsection{GDD Bounds Domain Generaization Gap}


We give the following Theorem~\ref{thm:3.3-gap} to elucidate the utility of GDD and its relation to model performance discrepancy between graph domains. In short, we seek to utilize empirical observations in the source domain to minimize the expected risk calculated on the target domain $P_t$, namely, $\mathbb E_{(\mathcal{G},y)\sim P_t}[\mathcal{L} (f(\mathcal G), y)]$, which promotes the model performance (i.e., lower expected risk) on the target domain.

\begin{theorem}[Graph Domain Generalization Gap]
\label{thm:3.3-gap}
Define the cost function among graph-label pairs as $d^c_{g\mathcal Z}$ with some positive $c$ (via Equation~(\ref{eq:method-gfgw-distance})). 
Let $\mathbf w$ denote the source distribution weight. 
Suppose that $d^c_{g\mathcal Z}$ is $C$-Lipschitz. Then for any model $f$ trained on a training set, we have 
\begin{align}
\underset{(\mathcal{G},y)\sim\mathbf q^\textnormal{val}}{\mathbb E}[\mathcal{L} (f(\mathcal G), y)] 
\le
\underset{(\mathcal G,y)\sim\mathbf p^\textnormal{train}(\mathbf w)}{\mathbb E}[\mathcal{L} (f(\mathcal G), y)] 
+ 
C \cdot 
\textnormal{GDD} 
(\mathcal{D^{\textnormal{train}}_\mathbf{w}, \mathcal{D}^\textnormal{val}}),\nonumber
\end{align}

where $ \textnormal{GDD} 
(\mathcal{D^{\textnormal{train}}_\mathbf{w}, \mathcal{D}^\textnormal{val}}) =  \textnormal{OT}(\mathbf p^\textnormal{train}(\mathbf w),\mathbf q^\textnormal{val},d^c_{g\mathcal Z})$ is the graph dataset distance between the weighted training dataset (defined by $\mathbf{w}$) and target dataset.
\begin{proof}
The proof can be found in Appendix~\ref{appendix:proof:1}.
\end{proof}
\end{theorem}
\vspace{-1em}

\paragraph{Implication of Theorem~\ref{thm:3.3-gap}. } 
If we can lower the GDD between training and validation data, the discrepancy in the model performance with respect to training and validation sets may also decrease. Specifically, under the scenario where distribution shift occurs, some source data might be irrelevant or even harmful when learning a GNN model that needs to generalize well on the target domain. This motivates us to present our main framework, \fullmethodname, which selects the best training data from the source domain for graph domain adaptation.

\subsubsection{GDD Minimization Problem}  
Guided by the implication of Theorem~\ref{thm:3.3-gap}, we formulate the GDD minimization problem as follows. 

\begin{definition}[Graph Dataset Distance Minimization]
Given training and validation sets $\mathcal{D}^\text{train}$ and $\mathcal{D}^\text{val}$, we aim to find the best distribution weight $\mathbf{w}^*$ that minimizes GDD betwen the training and validation sets under the sparsity constraint. Namely, 
\begin{align}
\label{eq:method-gdd-minimization}
    \mathbf{w}^* = &\min_{\mathbf{w}} \, \text{OT}(\mathbf{p}^\text{train} (\textbf{w}), \mathbf{q}^{\text{val}}, d^c_{g\mathcal Z} ),
    &\text{s.t.} \sum_i \mathbf{w}[i] = 1, \mathbf{w}\geq 0, \|\mathbf{w}\|_0 \leq \lfloor n\cdot \tau \rfloor,
\end{align}
where $\mathbf{p}^\text{train} (\textbf{w}) =  \sum_{i \in [n]} \mathbf{w}_i \delta_{(\mathcal{G}_i^\text{train}, y^\text{train}_i)}$ and $\mathbf{q}^{\text{val}} = \frac{1}{m} \sum_{j \in [m]} \delta_{(\mathcal{G}_j^\text{val}, y_j^{\text{val}})}$. 
\end{definition}


\paragraph{Optimization Procedure (\optmethodname\ algorithm).}
\label{ssec:alternate-got-shrinkage}
To solve the above GDD minimization problem, we propose \fulloptmethodname\
to optimize the weight $\mathbf{w}$ over the entire training set. 
Starting from a uniform training weight vector \( \mathbf{w} \), \optmethodname\ iteratively refines the importance of training samples through two key steps. In each iteration, it first computes the Graph Dataset Distance (GDD) between the reweighted training distribution \( \mathbf{p}^\text{train}(\mathbf{w}) \) and the validation distribution \( \mathbf{q}^\text{val} \), using a pairwise cost matrix \( \tilde{\mathbf{D}} \in \mathbb{R}^{n \times m} \) that encodes distances between individual training and validation samples. The gradient of GDD with respect to \( \mathbf{w} \), denoted \( \mathbf{g}_\mathbf{w} \), is then used to update the weights\footnote{Note that we leverage the conclusion introduced in~\citep{just2023lava} to compute this gradient (stated as Theorem~\ref{thm:grad-gdd}).}. Following this, the weight vector \( \mathbf{w} \) is sparsified by retaining only the top-\( k \) largest entries and re-normalized to remain on the probability simplex. After \( T \) such iterations, the non-zero indices in the final \( \mathbf{w} \) define the selected training subset \( \mathbf{S} \). The detailed algorithm procedure is presented in Appendix~\ref{appendix:great} and Algorithm~\ref{alg:alt-got-shrinkage}.




\subsubsection{\methodname}

Combining GDD computation and optimization module \optmethodname, we summarize the main procedure of \methodname\ in Algorithm~\ref{alg:overall-alg} (details can be found in Appendix~\ref{appendix:gradate}). In short, given training and validation data, \methodname\ iteratively calculates GDD based on current training weight and searches for a better one through \optmethodname. The final output of \methodname\ corresponds to the selected subset of training data that is best suitable for adaptation to the target domain. We further provide the time complexity analysis of \methodname\ as follows.

\paragraph{Time Complexity of \methodname.} 
Let $N$ be the number of training graphs, $M$ be the number of validation graphs, $n$ be the number of nodes in each graph (WLOG, we assume all graph share the same size), $L$ be the largest class size, $\tau$ is the approximation error introduced by approximate OT solvers~\footnote{This is due to the entropic regularization in Sinkhorn iterations for empirical OT calculation.}
, $K$ be the number of iterations for solving LinearFGW, and $T$ the number of update steps used in \optmethodname. The runtime complexity can be summarized in the following proposition.~\footnote{For empirical runtime behavior, we refer readers to Appendix~\ref{appendix:empirical-runtime}.}

\begin{proposition}[Time Complexity Analysis~\citep{alvarez2020geometric,just2023lava, altschuler2017near}]
The off-line procedure of \methodname\ (i.e. can be computed before accessing the test set) has the time complexity 
$\mathcal{O}(N M K n^3 + N M L^3 \log L) $ and the on-line procedure of \methodname\ has the time complexity $\mathcal{O}(TNM \log (\max (N,M) ) \tau^{-3})$.

\end{proposition}

\section{Experiments}
\label{sec:exp}



We conduct extensive experiments to evaluate the effectiveness of \methodname\ across six real-world graph classification settings under two different types of distribution shift. Our experiments are designed to answer the following research questions: 
\begin{itemize}[leftmargin=1.1em]
    \item (\textbf{RQ1}): How does \methodname\ compare to existing data selection methods?
    \item (\textbf{RQ2}): How does \methodname\ + vanilla GNNs compare to model-centric GDA methods?
    \item (\textbf{RQ3}): To what extent can \methodname\ enhance the effectiveness of model-centric GDA methods? 
\end{itemize}

We will answer these research questions in Section~\ref{ssec:as-graph-selector}, \ref{ssec:as-gda-model} and~\ref{ssec:as-gda-enhancer}, correspondingly.

\subsection{General Setup}

In this section, we state the details of datasets and settings of \methodname\ and baseline methods.

\paragraph{Datasets and Graph Domains.} We consider graph classification tasks conducted on six  real-world graph-level datasets, including \textsc{IMDB-BINARY}~\citep{yanardag2015deep}, \textsc{IMDB-MULTI}~\citep{yanardag2015deep}, \textsc{MSRC\_21}~\citep{neumann2016propagation}, \texttt{ogbg-molbace}~\citep{hu2020open}, \texttt{ogbg-molbbbp}~\citep{hu2020open} and \texttt{ogbg-molhiv}~\citep{hu2020open}. The former three datasets are from the TUDataset~\citep{morris2020tudataset}; while the latter three datasets are from the OGB benchmark~\citep{hu2020open}. We define the graph domains by \textit{graph density} and \textit{graph size}, which are the types of covariate shift that are widely studied in the literature~\citep{yin2024dream, luo2024rank, tang2024multi, bevilacqua2021size, yehudai2021local, chu2023wasserstein, zhou2022ood}. Specifically, graphs are sorted by corresponding properties in an \textit{ascending} order and split into train/val/test sets with ratios $60\%/20\%/20\%$. For brevity, we provide results on \textit{graph density} shift in the main content. Additional experiments on \textit{graph size} shift can be found in Appendix~\ref{appendix:additional-exp}. We also include the empirical cumulative distribution function (ECDF) 
plots of all settings in Appendix~\ref{appendix:ecdf-plots} to demonstrate the shift level.


\paragraph{Details of \methodname\ and Baselines.} To compute LinearFGW within \methodname, we follow the default parameter settings in its github repository.\footnote{\href{https://github.com/haidnguyen0909/LinearFGW/blob/main/main.py}{https://github.com/haidnguyen0909/LinearFGW}} The trade-off parameter $\alpha$ is computed in $\{0.5, 0.9\}$\footnote{Since the datasets do not contain node features, we consider a larger $\alpha$ to place a greater emphasis on the structural properties.} and the order $r$ is set to 2. The update step is fixed to $T=10$ and the learning rate equal to $\eta=10^{-4}$ across different settings. 
A popular model-free data valuation method is LAVA~\citep{just2023lava}. We apply LAVA for graph data selection and make the following modifications. We first leverage LinearFGW to form the pairwise distance matrix and compute GDD. LAVA then picks the smallest $k$ entries of the calibrated gradients as output.
For the computation of GDD, we consider label signal \( c \in \{0, 5\} \).
We also incorporate \textsc{KiDD}-LR~\citep{xu2023kernel} as a 
model-centric but data-efficient baseline, which is a state-of-the-art graph dataset distillation method.

\subsection{\methodname\ as a Model-Free Graph Selector}
\label{ssec:as-graph-selector}

To answer (\textbf{RQ1}), we compare \methodname\ with other data-efficient methods, including (1) model-free techniques: random selection and LAVA~\citep{just2023lava} and (2) model-centric techniques: \textsc{KiDD}-LR~\citep{xu2023kernel}.



\begin{table*}[t]
\centering
\resizebox{1.0\textwidth}{!}{%
\begin{tabular}{@{\extracolsep{\fill}}cc|cccc|cccc}
\toprule
\midrule
\multirow{2.4}{*}{\textbf{Dataset}}&   \multirow{1}{*}{\textbf{GNN Architecture} $\rightarrow$\!\!\!\! }& \multicolumn{4}{c|}{\makecell{\textbf{GCN}}}  & \multicolumn{4}{c}{\makecell{\textbf{GIN}}} \\
\cmidrule{2-10} 
&\textbf{Selection Method} $\downarrow$\!& $\tau=10\%$ & $\tau=20\%$ & $\tau=50\%$ & Full & $\tau=10\%$ & $\tau=20\%$ & $\tau=50\%$ & Full \\
\midrule
\multirow{4}{*}{\textsc{IMDB-BINARY}}
& Random &       
\underline{0.737}\scriptsize{${\pm}$ 0.056} & 0.660\scriptsize{${\pm}$ 0.012} & \underline{0.868}\scriptsize{${\pm}$ 0.009} & 

\multirow{4}{*}{0.822\scriptsize{${\pm}$0.012}}
&    
0.600\scriptsize{${\pm}$ 0.019} & 0.710\scriptsize{${\pm}$ 0.049} & 0.770\scriptsize{${\pm}$ 0.053} & 

\multirow{4}{*}{0.783\scriptsize{${\pm}$0.031}}\\
&\textsc{KiDD}-\textsc{LR} &   0.697\scriptsize{${\pm}$ 0.041} & \underline{0.787}\scriptsize{${\pm}$ 0.034} & 0.810\scriptsize{${\pm}$ 0.022} & 
  
& 
0.682\scriptsize{${\pm}$ 0.013} & 0.772\scriptsize{${\pm}$ 0.029} & 0.795\scriptsize{${\pm}$ 0.014} & 
  \\

& LAVA   &       
0.620\scriptsize{${\pm}$ 0.000} & 0.620\scriptsize{${\pm}$ 0.000} & 0.620\scriptsize{${\pm}$ 0.000} & 
 
&    
\underline{0.777}\scriptsize{${\pm}$ 0.019} & \underline{0.795}\scriptsize{${\pm}$ 0.007} & \underline{0.800}\scriptsize{${\pm}$ 0.007} & 

   \\

\cmidrule{2-5} \cmidrule{7-9}
  

&\cellcolor[HTML]{D3D3D3}\methodname    &   
\cellcolor[HTML]{D3D3D3}\textbf{0.805}\scriptsize{${\pm}$ 0.000} & \cellcolor[HTML]{D3D3D3}\textbf{0.855}\scriptsize{${\pm}$ 0.024} & \cellcolor[HTML]{D3D3D3}\textbf{0.890}\scriptsize{${\pm}$ 0.015} & 

& 
\cellcolor[HTML]{D3D3D3}\textbf{0.800}\scriptsize{${\pm}$ 0.008} & \cellcolor[HTML]{D3D3D3}\textbf{0.832}\scriptsize{${\pm}$ 0.002} & \cellcolor[HTML]{D3D3D3}\textbf{0.900}\scriptsize{${\pm}$ 0.013} & 

\\
\midrule
\multirow{4}{*}{\textsc{IMDB-MULTI}}
& Random &       
0.139\scriptsize{${\pm}$ 0.032} & 0.092\scriptsize{${\pm}$ 0.032} & 0.080\scriptsize{${\pm}$ 0.000} & 
\multirow{4}{*}{0.102\scriptsize{${\pm}$0.017}}
&   
0.102\scriptsize{${\pm}$ 0.015} & \underline{0.180}\scriptsize{${\pm}$ 0.005} & 0.156\scriptsize{${\pm}$ 0.057} & 

\multirow{4}{*}{0.143\scriptsize{${\pm}$0.056}}\\
&\textsc{KiDD}-\textsc{LR} &   0.156\scriptsize{${\pm}$ 0.022} & 0.154\scriptsize{${\pm}$ 0.046} & 0.171\scriptsize{${\pm}$ 0.052} & 
 
& 
0.058\scriptsize{${\pm}$ 0.044} & 0.093\scriptsize{${\pm}$ 0.010} & 0.077\scriptsize{${\pm}$ 0.025} & 
   \\

&LAVA   &   
\underline{0.183}\scriptsize{${\pm}$ 0.000} & \underline{0.183}\scriptsize{${\pm}$ 0.000} & \underline{0.183}\scriptsize{${\pm}$ 0.000} & 
  
& 
\textbf{0.190}\scriptsize{${\pm}$ 0.009} & 0.177\scriptsize{${\pm}$ 0.019} & \underline{0.193}\scriptsize{${\pm}$ 0.025} & 
   \\
\cmidrule{2-5} \cmidrule{7-9}
& \cellcolor[HTML]{D3D3D3}\methodname   & \cellcolor[HTML]{D3D3D3}\textbf{0.588}\scriptsize{${\pm}$ 0.286} &  \cellcolor[HTML]{D3D3D3}\textbf{0.349}\scriptsize{${\pm}$ 0.323} &  \cellcolor[HTML]{D3D3D3}\textbf{0.611}\scriptsize{${\pm}$ 0.242} & 
 
 &
 \cellcolor[HTML]{D3D3D3}\underline{0.183}\scriptsize{${\pm}$ 0.073} &  \cellcolor[HTML]{D3D3D3}\textbf{0.266}\scriptsize{${\pm}$ 0.133} &  
 \cellcolor[HTML]{D3D3D3}\textbf{0.361}\scriptsize{${\pm}$ 0.162} & 
   \\
\midrule
\multirow{4}{*}{\textsc{MSRC}\_21}
& Random &       
0.576\scriptsize{${\pm}$ 0.029} & 0.702\scriptsize{${\pm}$ 0.045} & 0.830\scriptsize{${\pm}$ 0.004} & 
\multirow{4}{*}{0.860\scriptsize{${\pm}$0.007}}
&   
0.427\scriptsize{${\pm}$ 0.035} & 0.801\scriptsize{${\pm}$ 0.046} & 0.857\scriptsize{${\pm}$ 0.011} &

\multirow{4}{*}{0.883\scriptsize{${\pm}$0.015}}\\
&\textsc{KiDD}-\textsc{LR} &   
\underline{0.702}\scriptsize{${\pm}$ 0.007} & 0.766\scriptsize{${\pm}$ 0.015} & 0.848\scriptsize{${\pm}$ 0.004} & 
& 
\underline{0.763}\scriptsize{${\pm}$ 0.025} & 0.792\scriptsize{${\pm}$ 0.017} & 0.863\scriptsize{${\pm}$ 0.015} & 
   \\

&LAVA   &   
0.623\scriptsize{${\pm}$ 0.007} & \textbf{0.819}\scriptsize{${\pm}$ 0.012} & \underline{0.895}\scriptsize{${\pm}$ 0.012} & 

& 
0.667\scriptsize{${\pm}$ 0.012} & \underline{0.851}\scriptsize{${\pm}$ 0.007} & \underline{0.933}\scriptsize{${\pm}$ 0.004} & 
   \\
\cmidrule{2-5} \cmidrule{7-9}
& \cellcolor[HTML]{D3D3D3}\methodname   
& \cellcolor[HTML]{D3D3D3}\textbf{0.719}\scriptsize{${\pm}$ 0.007} 
& \cellcolor[HTML]{D3D3D3}\underline{0.797}\scriptsize{${\pm}$ 0.008} 
& \cellcolor[HTML]{D3D3D3}\textbf{0.906}\scriptsize{${\pm}$ 0.004} & 
&
\cellcolor[HTML]{D3D3D3}\textbf{0.787}\scriptsize{${\pm}$ 0.046} 
&\cellcolor[HTML]{D3D3D3}\textbf{0.860}\scriptsize{${\pm}$ 0.007} 
& \cellcolor[HTML]{D3D3D3}\textbf{0.942}\scriptsize{${\pm}$ 0.008} & 
   \\
\midrule
\multirow{4}{*}{\texttt{ogbg-molbace}}
& Random &       
0.551\scriptsize{${\pm}$ 0.100} & 0.375\scriptsize{${\pm}$ 0.012} & 0.581\scriptsize{${\pm}$ 0.039} & 
\multirow{4}{*}{0.617\scriptsize{${\pm}$0.073}}
&   
\underline{0.637}\scriptsize{${\pm}$ 0.012} & 0.621\scriptsize{${\pm}$ 0.027} & 0.537\scriptsize{${\pm}$ 0.085} &

\multirow{4}{*}{0.560\scriptsize{${\pm}$0.063}}\\
&\textsc{KiDD}-\textsc{LR} &   
0.592\scriptsize{${\pm}$ 0.054} & 0.484\scriptsize{${\pm}$ 0.020} & 0.592\scriptsize{${\pm}$ 0.008} & &
0.613\scriptsize{${\pm}$ 0.090} & 0.456\scriptsize{${\pm}$ 0.041} & 0.589\scriptsize{${\pm}$ 0.035} & 
\\

&LAVA   &   
\underline{0.627}\scriptsize{${\pm}$ 0.033} &
\textbf{0.637}\scriptsize{${\pm}$ 0.030} &
\textbf{0.637}\scriptsize{${\pm}$ 0.014} & 

& 
0.602\scriptsize{${\pm}$ 0.028} & 
\underline{0.633}\scriptsize{${\pm}$ 0.048} & 
\underline{0.672}\scriptsize{${\pm}$ 0.028} & 
   \\
\cmidrule{2-5} \cmidrule{7-9}
& \cellcolor[HTML]{D3D3D3}\methodname   & 
\cellcolor[HTML]{D3D3D3}\textbf{0.655}\scriptsize{${\pm}$ 0.046} &
\cellcolor[HTML]{D3D3D3}\underline{0.578}\scriptsize{${\pm}$ 0.035} & 
\cellcolor[HTML]{D3D3D3}\underline{0.614}\scriptsize{${\pm}$ 0.042} & 
&
\cellcolor[HTML]{D3D3D3}\textbf{0.642}\scriptsize{${\pm}$ 0.083} & 
\cellcolor[HTML]{D3D3D3}\textbf{0.660}\scriptsize{${\pm}$ 0.026} & 
\cellcolor[HTML]{D3D3D3}\textbf{0.684}\scriptsize{${\pm}$ 0.026}  & 
   \\
\midrule
\multirow{4}{*}{\texttt{ogbg-molbbbp}}
& Random &       
0.567\scriptsize{${\pm}$ 0.037} & 0.488\scriptsize{${\pm}$ 0.088} & 0.478\scriptsize{${\pm}$ 0.021} & 
\multirow{4}{*}{0.478\scriptsize{${\pm}$0.069}}
&   
0.534\scriptsize{${\pm}$ 0.084} & \underline{0.648}\scriptsize{${\pm}$ 0.045} & 0.623\scriptsize{${\pm}$ 0.019} & 

\multirow{4}{*}{0.671\scriptsize{${\pm}$0.034}}\\
&\textsc{KiDD}-\textsc{LR} &   
0.428\scriptsize{${\pm}$ 0.025} & 0.477\scriptsize{${\pm}$ 0.080} & 0.457\scriptsize{${\pm}$ 0.013} & 
& 
0.424\scriptsize{${\pm}$ 0.005} & 0.450\scriptsize{${\pm}$ 0.070} & 0.464\scriptsize{${\pm}$ 0.052} & 
   \\

&LAVA   &   
\underline{0.596}\scriptsize{${\pm}$ 0.058} & 
\underline{0.566}\scriptsize{${\pm}$ 0.021} & 
\underline{0.547}\scriptsize{${\pm}$ 0.044} & 

& 
\underline{0.619}\scriptsize{${\pm}$ 0.044} & 
0.642\scriptsize{${\pm}$ 0.120} & 
\textbf{0.747}\scriptsize{${\pm}$ 0.024}
   \\
\cmidrule{2-5} \cmidrule{7-9}
& 
\cellcolor[HTML]{D3D3D3}\methodname   & 
\cellcolor[HTML]{D3D3D3}\textbf{0.604}\scriptsize{${\pm}$ 0.065} & 
\cellcolor[HTML]{D3D3D3}\textbf{0.601}\scriptsize{${\pm}$ 0.047} & 
\cellcolor[HTML]{D3D3D3}\textbf{0.557}\scriptsize{${\pm}$ 0.001}  & 
&
\cellcolor[HTML]{D3D3D3}\textbf{0.657}\scriptsize{${\pm}$ 0.039} &
\cellcolor[HTML]{D3D3D3}\textbf{0.677}\scriptsize{${\pm}$ 0.072} & 
\cellcolor[HTML]{D3D3D3}\underline{0.715}\scriptsize{${\pm}$ 0.015} & 
   \\
\midrule
\multirow{4}{*}{\texttt{ogbg-molhiv}}
& Random &       
\underline{0.603}\scriptsize{${\pm}$ 0.005} & \textbf{0.615}\scriptsize{${\pm}$ 0.004} & \underline{0.621}\scriptsize{${\pm}$ 0.001} & 
\multirow{4}{*}{0.625\scriptsize{${\pm}$0.001}}
&   
0.608\scriptsize{${\pm}$ 0.015} & 0.609\scriptsize{${\pm}$ 0.030} & 0.593\scriptsize{${\pm}$ 0.012} & 

\multirow{4}{*}{0.596\scriptsize{${\pm}$0.015}}\\
&\textsc{KiDD}-\textsc{LR} &   
0.590\scriptsize{${\pm}$ 0.005} & \underline{0.608}\scriptsize{${\pm}$ 0.001} & 0.595\scriptsize{${\pm}$ 0.011} & 
& 
0.597\scriptsize{${\pm}$ 0.042} & 0.595\scriptsize{${\pm}$ 0.039} & 0.608\scriptsize{${\pm}$ 0.020} & 
   \\

&LAVA   &   
0.531\scriptsize{${\pm}$ 0.035} & 
0.594\scriptsize{${\pm}$ 0.013} &  
0.601\scriptsize{${\pm}$ 0.013} & 

& 
\underline{0.614}\scriptsize{${\pm}$ 0.002} & 
\underline{0.638}\scriptsize{${\pm}$ 0.020} & 
\underline{0.641}\scriptsize{${\pm}$ 0.010}  & 
   \\
\cmidrule{2-5} \cmidrule{7-9}
& \cellcolor[HTML]{D3D3D3}\methodname   
& \cellcolor[HTML]{D3D3D3}\textbf{0.607}\scriptsize{${\pm}$ 0.018}
& \cellcolor[HTML]{D3D3D3}0.599\scriptsize{${\pm}$ 0.012} 
& \cellcolor[HTML]{D3D3D3}\textbf{0.622}\scriptsize{${\pm}$ 0.004}  & 
&
\cellcolor[HTML]{D3D3D3}\textbf{0.640}\scriptsize{${\pm}$ 0.013}
&\cellcolor[HTML]{D3D3D3}\textbf{0.651}\scriptsize{${\pm}$ 0.022}
& \cellcolor[HTML]{D3D3D3}\textbf{0.658}\scriptsize{${\pm}$ 0.018} & 
   \\
\midrule
\bottomrule
\end{tabular}
}
\caption{Performance comparison across data selection methods for \textit{graph density} shift. We use \textbf{bold}/\underline{underline} to indicate the 1st/2nd best results.  In most settings, \methodname\ achieves the best performance across datasets.}
\label{table:graph-select}
\vspace{-1em}
\end{table*}

\paragraph{Experiment Setup.} To test the effectiveness of these selection methods, we fix the backbone GNN models in use to 
train the data selected by each method.  Two popular GNN models are chosen, including GCN~\citep{kipf2016semi} and GIN~\citep{xu2018powerful} with default hyper-parameters following~\citet{zenggraph} and the corresponding original papers. We also consider results on GAT~\citep{velivckovic2017graph} and GraphSAGE~\citep{hamilton2017inductive}. Please see Appendix~\ref{appendix:exp-table-1-size-old}, \ref{appendix:exp-table-1-density-new} and \ref{appendix:exp-table-1-size-new} for more details.  Here we consider selection ratio $\tau \in [0.1, 0.2, 0.5]$. More details on the architectures and training protocol can be found in Appendix~\ref{appendix:gnn-settings}. 

\paragraph{Results.} As shown in Table~\ref{table:graph-select}, across all datasets, \methodname\ outperforms the baseline methods under different selection ratios. It is also worth noting that even with few selected data, \methodname\ can already achieve or excess GNN performance trained with full data, showing the importance of data quality in the source domain. The main reason is that because under distribution shift, a certain number of harmful graphs exist in training set. These samples can mislead the model and degrade generalization performance.




\subsection{\methodname\ as a GDA Method}
\label{ssec:as-gda-model}
We answer (\textbf{RQ2}) by directly compare the combination of \methodname\ and vanilla GNNs (including non-domain-adapted GCN, GIN, GAT and GraphSAGE) with state-of-the-art GDA models. We fix the sparsity ratio $\tau$ to $20\%$ across all selection methods. 

\paragraph{Experiment Setup.}
The four GDA methods we consider include AdaGCN~\citep{dai2022graph}, GRADE~\citep{wu2023non}, ASN~\citep{zhang2021adversarial} and UDAGCN~\citep{wu2020unsupervised}. We conduct  GDA experiments based on the codebase of OpenGDA~\citep{shi2023opengda}.
We include more details of model-specific parameter settings in Appendix~\ref{appendix:gda-settings}. We set the training set as the source domain and the validation set as the target domain. For GDD computation, we set label signal $c =0$ to match the requirement of unsupervised GDA methods. Results on \textit{graph size} shift can be found in Appendix~\ref{appendix:exp-table-2-size}.

\paragraph{Results.}
Results are in Table~\ref{table:gda-vs-graph-select}. For model-centric GDA methods trained with full training data, the severe domain shift prohibits these methods from learning rich knowledge to perform well on the test data in the target domain. Instead, \methodname\ finds the most useful data in the training set that results in simple GNN models with extraordinary classification accuracy while maintaining data efficiency. In addition, similar to the observation in Section~\ref{ssec:as-graph-selector}, \methodname\ selects non-trivial training data that outperforms other model-free data selection methods. In Appendix

\begin{table}[ht]
\centering
\scalebox{0.64}{
\begin{tabular}{@{\extracolsep{\fill}}ccccccccc}
\toprule
\midrule
&&&\multicolumn{6}{c}{\textbf{Dataset}}\\
\cmidrule{4-9}
\textbf{Type} & \textbf{Model} & \textbf{Data}  & \textsc{IMDB-BINARY}  & \textsc{IMDB-MULTI} & $\textsc{ MSRC\_21 }$ & \texttt{ogbg-molbace} & \texttt{ogbg-molbbbp} & \texttt{ogbg-molhiv} \\
\midrule
\multirow{4}{*}{GDA} &
AdaGCN  & Full &
0.808\scriptsize{${\pm}$ 0.015} &  0.073\scriptsize{${\pm}$ 0.000} & 0.319\scriptsize{${\pm}$ 0.032} &

0.607\scriptsize{${\pm}$ 0.068} &
\textbf{0.778}\scriptsize{${\pm}$ 0.002} &
0.428\scriptsize{${\pm}$ 0.011}\\
& GRADE  & Full &
0.822\scriptsize{${\pm}$ 0.012} &  0.123\scriptsize{${\pm}$ 0.061} & 0.804\scriptsize{${\pm}$ 0.011} &

\textbf{0.683}\scriptsize{${\pm}$ 0.016} &  0.489\scriptsize{${\pm}$ 0.005} & 0.564\scriptsize{${\pm}$ 0.005} \\

& ASN  & Full &
0.782\scriptsize{${\pm}$ 0.030} &   0.119\scriptsize{${\pm}$ 0.047} & 0.833\scriptsize{${\pm}$ 0.033} & 

0.580\scriptsize{${\pm}$ 0.065} &   0.476\scriptsize{${\pm}$ 0.027} & 0.516\scriptsize{${\pm}$ 0.021}  \\
& UDAGCN  & Full &
0.807\scriptsize{${\pm}$ 0.013} &  0.114\scriptsize{${\pm}$ 0.049} & 0.351\scriptsize{${\pm}$ 0.019} &

0.541\scriptsize{${\pm}$ 0.034} &  0.522\scriptsize{${\pm}$ 0.015} & 0.451\scriptsize{${\pm}$ 0.030} \\
\midrule
\multirow{16}{*}{Vanilla} &
\multirow{4}{*}{GCN}
& Random 20\% &
0.660\scriptsize{${\pm}$ 0.012} &   0.092\scriptsize{${\pm}$ 0.032} & 0.702\scriptsize{${\pm}$ 0.045}& 0.529\scriptsize{${\pm}$ 0.124} & 0.528\scriptsize{${\pm}$ 0.030} & 0.598\scriptsize{${\pm}$ 0.003}\\
& &LAVA 20\% & 
0.620\scriptsize{${\pm}$ 0.000}  &   0.092\scriptsize{${\pm}$ 0.032} &  0.819\scriptsize{${\pm}$ 0.011} & 0.541\scriptsize{${\pm}$ 0.067}  & 0.503\scriptsize{${\pm}$ 0.043} & 0.591\scriptsize{${\pm}$ 0.030}  \\
\cmidrule{3-9}
& & \cellcolor[HTML]{D3D3D3}\methodname\  20\%  &  
 \cellcolor[HTML]{D3D3D3}0.830\scriptsize{${\pm}$ 0.021} &   \cellcolor[HTML]{D3D3D3}0.349\scriptsize{${\pm}$ 0.323} &  \cellcolor[HTML]{D3D3D3}0.797\scriptsize{${\pm}$ 0.008}  & \cellcolor[HTML]{D3D3D3}0.585\scriptsize{${\pm}$ 0.074} & \cellcolor[HTML]{D3D3D3}0.571\scriptsize{${\pm}$ 0.035}  & \cellcolor[HTML]{D3D3D3}0.583\scriptsize{${\pm}$ 0.006}  \\
\cmidrule{2-3} \cmidrule{4-9}
& \multirow{4}{*}{GIN}
& Random 20\% &      
0.710\scriptsize{${\pm}$ 0.049} & 0.180\scriptsize{${\pm}$ 0.005} & 0.801\scriptsize{${\pm}$ 0.046} &0.622\scriptsize{${\pm}$ 0.028} & 0.480\scriptsize{${\pm}$ 0.041} & 0.590\scriptsize{${\pm}$ 0.033}\\
& & LAVA 20\%&      
0.778\scriptsize{${\pm}$ 0.045} &  0.170\scriptsize{${\pm}$ 0.009} & \underline{0.851}\scriptsize{${\pm}$ 0.012} & 0.655\scriptsize{${\pm}$ 0.067} & 0.644\scriptsize{${\pm}$ 0.021}  & \underline{0.638}\scriptsize{${\pm}$ 0.012} \\
\cmidrule{3-9}
& &  \cellcolor[HTML]{D3D3D3}\methodname\ 20\%&  
 \cellcolor[HTML]{D3D3D3}0.832\scriptsize{${\pm}$ 0.025} &   \cellcolor[HTML]{D3D3D3}0.266\scriptsize{${\pm}$ 0.133} & \cellcolor[HTML]{D3D3D3}\textbf{0.860}\scriptsize{${\pm}$ 0.007} & \cellcolor[HTML]{D3D3D3}\underline{0.662}\scriptsize{${\pm}$ 0.006} & \cellcolor[HTML]{D3D3D3}\underline{0.665}\scriptsize{${\pm}$ 0.053}  & \cellcolor[HTML]{D3D3D3}\textbf{0.644}\scriptsize{${\pm}$ 0.017} \\
 \cmidrule{2-3} \cmidrule{4-9}
& \multirow{4}{*}{GAT}
& Random 20\% &      
0.662\scriptsize{${\pm}$ 0.029} & 0.067\scriptsize{${\pm}$ 0.005} & 0.713\scriptsize{${\pm}$ 0.008} &0.472\scriptsize{${\pm}$ 0.034} & 0.486\scriptsize{${\pm}$ 0.041} & 0.593\scriptsize{${\pm}$ 0.012} \\
& & LAVA 20\%&      
0.835\scriptsize{${\pm}$ 0.002} &  \underline{0.790}\scriptsize{${\pm}$ 0.002} & 0.842\scriptsize{${\pm}$ 0.026} & 0.515\scriptsize{${\pm}$ 0.019} & 0.511\scriptsize{${\pm}$ 0.069} & 0.602\scriptsize{${\pm}$ 0.017} \\
\cmidrule{3-9}
& &  \cellcolor[HTML]{D3D3D3}\methodname\ 20\%&  
 \cellcolor[HTML]{D3D3D3}\textbf{0.858}\scriptsize{${\pm}$ 0.005} &   \cellcolor[HTML]{D3D3D3}\textbf{0.800}\scriptsize{${\pm}$ 0.133} & \cellcolor[HTML]{D3D3D3}\underline{0.857}\scriptsize{${\pm}$ 0.008} & \cellcolor[HTML]{D3D3D3}0.518\scriptsize{${\pm}$ 0.026} & \cellcolor[HTML]{D3D3D3}0.538\scriptsize{${\pm}$ 0.098} & \cellcolor[HTML]{D3D3D3}0.598\scriptsize{${\pm}$ 0.004} \\
 \cmidrule{2-3} \cmidrule{4-9}
& \multirow{4}{*}{GraphSAGE}
& Random 20\% &      0.738\scriptsize{${\pm}$ 0.059} & 0.132\scriptsize{${\pm}$ 0.036} & 0.731\scriptsize{${\pm}$ 0.027} &0.459\scriptsize{${\pm}$ 0.057} & 0.472\scriptsize{${\pm}$ 0.016} & 0.602\scriptsize{${\pm}$ 0.006}\\
& & LAVA 20\%&      
0.835\scriptsize{${\pm}$ 0.005} &  0.570\scriptsize{${\pm}$ 0.292} & 0.827\scriptsize{${\pm}$ 0.015} & 0.514\scriptsize{${\pm}$ 0.132} & 0.491\scriptsize{${\pm}$ 0.095} & 0.537\scriptsize{${\pm}$ 0.067} \\
\cmidrule{3-9}
& &  \cellcolor[HTML]{D3D3D3}\methodname\ 20\%&  
 \cellcolor[HTML]{D3D3D3}\underline{0.855}\scriptsize{${\pm}$ 0.005} &   \cellcolor[HTML]{D3D3D3}0.580\scriptsize{${\pm}$ 0.281} & \cellcolor[HTML]{D3D3D3}0.842\scriptsize{${\pm}$ 0.007} & \cellcolor[HTML]{D3D3D3}0.536\scriptsize{${\pm}$ 0.062} & \cellcolor[HTML]{D3D3D3}0.533\scriptsize{${\pm}$ 0.037} & \cellcolor[HTML]{D3D3D3}0.541\scriptsize{${\pm}$ 0.014} \\
\midrule
\bottomrule
\end{tabular}
}
\vspace{0.2em}
\caption{Performance comparison across GDA and vanilla methods for \textit{graph density} shift. We use \textbf{bold}/\underline{underline} to indicate the 1st/2nd best results. \methodname\ can consistently achieve top-2 performance across all datasets. 
}
\label{table:gda-vs-graph-select}
\end{table}

\subsection{\methodname\ as a Model-Free GDA Enhancer}
\label{ssec:as-gda-enhancer}
In order to answer (\textbf{RQ3}), we combine
\methodname\ with off-the-shelf GDA methods to study whether fewer but better training data can lead to even stronger adaptation performance.

\begin{table*}[t]
\centering
\resizebox{\textwidth}{!}{%
\begin{tabular}{@{\extracolsep{\fill}}cc|cccc|cccc}
\toprule
\midrule
\multirow{2.4}{*}{\textbf{Dataset}}&   \multirow{1}{*}{\textbf{GDA Method} $\rightarrow$\!\!\!\!\!\!} & \multicolumn{4}{c|}{\makecell{\textbf{AdaGCN}}}  & \multicolumn{4}{c}{\makecell{\textbf{GRADE}}} \\
\cmidrule{2-10} 
&\textbf{Selection Method} $\downarrow$\!\!& $\tau=10\%$ & $\tau=20\%$ & $\tau=50\%$ & Full & $\tau=10\%$ & $\tau=20\%$ & $\tau=50\%$ & Full \\
\midrule
\multirow{3}{*}{\textsc{IMDB-BINARY}}
& Random &       
\underline{0.763}\scriptsize{${\pm}$ 0.040} & \underline{0.773}\scriptsize{${\pm}$ 0.019} & \underline{0.798}\scriptsize{${\pm}$ 0.002} & 
\multirow{3}{*}{0.808\scriptsize{${\pm}$ 0.015}}
&    
\underline{0.683}\scriptsize{${\pm}$ 0.010} & \underline{0.792}\scriptsize{${\pm}$ 0.002} & \underline{0.780}\scriptsize{${\pm}$ 0.015} &

\multirow{3}{*}{0.822\scriptsize{${\pm}$ 0.012}}\\

& LAVA  &       
0.623\scriptsize{${\pm}$ 0.005} & 0.617\scriptsize{${\pm}$ 0.005} & 0.617\scriptsize{${\pm}$ 0.005} & 
 
&    
0.620\scriptsize{${\pm}$ 0.073} & 0.627\scriptsize{${\pm}$ 0.009} & 0.680\scriptsize{${\pm}$ 0.047} & 

\\

\cmidrule{2-5} \cmidrule{7-9}
&\cellcolor[HTML]{D3D3D3}\methodname   
&\cellcolor[HTML]{D3D3D3}\textbf{0.810}\scriptsize{${\pm}$ 0.032} 
&\cellcolor[HTML]{D3D3D3}\textbf{0.817}\scriptsize{${\pm}$ 0.024}
&\cellcolor[HTML]{D3D3D3}\textbf{0.822}\scriptsize{${\pm}$ 0.017} & 

&\cellcolor[HTML]{D3D3D3}\textbf{0.782}\scriptsize{${\pm}$ 0.009} 
&\cellcolor[HTML]{D3D3D3}\textbf{0.832}\scriptsize{${\pm}$ 0.013} 
&\cellcolor[HTML]{D3D3D3}\textbf{0.848}\scriptsize{${\pm}$ 0.009} & 

\\
\midrule
\multirow{3}{*}{\textsc{IMDB-MULTI}}
& Random &       
0.100\scriptsize{${\pm}$ 0.000} & 0.168\scriptsize{${\pm}$ 0.072} & 0.116\scriptsize{${\pm}$ 0.048} &

\multirow{3}{*}{0.073\scriptsize{${\pm}$ 0.000}}
&    
0.106\scriptsize{${\pm}$ 0.055} & 0.112\scriptsize{${\pm}$ 0.050} & 0.149\scriptsize{${\pm}$ 0.049} & 

\multirow{3}{*}{0.123\scriptsize{${\pm}$ 0.061}}\\

& LAVA  &       
\underline{0.191}\scriptsize{${\pm}$ 0.007} & \underline{0.183}\scriptsize{${\pm}$ 0.000} & \underline{0.184}\scriptsize{${\pm}$ 0.002} & 

&    
\textbf{0.183}\scriptsize{${\pm}$ 0.000} & \underline{0.189}\scriptsize{${\pm}$ 0.008} & \textbf{0.186}\scriptsize{${\pm}$ 0.003} &

   \\

\cmidrule{2-5} \cmidrule{7-9}
&\cellcolor[HTML]{D3D3D3}\methodname   
&\cellcolor[HTML]{D3D3D3}\textbf{0.333}\scriptsize{${\pm}$ 0.229} 
&\cellcolor[HTML]{D3D3D3}\textbf{0.373}\scriptsize{${\pm}$ 0.285} 
&\cellcolor[HTML]{D3D3D3}\textbf{0.391}\scriptsize{${\pm}$ 0.294} & 

&\cellcolor[HTML]{D3D3D3}\underline{0.131}\scriptsize{${\pm}$ 0.074} 
&\cellcolor[HTML]{D3D3D3}\textbf{0.386}\scriptsize{${\pm}$ 0.286} 
&\cellcolor[HTML]{D3D3D3}\underline{0.173}\scriptsize{${\pm}$ 0.100} & 
\\
\midrule
\multirow{3}{*}{\textsc{MSRC}\_21}
& Random &       
0.208\scriptsize{${\pm}$ 0.027} & 0.374\scriptsize{${\pm}$ 0.011} & 0.307\scriptsize{${\pm}$ 0.087} &

\multirow{3}{*}{0.319\scriptsize{${\pm}$ 0.032}}
&    
0.512\scriptsize{${\pm}$ 0.041} & 0.626\scriptsize{${\pm}$ 0.055} & 0.708\scriptsize{${\pm}$ 0.023} &

\multirow{3}{*}{0.804\scriptsize{${\pm}$ 0.011}}\\

& LAVA  &       
\underline{0.398}\scriptsize{${\pm}$ 0.004} & \textbf{0.456}\scriptsize{${\pm}$ 0.012} & \underline{0.480}\scriptsize{${\pm}$ 0.061} & 

&    
0.608\scriptsize{${\pm}$ 0.018} & 0.743\scriptsize{${\pm}$ 0.021} & 0.860\scriptsize{${\pm}$ 0.014} &

   \\

\cmidrule{2-5} \cmidrule{7-9}
&\cellcolor[HTML]{D3D3D3}\methodname   
&\cellcolor[HTML]{D3D3D3}\textbf{0.415}\scriptsize{${\pm}$ 0.112} 
&\cellcolor[HTML]{D3D3D3}\underline{0.406}\scriptsize{${\pm}$ 0.043} 
&\cellcolor[HTML]{D3D3D3}\textbf{0.532}\scriptsize{${\pm}$ 0.039} & 

&\cellcolor[HTML]{D3D3D3}\textbf{0.664}\scriptsize{${\pm}$ 0.021} 
&\cellcolor[HTML]{D3D3D3}\textbf{0.778}\scriptsize{${\pm}$ 0.015} 
&\cellcolor[HTML]{D3D3D3}\textbf{0.865}\scriptsize{${\pm}$ 0.027} & 

\\
\midrule
\multirow{3}{*}{\texttt{ogbg-molbace}}
& Random &       
0.436\scriptsize{${\pm}$ 0.021} & 0.485\scriptsize{${\pm}$ 0.038} & 0.565\scriptsize{${\pm}$ 0.085} &

\multirow{3}{*}{0.607\scriptsize{${\pm}$ 0.068}}
&    
0.538\scriptsize{${\pm}$ 0.023} & 0.554\scriptsize{${\pm}$ 0.025} & 0.611\scriptsize{${\pm}$ 0.015} &

\multirow{3}{*}{0.683\scriptsize{${\pm}$ 0.016}}\\

& LAVA  &       
\underline{0.574}\scriptsize{${\pm}$ 0.017} & 
\underline{0.589}\scriptsize{${\pm}$ 0.074}&
\textbf{0.607}\scriptsize{${\pm}$ 0.071}& 

&    
\underline{0.557}\scriptsize{${\pm}$ 0.055}& 
\textbf{0.653}\scriptsize{${\pm}$ 0.054} &
\underline{0.625}\scriptsize{${\pm}$ 0.015} &

   \\

\cmidrule{2-5} \cmidrule{7-9}
&\cellcolor[HTML]{D3D3D3}\methodname   
&\cellcolor[HTML]{D3D3D3}\textbf{0.598}\scriptsize{${\pm}$ 0.066}
&\cellcolor[HTML]{D3D3D3}\textbf{0.614}\scriptsize{${\pm}$ 0.043} 
&\cellcolor[HTML]{D3D3D3}\underline{0.572}\scriptsize{${\pm}$ 0.047}  & 

&\cellcolor[HTML]{D3D3D3}\textbf{0.599}\scriptsize{${\pm}$ 0.044} 
&\cellcolor[HTML]{D3D3D3}\underline{0.636}\scriptsize{${\pm}$ 0.035}
&\cellcolor[HTML]{D3D3D3}\textbf{0.634}\scriptsize{${\pm}$ 0.006}
& 

\\

\midrule
\multirow{3}{*}{\texttt{ogbg-molbbbp}}
& Random &       
0.494\scriptsize{${\pm}$ 0.014} & 0.469\scriptsize{${\pm}$ 0.031} & 0.527\scriptsize{${\pm}$ 0.035} &

\multirow{3}{*}{0.778\scriptsize{${\pm}$ 0.002}}
&    
0.511\scriptsize{${\pm}$ 0.032} & 0.433\scriptsize{${\pm}$ 0.001} & \underline{0.495}\scriptsize{${\pm}$ 0.041} &

\multirow{3}{*}{0.489\scriptsize{${\pm}$ 0.005}}\\

& LAVA  &       
0.583\scriptsize{${\pm}$ 0.075}& 
0.556\scriptsize{${\pm}$ 0.015} & 
\textbf{0.561}\scriptsize{${\pm}$ 0.040} & 

&    
0.549\scriptsize{${\pm}$ 0.013} & 
0\textbf{.579}\scriptsize{${\pm}$ 0.041} & 
\textbf{0.543}\scriptsize{${\pm}$ 0.013} &

   \\

\cmidrule{2-5} \cmidrule{7-9}
&\cellcolor[HTML]{D3D3D3}\methodname   
&\cellcolor[HTML]{D3D3D3}\textbf{0.593}\scriptsize{${\pm}$ 0.038}
&\cellcolor[HTML]{D3D3D3}\textbf{0.596}\scriptsize{${\pm}$ 0.022}
&\cellcolor[HTML]{D3D3D3}\underline{0.546}\scriptsize{${\pm}$ 0.026} & 

&\cellcolor[HTML]{D3D3D3}\textbf{0.582}\scriptsize{${\pm}$ 0.077}
&\cellcolor[HTML]{D3D3D3}\underline{0.503}\scriptsize{${\pm}$ 0.012}
&\cellcolor[HTML]{D3D3D3}0.490\scriptsize{${\pm}$ 0.006}& 

\\

\midrule
\multirow{3}{*}{\texttt{ogbg-molhiv}}
& Random &       
0.407\scriptsize{${\pm}$ 0.022} & \underline{0.429}\scriptsize{${\pm}$ 0.032} & 0.417\scriptsize{${\pm}$ 0.013} &

\multirow{3}{*}{0.428\scriptsize{${\pm}$ 0.011}}
&    
\underline{0.581}\scriptsize{${\pm}$ 0.008} & 0.544\scriptsize{${\pm}$ 0.001} & \underline{0.581}\scriptsize{${\pm}$ 0.009} &

\multirow{3}{*}{0.564\scriptsize{${\pm}$ 0.005}}\\

& LAVA  &       
\underline{0.453}\scriptsize{${\pm}$ 0.016} & 
0.428\scriptsize{${\pm}$ 0.013} & 
\underline{0.440}\scriptsize{${\pm}$ 0.003} & 

&    
0.566\scriptsize{${\pm}$ 0.011} & 
\underline{0.571}\scriptsize{${\pm}$ 0.005} & 
0.572\scriptsize{${\pm}$ 0.019} &

   \\

\cmidrule{2-5} \cmidrule{7-9}
&\cellcolor[HTML]{D3D3D3}\methodname   
&\cellcolor[HTML]{D3D3D3}\textbf{0.463}\scriptsize{${\pm}$ 0.041} 
&\cellcolor[HTML]{D3D3D3}\textbf{0.473}\scriptsize{${\pm}$ 0.021}
&\cellcolor[HTML]{D3D3D3}\textbf{0.447}\scriptsize{${\pm}$ 0.038} & 

&\cellcolor[HTML]{D3D3D3}\textbf{0.584}\scriptsize{${\pm}$ 0.012}
&\cellcolor[HTML]{D3D3D3}\textbf{0.589}\scriptsize{${\pm}$ 0.003}
&\cellcolor[HTML]{D3D3D3}\textbf{0.586}\scriptsize{${\pm}$ 0.003} & 

\\

\midrule
\midrule
\multirow{2.4}{*}{\textbf{Dataset}}&   \multirow{1}{*}{\textbf{GDA Method} $\rightarrow$\!\!\!\!\!\!} & \multicolumn{4}{c|}{\makecell{\textbf{ASN}}}  & \multicolumn{4}{c}{\makecell{\textbf{UDAGCN}}} \\
\cmidrule{2-10} 
&\textbf{Selection Method} $\downarrow$\!\!& $\tau=10\%$ & $\tau=20\%$ & $\tau=50\%$ & Full & $\tau=10\%$ & $\tau=20\%$ & $\tau=50\%$ & Full \\
\midrule
\multirow{3}{*}{\textsc{IMDB-BINARY}}
& Random &       
0.660\scriptsize{${\pm}$ 0.043} & \underline{0.707}\scriptsize{${\pm}$ 0.017} & \underline{0.678}\scriptsize{${\pm}$ 0.031} &

\multirow{3}{*}{0.782\scriptsize{${\pm}$ 0.030}}
&    
0.620\scriptsize{${\pm}$ 0.041} & \underline{0.763}\scriptsize{${\pm}$ 0.008} & \underline{0.823}\scriptsize{${\pm}$ 0.005} &

\multirow{3}{*}{0.807\scriptsize{${\pm}$ 0.013}}\\

& LAVA  &       
\underline{0.733}\scriptsize{${\pm}$ 0.081} & 0.620\scriptsize{${\pm}$ 0.000} & 0.620\scriptsize{${\pm}$ 0.000} &

&    
\underline{0.620}\scriptsize{${\pm}$ 0.000} & 0.643\scriptsize{${\pm}$ 0.033} & 0.620\scriptsize{${\pm}$ 0.000} &

   \\

\cmidrule{2-5} \cmidrule{7-9}
&\cellcolor[HTML]{D3D3D3}\methodname   
&\cellcolor[HTML]{D3D3D3}\textbf{0.748}\scriptsize{${\pm}$ 0.037} 
&\cellcolor[HTML]{D3D3D3}\textbf{0.818}\scriptsize{${\pm}$ 0.016} 
&\cellcolor[HTML]{D3D3D3}\textbf{0.855}\scriptsize{${\pm}$ 0.011} & 

&\cellcolor[HTML]{D3D3D3}\textbf{0.770}\scriptsize{${\pm}$ 0.023} 
&\cellcolor[HTML]{D3D3D3}\textbf{0.847}\scriptsize{${\pm}$ 0.012} 
&\cellcolor[HTML]{D3D3D3}\textbf{0.852}\scriptsize{${\pm}$ 0.005} & 
\\
\midrule
\multirow{3}{*}{\textsc{IMDB-MULTI}}
& Random &       
0.126\scriptsize{${\pm}$ 0.013} & 0.101\scriptsize{${\pm}$ 0.058} & 0.156\scriptsize{${\pm}$ 0.039} & 

\multirow{3}{*}{0.119\scriptsize{${\pm}$ 0.047} }
&    
\underline{0.150}\scriptsize{${\pm}$ 0.024} & 0.101\scriptsize{${\pm}$ 0.045} & 0.076\scriptsize{${\pm}$ 0.003} & 

\multirow{3}{*}{0.114\scriptsize{${\pm}$ 0.049}}\\

& LAVA  &       
\underline{0.183}\scriptsize{${\pm}$ 0.000} & \underline{0.183}\scriptsize{${\pm}$ 0.000} & \underline{0.190}\scriptsize{${\pm}$ 0.009} & 

&    
\textbf{0.183}\scriptsize{${\pm}$ 0.000} & \underline{0.183}\scriptsize{${\pm}$ 0.000} & \underline{0.182}\scriptsize{${\pm}$ 0.002} & 
\\

\cmidrule{2-5} \cmidrule{7-9}
&\cellcolor[HTML]{D3D3D3}\methodname   
&\cellcolor[HTML]{D3D3D3}\textbf{0.292}\scriptsize{${\pm}$ 0.352} 
&\cellcolor[HTML]{D3D3D3}\textbf{0.588}\scriptsize{${\pm}$ 0.286} 
&\cellcolor[HTML]{D3D3D3}\textbf{0.381}\scriptsize{${\pm}$ 0.301} & 

&\cellcolor[HTML]{D3D3D3}0.093\scriptsize{${\pm}$ 0.066} 
&\cellcolor[HTML]{D3D3D3}\textbf{0.554}\scriptsize{${\pm}$ 0.263} 
&\cellcolor[HTML]{D3D3D3}\textbf{0.339}\scriptsize{${\pm}$ 0.337} & 
\\
\midrule
\multirow{3}{*}{\textsc{MSRC\_21}}
& Random &       
0.421\scriptsize{${\pm}$ 0.026} & 0.673\scriptsize{${\pm}$ 0.011} & 0.661\scriptsize{${\pm}$ 0.032} & 

\multirow{3}{*}{0.833\scriptsize{${\pm}$ 0.033} }
&    
0.287\scriptsize{${\pm}$ 0.018} & 0.178\scriptsize{${\pm}$ 0.039} & 0.287\scriptsize{${\pm}$ 0.075} & 

\multirow{3}{*}{0.351\scriptsize{${\pm}$ 0.019}}\\

& LAVA  &       
\underline{0.635}\scriptsize{${\pm}$ 0.015} & \underline{0.746}\scriptsize{${\pm}$ 0.019} & \underline{0.868}\scriptsize{${\pm}$ 0.014} & 

&    
\textbf{0.453}\scriptsize{${\pm}$ 0.035} & \underline{0.447}\scriptsize{${\pm}$ 0.052} & \underline{0.623}\scriptsize{${\pm}$ 0.059} & 

\\

\cmidrule{2-5} \cmidrule{7-9}
&\cellcolor[HTML]{D3D3D3}\methodname   
&\cellcolor[HTML]{D3D3D3}\textbf{0.687}\scriptsize{${\pm}$ 0.048} 
&\cellcolor[HTML]{D3D3D3}\textbf{0.804}\scriptsize{${\pm}$ 0.021}
&\cellcolor[HTML]{D3D3D3}\textbf{0.904}\scriptsize{${\pm}$ 0.012} & 

&\cellcolor[HTML]{D3D3D3}\underline{0.444}\scriptsize{${\pm}$ 0.048} 
&\cellcolor[HTML]{D3D3D3}\textbf{0.453}\scriptsize{${\pm}$ 0.011} 
&\cellcolor[HTML]{D3D3D3}\textbf{0.664}\scriptsize{${\pm}$ 0.029} & 

\\

\midrule
\multirow{3}{*}{\texttt{ogbg-molbace}}
& Random &       
0.539\scriptsize{${\pm}$ 0.074} & \textbf{0.637}\scriptsize{${\pm}$ 0.009} & 0.507\scriptsize{${\pm}$ 0.061} &

\multirow{3}{*}{0.580\scriptsize{${\pm}$ 0.065}}
&    
0.478\scriptsize{${\pm}$ 0.037} & \textbf{0.581}\scriptsize{${\pm}$ 0.018} & \underline{0.513}\scriptsize{${\pm}$ 0.028} &

\multirow{3}{*}{0.541\scriptsize{${\pm}$ 0.034}}\\

& LAVA  &       
\underline{0.578}\scriptsize{${\pm}$ 0.036} & 
\underline{0.603}\scriptsize{${\pm}$ 0.009} & 
\underline{0.646}\scriptsize{${\pm}$ 0.050} & 

&    
\textbf{0.562}\scriptsize{${\pm}$ 0.039} & 
\underline{0.578}\scriptsize{${\pm}$ 0.015} & 
0.513\scriptsize{${\pm}$ 0.077} &

   \\

\cmidrule{2-5} \cmidrule{7-9}
&\cellcolor[HTML]{D3D3D3}\methodname   
&\cellcolor[HTML]{D3D3D3}\textbf{0.636}\scriptsize{${\pm}$ 0.022} 
&\cellcolor[HTML]{D3D3D3}0.596\scriptsize{${\pm}$ 0.053}
&\cellcolor[HTML]{D3D3D3}\textbf{0.651}\scriptsize{${\pm}$ 0.036} & 

&\cellcolor[HTML]{D3D3D3}\underline{0.533}\scriptsize{${\pm}$ 0.041}
&\cellcolor[HTML]{D3D3D3}0.565\scriptsize{${\pm}$ 0.039}
&\cellcolor[HTML]{D3D3D3}\textbf{0.531}\scriptsize{${\pm}$ 0.051} & 
\\

\midrule
\multirow{3}{*}{\texttt{ogbg-molbbbp}}
& Random &       
0.504\scriptsize{${\pm}$ 0.015} & 0.533\scriptsize{${\pm}$ 0.025} & 0.497\scriptsize{${\pm}$ 0.032} &

\multirow{3}{*}{0.476\scriptsize{${\pm}$ 0.027}}
&    
0.538\scriptsize{${\pm}$ 0.026} & 0.529\scriptsize{${\pm}$ 0.040} & 0.530\scriptsize{${\pm}$ 0.051} & 

\multirow{3}{*}{0.522\scriptsize{${\pm}$ 0.015}}\\

& LAVA  &       
 \underline{0.567}\scriptsize{${\pm}$ 0.040} & 
\textbf{0.616}\scriptsize{${\pm}$ 0.072} & 
\textbf{0.573}\scriptsize{${\pm}$ 0.035} & 

&    
 \underline{0.579}\scriptsize{${\pm}$ 0.031} & 
 \underline{0.547}\scriptsize{${\pm}$ 0.021} & 
 \underline{0.558}\scriptsize{${\pm}$ 0.021}&

   \\

\cmidrule{2-5} \cmidrule{7-9}
&\cellcolor[HTML]{D3D3D3}\methodname   
&\cellcolor[HTML]{D3D3D3}\textbf{0.573}\scriptsize{${\pm}$ 0.088}
&\cellcolor[HTML]{D3D3D3} \underline{0.596}\scriptsize{${\pm}$ 0.100}
&\cellcolor[HTML]{D3D3D3} \underline{0.535}\scriptsize{${\pm}$ 0.027}& 

&\cellcolor[HTML]{D3D3D3}\textbf{0.591}\scriptsize{${\pm}$ 0.040}
&\cellcolor[HTML]{D3D3D3}\textbf{0.575}\scriptsize{${\pm}$ 0.030}
&\cellcolor[HTML]{D3D3D3}\textbf{0.570}\scriptsize{${\pm}$ 0.009} & 

\\

\midrule
\multirow{3}{*}{\texttt{ogbg-molhiv}}
& Random &       
0.436\scriptsize{${\pm}$ 0.038} & 0.483\scriptsize{${\pm}$ 0.044} & 0.455\scriptsize{${\pm}$ 0.059} & 

\multirow{3}{*}{0.516\scriptsize{${\pm}$ 0.021}}
&    
0.453\scriptsize{${\pm}$ 0.015} & 0.406\scriptsize{${\pm}$ 0.015} & \textbf{0.464}\scriptsize{${\pm}$ 0.024} &

\multirow{3}{*}{0.451\scriptsize{${\pm}$ 0.030}}\\

& LAVA  &       
\underline{0.511}\scriptsize{${\pm}$ 0.018} & 
\textbf{0.540}\scriptsize{${\pm}$ 0.010} & 
\underline{0.482}\scriptsize{${\pm}$ 0.023} & 

&    
0.458\scriptsize{${\pm}$ 0.029} & 
0.427\scriptsize{${\pm}$ 0.007} & 
\underline{0.445}\scriptsize{${\pm}$ 0.018} &

   \\

\cmidrule{2-5} \cmidrule{7-9}
&\cellcolor[HTML]{D3D3D3}\methodname   
&\cellcolor[HTML]{D3D3D3}\textbf{0.527}\scriptsize{${\pm}$ 0.041}
&\cellcolor[HTML]{D3D3D3}\underline{0.491}\scriptsize{${\pm}$ 0.080} 
&\cellcolor[HTML]{D3D3D3}\textbf{0.491}\scriptsize{${\pm}$ 0.050} & 

&\cellcolor[HTML]{D3D3D3}\underline{0.453}\scriptsize{${\pm}$ 0.011}
&\cellcolor[HTML]{D3D3D3}\textbf{0.445}\scriptsize{${\pm}$ 0.018}
&\cellcolor[HTML]{D3D3D3}0.444\scriptsize{${\pm}$ 0.020}& 

\\
\midrule
\bottomrule
\end{tabular}
}
\caption{Performance comparison across combinations of GDA methods and data selection methods for \textit{graph density} shift. We use \textbf{bold}/\underline{underline} to indicate the 1st/2nd best results. \methodname\ achieves the best performance in most settings.} 
\label{table:graph-enhancer}
\end{table*}

\paragraph{Experiment Setup.}  Coupled with $10\%, 20\%, 50\%$ data selected by each model-free method (i.e., random, LAVA and \methodname),  four GDA baselines (considered in Section~\ref{ssec:as-gda-model}) are directly run on the shrunk training dataset with the same validation dataset under \textit{graph density} shift. For results on \textit{graph size} shift, we refer the readers to Appendix~\ref{appendix:exp-table-3-size}. 
\vspace{-1em}
\paragraph{Results.}
As shown in Table~\ref{table:graph-enhancer}, for most of the settings, \methodname\ selects data that is the most beneficial to adapting to the target set. Notably, across many settings, only $10\%$ or $20\%$ \methodname-selected data can outperform naively applying GDA methods on the full training data. This suggests that
\methodname\ can indeed improve \textit{data-efficiency} by promoting the quality of training data.  Furthermore, by effective data selection performed by \methodname, the difficulty of addressing the domain shift can be lowered significantly and thus result in better adaptation performance.

\subsection{Further Discussion}
\label{ssec:final-discussion}

\paragraph{LAVA vs \methodname.}
The modified version of LAVA utilizes LinearFGW to compare graphs and selects the training data with the smallest gradient value w.r.t. GDD. In contrast, \methodname\ aims at finding optimal training data that directly minimizes GDD, which has a complete different motivation and enjoys a theoretical justification. Empirically, we also observe the superiority of \methodname\ in most cases. Occasionally, LAVA achieve marginally better results than \methodname, which 
may be attributed to the approximation error of LinearFGW and thus over-optimization on GDD. 
\vspace{-0.8em}
\paragraph{Random vs \methodname.} From Table~\ref{table:gda-vs-graph-select}, we found \methodname\ occasionally underperforms random selection with GraphSAGE, possibly because the neighbor sampling strategy introduces noise into global representations, weakening the supervision signal even for well-chosen training graphs.
\vspace{-0.8em}

\paragraph{Selection Ratio vs GNN Performance.}
From Tables~\ref{table:graph-select} \&~\ref{table:graph-enhancer}, we find that a larger selection ratio may not always guarantee a better performance for selection methods including LAVA and \methodname. 
This is because, under severe distribution shift between domains, a larger portion of training data may actually contain patterns that are irrelevant or even harmful to the target domain. 

\paragraph{Effect of label signal $c$.} While we treat $c$ as a tunable hyper-parameter that can be optimized for different settings (i.e. various combinations of dataset, shift types and selection ratio), we empirically find that searching within $\{0,5\}$ can already lead to good performance throughout all experiments in this paper. We also provide additional experiments under a label-free setting (i.e. $c$ is forced to be equal to $0$) in Appendix ~\ref{appendix:exp-label-free}
.






\section{Related Work}
\label{sec:related-work}

\paragraph{Data Selection.}
Recent advancements in data selection have focused on optimizing  data utilization, mainly on text and vision data to facilitate efficient training for large language/image models~\citep{ killamsetty2021glister, killamsetty2021grad, mirzasoleiman2020coresets, bai2024multi, xie2024doremi, fan2023doge, ye2024data, liu2024regmix}. 
For general model-free data selection, LAVA~\citep{just2023lava} offers a learning-agnostic data valuation method
by seeking the data point that contributes the most to the distance between training and validation datasets. However, the paper studies predominantly on raw image datasets such as CIFAR-10~\citep{krizhevsky2009learning} and MNIST~\citep{lecun1998gradient}, where they already have high-quality pixel-value representations for computation. Unlike text or images, graphs lack a natural and uniform representation, making the development of model-free data selection more intricate. 
Tailored for graph-level tasks, graph dataset distillation is also a related topic. For example, \citet{jin2022condensing} and \citet{xu2023kernel} both propose to formulate a bi-level optimization problem to train a graph-level classifier. \citet{jain2025subsampling}, on the other hand, utilizes Tree Mover Distance~\cite{chuang2022tree} to conduct graph-level sub-sampling with theoretical guarantees.
However, these non-model-free methods might not be able to combat severe downstream distribution changes. 

\paragraph{Graph Domain Adaptation (GDA).}
For grpah classification, GDA focuses on transferring knowledge from a source domain with labeled graph to a target domain. Model-centric GDA methods relying on GNNs have been pivotal in this area. For instance, ~\citet{wu2020unsupervised} introduce
UDAGCN, which integrates domain adaptation with GNNs to align feature distributions between domains. AdaGCN~\cite{dai2022graph} addresses cross-network node classification leveraging adversarial domain adaptation to transfer label information between domains. \citet{wu2023non} explore cross-network transfer learning through Weisfeiler-Lehman graph isomorphism test and introduce the GRADE algorithm that minimizes distribution shift to perform adaptation.
ASN~\cite{zhang2021adversarial} explicitly separates domain-private and domain-shared information while capturing network consistency. More recently, \citet{liu2024rethinking} argue that excessive message passing exacerbates domain bias and propose A2GNN as a refined propagation scheme that disentangles transferable and domain-specific information; \citet{chen2025smoothness} highlight the critical role of graph smoothness, presenting TDSS that enforces cross-domain consistency through spectral alignment. Meanwhile, \citet{liu2024pairwise} introduce a pairwise alignment strategy that leverages node-level relational matching to enhance inter-domain correspondence. However, these approaches mostly focus on designing architectures or training procedures and often rely heavily on the assumption that provided data in the training set is already optimal for the task, which is often invalid in real-world scenarios.

\vspace{-0.5em}
\section{Conclusion}
\label{sec:conclusion}

We introduce \methodname, a model-free framework for graph classification that addresses distribution shift by solving a Graph Dataset Distance (GDD) minimization problem. By selecting the most beneficial data from the source domain, it offers a novel approach to improving GNN performance without relying on specific model predictions or training procedures. We also establish theoretical analysis on Fused Gromov–Wasserstein distance as a meaningful upper bound on GNN representation differences, and further justifies GDD as an optimization target to improve generalization performance.
Across multiple real-world datasets and shift types, \methodname\ consistently outperforms existing selection methods and GDA methods with better data efficiency. For future directions, we consider graph continual learning and multi-source domain adaptation.

\section*{Acknowledgement}
\thanks{This work is supported by NSF (2416070) and AFOSR (FA9550-24-1-0002). The content of the information in this document does not necessarily reflect the position or the policy of the Government, and no official endorsement should be inferred.  The U.S. Government is authorized to reproduce and distribute reprints for Government purposes notwithstanding any copyright notation here on.}




\newpage
\bibliographystyle{plainnat}
\bibliography{Reference}

\begin{thebibliography}{85}
\providecommand{\natexlab}[1]{#1}
\providecommand{\url}[1]{\texttt{#1}}
\expandafter\ifx\csname urlstyle\endcsname\relax
  \providecommand{\doi}[1]{doi: #1}\else
  \providecommand{\doi}{doi: \begingroup \urlstyle{rm}\Url}\fi

\bibitem[Altschuler et~al.(2017)Altschuler, Niles-Weed, and Rigollet]{altschuler2017near}
Jason Altschuler, Jonathan Niles-Weed, and Philippe Rigollet.
\newblock Near-linear time approximation algorithms for optimal transport via sinkhorn iteration.
\newblock \emph{Advances in neural information processing systems}, 30, 2017.

\bibitem[Alvarez-Melis and Fusi(2020)]{alvarez2020geometric}
David Alvarez-Melis and Nicolo Fusi.
\newblock Geometric dataset distances via optimal transport.
\newblock \emph{Advances in Neural Information Processing Systems}, 33:\penalty0 21428--21439, 2020.

\bibitem[Bai et~al.(2024)Bai, Yang, Wong, Peng, Zhuang, Zhang, Wu, Qiu, Zhang, Yuan, et~al.]{bai2024multi}
Tianyi Bai, Ling Yang, Zhen~Hao Wong, Jiahui Peng, Xinlin Zhuang, Chi Zhang, Lijun Wu, Jiantao Qiu, Wentao Zhang, Binhang Yuan, et~al.
\newblock Multi-agent collaborative data selection for efficient llm pretraining.
\newblock \emph{arXiv preprint arXiv:2410.08102}, 2024.

\bibitem[Bertsekas(1997)]{bertsekas1997nonlinear}
Dimitri~P Bertsekas.
\newblock Nonlinear programming.
\newblock \emph{Journal of the Operational Research Society}, 48\penalty0 (3):\penalty0 334--334, 1997.

\bibitem[Bevilacqua et~al.(2021)Bevilacqua, Zhou, and Ribeiro]{bevilacqua2021size}
Beatrice Bevilacqua, Yangze Zhou, and Bruno Ribeiro.
\newblock Size-invariant graph representations for graph classification extrapolations.
\newblock In \emph{International Conference on Machine Learning}, pages 837--851. PMLR, 2021.

\bibitem[Bongini et~al.(2022)Bongini, Pancino, Scarselli, and Bianchini]{bongini2022biognn}
Pietro Bongini, Niccol{\`o} Pancino, Franco Scarselli, and Monica Bianchini.
\newblock Biognn: how graph neural networks can solve biological problems.
\newblock In \emph{Artificial Intelligence and Machine Learning for Healthcare: Vol. 1: Image and Data Analytics}, pages 211--231. Springer, 2022.

\bibitem[Chen et~al.(2025)Chen, Ye, Wang, Zhang, Zhang, Wang, Zhang, and Zhuang]{chen2025smoothness}
Wei Chen, Guo Ye, Yakun Wang, Zhao Zhang, Libang Zhang, Daixin Wang, Zhiqiang Zhang, and Fuzhen Zhuang.
\newblock Smoothness really matters: A simple yet effective approach for unsupervised graph domain adaptation.
\newblock In \emph{Proceedings of the AAAI Conference on Artificial Intelligence}, pages 15875--15883, 2025.

\bibitem[Cheng et~al.(2023)Cheng, Han, Liu, Zhu, Gao, and Peng]{cheng2023multi}
Zhiyong Cheng, Sai Han, Fan Liu, Lei Zhu, Zan Gao, and Yuxin Peng.
\newblock Multi-behavior recommendation with cascading graph convolution networks.
\newblock In \emph{Proceedings of the ACM Web Conference 2023}, pages 1181--1189, 2023.

\bibitem[Chizari et~al.(2023)Chizari, Tajfar, and Moreno-Garc{\'\i}a]{chizari2023bias}
Nikzad Chizari, Keywan Tajfar, and Mar{\'\i}a~N Moreno-Garc{\'\i}a.
\newblock Bias assessment approaches for addressing user-centered fairness in gnn-based recommender systems.
\newblock \emph{Information}, 14\penalty0 (2):\penalty0 131, 2023.

\bibitem[Chu et~al.(2023)Chu, Jin, Wang, Zhang, Wang, Zhu, and Mei]{chu2023wasserstein}
Xu~Chu, Yujie Jin, Xin Wang, Shanghang Zhang, Yasha Wang, Wenwu Zhu, and Hong Mei.
\newblock Wasserstein barycenter matching for graph size generalization of message passing neural networks.
\newblock In \emph{International Conference on Machine Learning}, pages 6158--6184. PMLR, 2023.

\bibitem[Chuang and Jegelka(2022)]{chuang2022tree}
Ching-Yao Chuang and Stefanie Jegelka.
\newblock Tree mover's distance: Bridging graph metrics and stability of graph neural networks.
\newblock \emph{Advances in Neural Information Processing Systems}, 35:\penalty0 2944--2957, 2022.

\bibitem[Dai et~al.(2022)Dai, Wu, Xiao, Shen, and Wang]{dai2022graph}
Quanyu Dai, Xiao-Ming Wu, Jiaren Xiao, Xiao Shen, and Dan Wang.
\newblock Graph transfer learning via adversarial domain adaptation with graph convolution.
\newblock \emph{IEEE Transactions on Knowledge and Data Engineering}, 35\penalty0 (5):\penalty0 4908--4922, 2022.

\bibitem[Fan et~al.(2023)Fan, Pagliardini, and Jaggi]{fan2023doge}
Simin Fan, Matteo Pagliardini, and Martin Jaggi.
\newblock Doge: Domain reweighting with generalization estimation.
\newblock \emph{arXiv preprint arXiv:2310.15393}, 2023.

\bibitem[Fey and Lenssen(2019)]{fey2019fast}
Matthias Fey and Jan~Eric Lenssen.
\newblock Fast graph representation learning with pytorch geometric.
\newblock \emph{arXiv preprint arXiv:1903.02428}, 2019.

\bibitem[Gao et~al.(2022)Gao, Wang, He, and Li]{gao2022graph}
Chen Gao, Xiang Wang, Xiangnan He, and Yong Li.
\newblock Graph neural networks for recommender system.
\newblock In \emph{Proceedings of the fifteenth ACM international conference on web search and data mining}, pages 1623--1625, 2022.

\bibitem[Gasteiger et~al.(2019)Gasteiger, Bojchevski, and Günnemann]{gasteiger2018combining}
Johannes Gasteiger, Aleksandar Bojchevski, and Stephan Günnemann.
\newblock Combining neural networks with personalized pagerank for classification on graphs.
\newblock In \emph{International Conference on Learning Representations}, 2019.
\newblock URL \url{https://openreview.net/forum?id=H1gL-2A9Ym}.

\bibitem[Gretton et~al.(2012)Gretton, Borgwardt, Rasch, Sch{\"o}lkopf, and Smola]{gretton2012kernel}
Arthur Gretton, Karsten~M Borgwardt, Malte~J Rasch, Bernhard Sch{\"o}lkopf, and Alexander Smola.
\newblock A kernel two-sample test.
\newblock \emph{The Journal of Machine Learning Research}, 13\penalty0 (1):\penalty0 723--773, 2012.

\bibitem[Hamilton et~al.(2017)Hamilton, Ying, and Leskovec]{hamilton2017inductive}
Will Hamilton, Zhitao Ying, and Jure Leskovec.
\newblock Inductive representation learning on large graphs.
\newblock \emph{Advances in neural information processing systems}, 30, 2017.

\bibitem[He et~al.(2020)He, Deng, Wang, Li, Zhang, and Wang]{he2020lightgcn}
Xiangnan He, Kuan Deng, Xiang Wang, Yan Li, Yongdong Zhang, and Meng Wang.
\newblock Lightgcn: Simplifying and powering graph convolution network for recommendation.
\newblock In \emph{Proceedings of the 43rd International ACM SIGIR conference on research and development in Information Retrieval}, pages 639--648, 2020.

\bibitem[Hong et~al.(2024)Hong, Li, Wang, Lin, and Lu]{hong2024label}
Xiaobin Hong, Wenzhong Li, Chaoqun Wang, Mingkai Lin, and Sanglu Lu.
\newblock Label attentive distillation for gnn-based graph classification.
\newblock In \emph{Proceedings of the AAAI Conference on Artificial Intelligence}, 2024.

\bibitem[Hu et~al.(2020{\natexlab{a}})Hu, Fey, Zitnik, Dong, Ren, Liu, Catasta, and Leskovec]{hu2020ogb}
Weihua Hu, Matthias Fey, Marinka Zitnik, Yuxiao Dong, Hongyu Ren, Bowen Liu, Michele Catasta, and Jure Leskovec.
\newblock Open graph benchmark: Datasets for machine learning on graphs.
\newblock \emph{arXiv preprint arXiv:2005.00687}, 2020{\natexlab{a}}.

\bibitem[Hu et~al.(2020{\natexlab{b}})Hu, Fey, Zitnik, Dong, Ren, Liu, Catasta, and Leskovec]{hu2020open}
Weihua Hu, Matthias Fey, Marinka Zitnik, Yuxiao Dong, Hongyu Ren, Bowen Liu, Michele Catasta, and Jure Leskovec.
\newblock Open graph benchmark: Datasets for machine learning on graphs.
\newblock \emph{Advances in neural information processing systems}, 33:\penalty0 22118--22133, 2020{\natexlab{b}}.

\bibitem[Huang et~al.(2020)Huang, Xiao, Glass, Zitnik, and Sun]{huang2020skipgnn}
Kexin Huang, Cao Xiao, Lucas~M Glass, Marinka Zitnik, and Jimeng Sun.
\newblock Skipgnn: predicting molecular interactions with skip-graph networks.
\newblock \emph{Scientific reports}, 10\penalty0 (1):\penalty0 21092, 2020.

\bibitem[Jain et~al.(2025)Jain, Jegelka, Karmarkar, Ruiz, and Vitercik]{jain2025subsampling}
Mika~Sarkin Jain, Stefanie Jegelka, Ishani Karmarkar, Luana Ruiz, and Ellen Vitercik.
\newblock Subsampling graphs with gnn performance guarantees.
\newblock \emph{arXiv preprint arXiv:2502.16703}, 2025.

\bibitem[Jin et~al.(2022)Jin, Tang, Jiang, Li, Zhang, Tang, and Yin]{jin2022condensing}
Wei Jin, Xianfeng Tang, Haoming Jiang, Zheng Li, Danqing Zhang, Jiliang Tang, and Bing Yin.
\newblock Condensing graphs via one-step gradient matching.
\newblock In \emph{Proceedings of the 28th ACM SIGKDD Conference on Knowledge Discovery and Data Mining}, pages 720--730, 2022.

\bibitem[Just et~al.(2023)Just, Kang, Wang, Zeng, Ko, Jin, and Jia]{just2023lava}
Hoang~Anh Just, Feiyang Kang, Jiachen~T Wang, Yi~Zeng, Myeongseob Ko, Ming Jin, and Ruoxi Jia.
\newblock Lava: Data valuation without pre-specified learning algorithms.
\newblock \emph{arXiv preprint arXiv:2305.00054}, 2023.

\bibitem[Kantorovich(1942)]{kantorovich1942translocation}
Leonid~V Kantorovich.
\newblock On the translocation of masses.
\newblock In \emph{Dokl. Akad. Nauk. USSR (NS)}, volume~37, pages 199--201, 1942.

\bibitem[Killamsetty et~al.(2021{\natexlab{a}})Killamsetty, Durga, Ramakrishnan, De, and Iyer]{killamsetty2021grad}
Krishnateja Killamsetty, Sivasubramanian Durga, Ganesh Ramakrishnan, Abir De, and Rishabh Iyer.
\newblock Grad-match: Gradient matching based data subset selection for efficient deep model training.
\newblock In \emph{International Conference on Machine Learning}, pages 5464--5474. PMLR, 2021{\natexlab{a}}.

\bibitem[Killamsetty et~al.(2021{\natexlab{b}})Killamsetty, Sivasubramanian, Ramakrishnan, and Iyer]{killamsetty2021glister}
Krishnateja Killamsetty, Durga Sivasubramanian, Ganesh Ramakrishnan, and Rishabh Iyer.
\newblock Glister: Generalization based data subset selection for efficient and robust learning.
\newblock In \emph{Proceedings of the AAAI Conference on Artificial Intelligence}, volume~35, pages 8110--8118, 2021{\natexlab{b}}.

\bibitem[Kipf and Welling(2016)]{kipf2016semi}
Thomas~N Kipf and Max Welling.
\newblock Semi-supervised classification with graph convolutional networks.
\newblock \emph{arXiv preprint arXiv:1609.02907}, 2016.

\bibitem[Krizhevsky(2009)]{krizhevsky2009learning}
Alex Krizhevsky.
\newblock Learning multiple layers of features from tiny images.
\newblock Technical report, University of Toronto, 2009.
\newblock Technical Report.

\bibitem[LeCun et~al.(1998)LeCun, Bottou, Bengio, and Haffner]{lecun1998gradient}
Yann LeCun, L{\'e}on Bottou, Yoshua Bengio, and Patrick Haffner.
\newblock Gradient-based learning applied to document recognition.
\newblock \emph{Proceedings of the IEEE}, 86\penalty0 (11):\penalty0 2278--2324, 1998.

\bibitem[Li et~al.(2023)Li, Mei, and Ma]{li2023metadata}
Ting~Wei Li, Qiaozhu Mei, and Jiaqi Ma.
\newblock A metadata-driven approach to understand graph neural networks.
\newblock \emph{Advances in Neural Information Processing Systems}, 36:\penalty0 15320--15340, 2023.

\bibitem[Liu et~al.(2024{\natexlab{a}})Liu, Fang, Zhang, Gu, Zhou, Wang, and Bu]{liu2024rethinking}
Meihan Liu, Zeyu Fang, Zhen Zhang, Ming Gu, Sheng Zhou, Xin Wang, and Jiajun Bu.
\newblock Rethinking propagation for unsupervised graph domain adaptation.
\newblock In \emph{Proceedings of the AAAI Conference on Artificial Intelligence}, 2024{\natexlab{a}}.

\bibitem[Liu et~al.(2024{\natexlab{b}})Liu, Zheng, Muennighoff, Zeng, Dou, Pang, Jiang, and Lin]{liu2024regmix}
Qian Liu, Xiaosen Zheng, Niklas Muennighoff, Guangtao Zeng, Longxu Dou, Tianyu Pang, Jing Jiang, and Min Lin.
\newblock Regmix: Data mixture as regression for language model pre-training.
\newblock \emph{arXiv preprint arXiv:2407.01492}, 2024{\natexlab{b}}.

\bibitem[Liu et~al.(2024{\natexlab{c}})Liu, Zou, Zhao, and Li]{liu2024pairwise}
Shikun Liu, Deyu Zou, Han Zhao, and Pan Li.
\newblock Pairwise alignment improves graph domain adaptation.
\newblock In \emph{International Conference on Machine Learning}, pages 32552--32575. PMLR, 2024{\natexlab{c}}.

\bibitem[Liu et~al.(2023{\natexlab{a}})Liu, Wang, Ying, and Zhao]{liu2023muse}
Tianyu Liu, Yuge Wang, Rex Ying, and Hongyu Zhao.
\newblock Muse-gnn: Learning unified gene representation from multimodal biological graph data.
\newblock \emph{Advances in neural information processing systems}, 36:\penalty0 24661--24677, 2023{\natexlab{a}}.

\bibitem[Liu et~al.(2023{\natexlab{b}})Liu, Chen, and Wen]{liu2023survey}
Xingyu Liu, Juan Chen, and Quan Wen.
\newblock A survey on graph classification and link prediction based on gnn.
\newblock \emph{arXiv preprint arXiv:2307.00865}, 2023{\natexlab{b}}.

\bibitem[Liu et~al.(2023{\natexlab{c}})Liu, Zeng, Qiu, Yoo, Zhou, Xu, Zhu, Weldemariam, He, and Tong]{liu2023topological}
Zhining Liu, Zhichen Zeng, Ruizhong Qiu, Hyunsik Yoo, David Zhou, Zhe Xu, Yada Zhu, Kommy Weldemariam, Jingrui He, and Hanghang Tong.
\newblock Topological augmentation for class-imbalanced node classification, 2023{\natexlab{c}}.

\bibitem[Liu et~al.(2024{\natexlab{d}})Liu, Qiu, Zeng, Yoo, Zhou, Xu, Zhu, Weldemariam, He, and Tong]{liu2024class}
Zhining Liu, Ruizhong Qiu, Zhichen Zeng, Hyunsik Yoo, David Zhou, Zhe Xu, Yada Zhu, Kommy Weldemariam, Jingrui He, and Hanghang Tong.
\newblock Class-imbalanced graph learning without class rebalancing.
\newblock In \emph{Proceedings of the 41st International Conference on Machine Learning}, 2024{\natexlab{d}}.

\bibitem[Luo et~al.(2024)Luo, Xiao, Wang, Luo, Yuan, Ju, Liu, and Zhang]{luo2024rank}
Junyu Luo, Zhiping Xiao, Yifan Wang, Xiao Luo, Jingyang Yuan, Wei Ju, Langechuan Liu, and Ming Zhang.
\newblock Rank and align: towards effective source-free graph domain adaptation.
\newblock \emph{arXiv preprint arXiv:2408.12185}, 2024.

\bibitem[M{\'e}moli(2011)]{memoli2011gromov}
Facundo M{\'e}moli.
\newblock Gromov--wasserstein distances and the metric approach to object matching.
\newblock \emph{Foundations of computational mathematics}, 11:\penalty0 417--487, 2011.

\bibitem[Mirzasoleiman et~al.(2020)Mirzasoleiman, Bilmes, and Leskovec]{mirzasoleiman2020coresets}
Baharan Mirzasoleiman, Jeff Bilmes, and Jure Leskovec.
\newblock Coresets for data-efficient training of machine learning models.
\newblock In \emph{International Conference on Machine Learning}, pages 6950--6960. PMLR, 2020.

\bibitem[Morris et~al.(2020)Morris, Kriege, Bause, Kersting, Mutzel, and Neumann]{morris2020tudataset}
Christopher Morris, Nils~M Kriege, Franka Bause, Kristian Kersting, Petra Mutzel, and Marion Neumann.
\newblock Tudataset: A collection of benchmark datasets for learning with graphs.
\newblock \emph{arXiv preprint arXiv:2007.08663}, 2020.

\bibitem[Neumann et~al.(2016)Neumann, Garnett, Bauckhage, and Kersting]{neumann2016propagation}
Marion Neumann, Roman Garnett, Christian Bauckhage, and Kristian Kersting.
\newblock Propagation kernels: efficient graph kernels from propagated information.
\newblock \emph{Machine learning}, 102:\penalty0 209--245, 2016.

\bibitem[Nguyen and Tsuda(2023)]{nguyen2023linear}
Dai~Hai Nguyen and Koji Tsuda.
\newblock On a linear fused gromov-wasserstein distance for graph structured data.
\newblock \emph{Pattern Recognition}, 138:\penalty0 109351, 2023.

\bibitem[Qiu et~al.(2022)Qiu, Sun, and Yang]{qiu2022dimes}
Ruizhong Qiu, Zhiqing Sun, and Yiming Yang.
\newblock {DIMES}: A differentiable meta solver for combinatorial optimization problems.
\newblock In \emph{Advances in Neural Information Processing Systems}, volume~35, pages 25531--25546, 2022.

\bibitem[Qiu et~al.(2025{\natexlab{a}})Qiu, Li, Wei, He, and Tong]{qiu2025saffron}
Ruizhong Qiu, Gaotang Li, Tianxin Wei, Jingrui He, and Hanghang Tong.
\newblock Saffron-1: Safety inference scaling, 2025{\natexlab{a}}.

\bibitem[Qiu et~al.(2025{\natexlab{b}})Qiu, Li, Li, and Tong]{qiu2025graph}
Ruizhong Qiu, Ting-Wei Li, Gaotang Li, and Hanghang Tong.
\newblock Graph homophily booster: Rethinking the role of discrete features on heterophilic graphs.
\newblock \emph{arXiv preprint arXiv:2509.12530}, 2025{\natexlab{b}}.

\bibitem[Qiu et~al.(2025{\natexlab{c}})Qiu, Xu, Bao, and Tong]{qiu2025ask}
Ruizhong Qiu, Zhe Xu, Wenxuan Bao, and Hanghang Tong.
\newblock Ask, and it shall be given: {On} the {Turing} completeness of prompting.
\newblock In \emph{13th International Conference on Learning Representations}, 2025{\natexlab{c}}.

\bibitem[Qiu et~al.(2025{\natexlab{d}})Qiu, Zeng, Tong, Ezick, and Lott]{qiu2025efficient}
Ruizhong Qiu, Weiliang~Will Zeng, Hanghang Tong, James Ezick, and Christopher Lott.
\newblock How efficient is {LLM}-generated code? {A} rigorous \& high-standard benchmark.
\newblock In \emph{13th International Conference on Learning Representations}, 2025{\natexlab{d}}.

\bibitem[Rubner et~al.(2000)Rubner, Tomasi, and Guibas]{rubner2000earth}
Yossi Rubner, Carlo Tomasi, and Leonidas~J Guibas.
\newblock The earth mover's distance as a metric for image retrieval.
\newblock \emph{International journal of computer vision}, 40:\penalty0 99--121, 2000.

\bibitem[Shi et~al.(2023)Shi, Wang, Guo, Shao, Shen, and Cheng]{shi2023opengda}
Boshen Shi, Yongqing Wang, Fangda Guo, Jiangli Shao, Huawei Shen, and Xueqi Cheng.
\newblock Opengda: Graph domain adaptation benchmark for cross-network learning.
\newblock In \emph{Proceedings of the 32nd ACM International Conference on Information and Knowledge Management}, pages 5396--5400, 2023.

\bibitem[Sturm(2023)]{sturm2023space}
Karl-Theodor Sturm.
\newblock \emph{The space of spaces: curvature bounds and gradient flows on the space of metric measure spaces}, volume 290.
\newblock American Mathematical Society, 2023.

\bibitem[Sun et~al.(2019)Sun, Zhu, and Lin]{sun2019adagcn}
Ke~Sun, Zhanxing Zhu, and Zhouchen Lin.
\newblock Adagcn: Adaboosting graph convolutional networks into deep models.
\newblock \emph{arXiv preprint arXiv:1908.05081}, 2019.

\bibitem[Tang et~al.(2024)Tang, Luo, Yang, Luo, Zhang, and Cui]{tang2024multi}
Yuhao Tang, Junyu Luo, Ling Yang, Xiao Luo, Wentao Zhang, and Bin Cui.
\newblock Multi-view teacher with curriculum data fusion for robust unsupervised domain adaptation.
\newblock In \emph{2024 IEEE 40th International Conference on Data Engineering (ICDE)}, pages 2598--2611. IEEE, 2024.

\bibitem[Vayer et~al.(2020)Vayer, Chapel, Flamary, Tavenard, and Courty]{vayer2020fused}
Titouan Vayer, Laetitia Chapel, R{\'e}mi Flamary, Romain Tavenard, and Nicolas Courty.
\newblock Fused gromov-wasserstein distance for structured objects.
\newblock \emph{Algorithms}, 13\penalty0 (9):\penalty0 212, 2020.

\bibitem[Veli{\v{c}}kovi{\'c} et~al.(2017)Veli{\v{c}}kovi{\'c}, Cucurull, Casanova, Romero, Lio, and Bengio]{velivckovic2017graph}
Petar Veli{\v{c}}kovi{\'c}, Guillem Cucurull, Arantxa Casanova, Adriana Romero, Pietro Lio, and Yoshua Bengio.
\newblock Graph attention networks.
\newblock \emph{arXiv preprint arXiv:1710.10903}, 2017.

\bibitem[Wang et~al.(2013)Wang, Slep{\v{c}}ev, Basu, Ozolek, and Rohde]{wang2013linear}
Wei Wang, Dejan Slep{\v{c}}ev, Saurav Basu, John~A Ozolek, and Gustavo~K Rohde.
\newblock A linear optimal transportation framework for quantifying and visualizing variations in sets of images.
\newblock \emph{International journal of computer vision}, 101:\penalty0 254--269, 2013.

\bibitem[Wu et~al.(2023{\natexlab{a}})Wu, Courty, Jin, and Li]{wu2023improving}
Fang Wu, Nicolas Courty, Shuting Jin, and Stan~Z Li.
\newblock Improving molecular representation learning with metric learning-enhanced optimal transport.
\newblock \emph{Patterns}, 4\penalty0 (4), 2023{\natexlab{a}}.

\bibitem[Wu et~al.(2023{\natexlab{b}})Wu, He, and Ainsworth]{wu2023non}
Jun Wu, Jingrui He, and Elizabeth Ainsworth.
\newblock Non-iid transfer learning on graphs.
\newblock In \emph{Proceedings of the AAAI Conference on Artificial Intelligence}, volume~37, pages 10342--10350, 2023{\natexlab{b}}.

\bibitem[Wu et~al.(2020)Wu, Pan, Zhou, Chang, and Zhu]{wu2020unsupervised}
Man Wu, Shirui Pan, Chuan Zhou, Xiaojun Chang, and Xingquan Zhu.
\newblock Unsupervised domain adaptive graph convolutional networks.
\newblock In \emph{Proceedings of the web conference 2020}, pages 1457--1467, 2020.

\bibitem[Xie et~al.(2024)Xie, Pham, Dong, Du, Liu, Lu, Liang, Le, Ma, and Yu]{xie2024doremi}
Sang~Michael Xie, Hieu Pham, Xuanyi Dong, Nan Du, Hanxiao Liu, Yifeng Lu, Percy~S Liang, Quoc~V Le, Tengyu Ma, and Adams~Wei Yu.
\newblock Doremi: Optimizing data mixtures speeds up language model pretraining.
\newblock \emph{Advances in Neural Information Processing Systems}, 36, 2024.

\bibitem[Xie et~al.(2022)Xie, Liang, Gong, Qin, Ong, and He]{xie2022semisupervised}
Yu~Xie, Yanfeng Liang, Maoguo Gong, A~Kai Qin, Yew-Soon Ong, and Tiantian He.
\newblock Semisupervised graph neural networks for graph classification.
\newblock \emph{IEEE Transactions on Cybernetics}, 53\penalty0 (10):\penalty0 6222--6235, 2022.

\bibitem[Xu et~al.(2018)Xu, Hu, Leskovec, and Jegelka]{xu2018powerful}
Keyulu Xu, Weihua Hu, Jure Leskovec, and Stefanie Jegelka.
\newblock How powerful are graph neural networks?
\newblock \emph{arXiv preprint arXiv:1810.00826}, 2018.

\bibitem[Xu et~al.(2023)Xu, Chen, Pan, Chen, Das, Yang, and Tong]{xu2023kernel}
Zhe Xu, Yuzhong Chen, Menghai Pan, Huiyuan Chen, Mahashweta Das, Hao Yang, and Hanghang Tong.
\newblock Kernel ridge regression-based graph dataset distillation.
\newblock In \emph{Proceedings of the 29th ACM SIGKDD Conference on Knowledge Discovery and Data Mining}, pages 2850--2861, 2023.

\bibitem[Xu et~al.(2024{\natexlab{a}})Xu, Hassani, Zhang, Zeng, Yasunaga, Wang, Fu, Yao, Long, and Tong]{xu2024language}
Zhe Xu, Kaveh Hassani, Si~Zhang, Hanqing Zeng, Michihiro Yasunaga, Limei Wang, Dongqi Fu, Ning Yao, Bo~Long, and Hanghang Tong.
\newblock Language models are graph learners, 2024{\natexlab{a}}.

\bibitem[Xu et~al.(2024{\natexlab{b}})Xu, Qiu, Chen, Chen, Fan, Pan, Zeng, Das, and Tong]{xu2024discrete}
Zhe Xu, Ruizhong Qiu, Yuzhong Chen, Huiyuan Chen, Xiran Fan, Menghai Pan, Zhichen Zeng, Mahashweta Das, and Hanghang Tong.
\newblock Discrete-state continuous-time diffusion for graph generation.
\newblock In \emph{Advances in Neural Information Processing Systems}, volume~37, 2024{\natexlab{b}}.

\bibitem[Yanardag and Vishwanathan(2015)]{yanardag2015deep}
Pinar Yanardag and SVN Vishwanathan.
\newblock Deep graph kernels.
\newblock In \emph{Proceedings of the 21th ACM SIGKDD international conference on knowledge discovery and data mining}, pages 1365--1374, 2015.

\bibitem[Ye et~al.(2024)Ye, Liu, Sun, Zhou, Zhan, and Qiu]{ye2024data}
Jiasheng Ye, Peiju Liu, Tianxiang Sun, Yunhua Zhou, Jun Zhan, and Xipeng Qiu.
\newblock Data mixing laws: Optimizing data mixtures by predicting language modeling performance.
\newblock \emph{arXiv preprint arXiv:2403.16952}, 2024.

\bibitem[Yehudai et~al.(2021)Yehudai, Fetaya, Meirom, Chechik, and Maron]{yehudai2021local}
Gilad Yehudai, Ethan Fetaya, Eli Meirom, Gal Chechik, and Haggai Maron.
\newblock From local structures to size generalization in graph neural networks.
\newblock In \emph{International Conference on Machine Learning}, pages 11975--11986. PMLR, 2021.

\bibitem[Yin et~al.(2024)Yin, Wang, Chen, Shen, Xiong, Gu, and Luo]{yin2024dream}
Nan Yin, Mengzhu Wang, Zhenghan Chen, Li~Shen, Huan Xiong, Bin Gu, and Xiao Luo.
\newblock Dream: Dual structured exploration with mixup for open-set graph domain adaption.
\newblock In \emph{The Twelfth International Conference on Learning Representations}, 2024.

\bibitem[Ying et~al.(2018)Ying, He, Chen, Eksombatchai, Hamilton, and Leskovec]{ying2018graph}
Rex Ying, Ruining He, Kaifeng Chen, Pong Eksombatchai, William~L Hamilton, and Jure Leskovec.
\newblock Graph convolutional neural networks for web-scale recommender systems.
\newblock In \emph{Proceedings of the 24th ACM SIGKDD international conference on knowledge discovery \& data mining}, pages 974--983, 2018.

\bibitem[Yoo et~al.(2024)Yoo, Zeng, Kang, Qiu, Zhou, Liu, Wang, Xu, Chan, and Tong]{yoo2024ensuring}
Hyunsik Yoo, Zhichen Zeng, Jian Kang, Ruizhong Qiu, David Zhou, Zhining Liu, Fei Wang, Charlie Xu, Eunice Chan, and Hanghang Tong.
\newblock Ensuring user-side fairness in dynamic recommender systems.
\newblock In \emph{Proceedings of the ACM on Web Conference 2024}, pages 3667--3678, 2024.

\bibitem[Yoo et~al.(2025{\natexlab{a}})Yoo, Kang, Qiu, Xu, Wang, and Tong]{yoo2025embracing}
Hyunsik Yoo, SeongKu Kang, Ruizhong Qiu, Charlie Xu, Fei Wang, and Hanghang Tong.
\newblock Embracing plasticity: Balancing stability and plasticity in continual recommender systems.
\newblock In \emph{Proceedings of the 48th International ACM SIGIR Conference on Research and Development in Information Retrieval}, 2025{\natexlab{a}}.

\bibitem[Yoo et~al.(2025{\natexlab{b}})Yoo, Qiu, Xu, Wang, and Tong]{yoo2025generalizable}
Hyunsik Yoo, Ruizhong Qiu, Charlie Xu, Fei Wang, and Hanghang Tong.
\newblock Generalizable recommender system during temporal popularity distribution shifts.
\newblock In \emph{Proceedings of the 31st ACM SIGKDD Conference on Knowledge Discovery and Data Mining}, 2025{\natexlab{b}}.

\bibitem[Yu et~al.(2025)Yu, Zeng, Yan, Ying, Srikant, and Tong]{yu2025joint}
Qi~Yu, Zhichen Zeng, Yuchen Yan, Lei Ying, R~Srikant, and Hanghang Tong.
\newblock Joint optimal transport and embedding for network alignment.
\newblock In \emph{Proceedings of the ACM on Web Conference 2025}, pages 2064--2075, 2025.

\bibitem[Zeng et~al.(2023{\natexlab{a}})Zeng, Zhang, Xia, and Tong]{zeng2023parrot}
Zhichen Zeng, Si~Zhang, Yinglong Xia, and Hanghang Tong.
\newblock Parrot: Position-aware regularized optimal transport for network alignment.
\newblock In \emph{Proceedings of the ACM web conference 2023}, pages 372--382, 2023{\natexlab{a}}.

\bibitem[Zeng et~al.(2023{\natexlab{b}})Zeng, Zhu, Xia, Zeng, and Tong]{zeng2023generative}
Zhichen Zeng, Ruike Zhu, Yinglong Xia, Hanqing Zeng, and Hanghang Tong.
\newblock Generative graph dictionary learning.
\newblock In \emph{International Conference on Machine Learning}, pages 40749--40769. PMLR, 2023{\natexlab{b}}.

\bibitem[Zeng et~al.(2024{\natexlab{a}})Zeng, Du, Zhang, Xia, Liu, and Tong]{zeng2024hierarchical}
Zhichen Zeng, Boxin Du, Si~Zhang, Yinglong Xia, Zhining Liu, and Hanghang Tong.
\newblock Hierarchical multi-marginal optimal transport for network alignment.
\newblock In \emph{Proceedings of the AAAI Conference on Artificial Intelligence}, volume~38, pages 16660--16668, 2024{\natexlab{a}}.

\bibitem[Zeng et~al.(2024{\natexlab{b}})Zeng, Qiu, Xu, Liu, Yan, Wei, Ying, He, and Tong]{zenggraph}
Zhichen Zeng, Ruizhong Qiu, Zhe Xu, Zhining Liu, Yuchen Yan, Tianxin Wei, Lei Ying, Jingrui He, and Hanghang Tong.
\newblock Graph mixup on approximate gromov--wasserstein geodesics.
\newblock In \emph{Forty-first International Conference on Machine Learning}, 2024{\natexlab{b}}.

\bibitem[Zeng et~al.(2025)Zeng, Qiu, Bao, Wei, Lin, Yan, Abdelzaher, Han, and Tong]{zeng2025pave}
Zhichen Zeng, Ruizhong Qiu, Wenxuan Bao, Tianxin Wei, Xiao Lin, Yuchen Yan, Tarek~F Abdelzaher, Jiawei Han, and Hanghang Tong.
\newblock Pave your own path: Graph gradual domain adaptation on fused gromov-wasserstein geodesics.
\newblock \emph{arXiv preprint arXiv:2505.12709}, 2025.

\bibitem[Zhang et~al.(2021)Zhang, Du, Xie, and Wang]{zhang2021adversarial}
Xiaowen Zhang, Yuntao Du, Rongbiao Xie, and Chongjun Wang.
\newblock Adversarial separation network for cross-network node classification.
\newblock In \emph{Proceedings of the 30th ACM international conference on information \& knowledge management}, pages 2618--2626, 2021.

\bibitem[Zhou et~al.(2022)Zhou, Kutyniok, and Ribeiro]{zhou2022ood}
Yangze Zhou, Gitta Kutyniok, and Bruno Ribeiro.
\newblock Ood link prediction generalization capabilities of message-passing gnns in larger test graphs.
\newblock \emph{Advances in Neural Information Processing Systems}, 35:\penalty0 20257--20272, 2022.

\bibitem[Zitnik et~al.(2024)Zitnik, Li, Wells, Glass, Morselli~Gysi, Krishnan, Murali, Radivojac, Roy, Baudot, et~al.]{zitnik2024current}
Marinka Zitnik, Michelle~M Li, Aydin Wells, Kimberly Glass, Deisy Morselli~Gysi, Arjun Krishnan, T\_M Murali, Predrag Radivojac, Sushmita Roy, Ana{\"\i}s Baudot, et~al.
\newblock Current and future directions in network biology, 2024.

\end{thebibliography}


\newpage
\appendix

\clearpage
\section*{NeurIPS Paper Checklist}

\begin{enumerate}

\item {\bf Claims}
    \item[] Question: Do the main claims made in the abstract and introduction accurately reflect the paper's contributions and scope?
    \item[] Answer: \answerYes{} 
    \item[] Justification: Our main contributions are summarized in five bullet points in Section~\ref{sec:intro} and we point the readers to corresponding sections to see justification (including theorems, methodologies and empirical results).

\item {\bf Limitations}
    \item[] Question: Does the paper discuss the limitations of the work performed by the authors?
    \item[] Answer: \answerYes{} 
    \item[] Justification: We discuss potential limitations of our presented work in Appendix~\ref{appendix:limit}.

\item {\bf Theory assumptions and proofs}
    \item[] Question: For each theoretical result, does the paper provide the full set of assumptions and a complete (and correct) proof?
    \item[] Answer: \answerYes{} 
    \item[] Justification: We provide all proofs and assumptions of our theorem results in Appendix~\ref{appendix:all-proofs}.

    \item {\bf Experimental result reproducibility}
    \item[] Question: Does the paper fully disclose all the information needed to reproduce the main experimental results of the paper to the extent that it affects the main claims and/or conclusions of the paper (regardless of whether the code and data are provided or not)?
    \item[] Answer: \answerYes{} 
    \item[] Justification: We provide all experimental details needed in Section~\ref{sec:exp}, Appendix~\ref{appendix:gnn-settings} and Appendix~\ref{appendix:gda-settings}.

\item {\bf Open access to data and code}
    \item[] Question: Does the paper provide open access to the data and code, with sufficient instructions to faithfully reproduce the main experimental results, as described in supplemental material?
    \item[] Answer: \answerYes{} 
    \item[] Justification: We will provide the code package during submission and make the code available upon acceptance.

\item {\bf Experimental setting/details}
    \item[] Question: Does the paper specify all the training and test details (e.g., data splits, hyperparameters, how they were chosen, type of optimizer, etc.) necessary to understand the results?
    \item[] Answer: \answerYes{} 
    \item[] Justification: We provide all experimental details needed in Section~\ref{sec:exp}, Appendix~\ref{appendix:gnn-settings} and Appendix~\ref{appendix:gda-settings}.

\item {\bf Experiment statistical significance}
    \item[] Question: Does the paper report error bars suitably and correctly defined or other appropriate information about the statistical significance of the experiments?
    \item[] Answer: \answerYes{} 
    \item[] Justification: We include std in all of our main tables (see Section~\ref{sec:exp}).

\item {\bf Experiments compute resources}
    \item[] Question: For each experiment, does the paper provide sufficient information on the computer resources (type of compute workers, memory, time of execution) needed to reproduce the experiments?
    \item[] Answer: \answerYes{} 
    \item[] Justification: We specify the compute resources in Appendix~\ref{appendix:gnn-settings}.
    
\item {\bf Code of ethics}
    \item[] Question: Does the research conducted in the paper conform, in every respect, with the NeurIPS Code of Ethics \url{https://neurips.cc/public/EthicsGuidelines}?
    \item[] Answer: \answerYes{} 
    \item[] Justification: We confirm that this work is conducted with the NeurIPS Code of Ethics.

\item {\bf Broader impacts}
    \item[] Question: Does the paper discuss both potential positive societal impacts and negative societal impacts of the work performed?
    \item[] Answer: \answerYes{} 
    \item[] Justification: We provide related discussion in Section~\ref{appendix:impact}.
    
\item {\bf Safeguards}
    \item[] Question: Does the paper describe safeguards that have been put in place for responsible release of data or models that have a high risk for misuse (e.g., pretrained language models, image generators, or scraped datasets)?
    \item[] Answer: \answerNA{} 
    \item[] Justification: We confirm that this work does not pose safety risks.

\item {\bf Licenses for existing assets}
    \item[] Question: Are the creators or original owners of assets (e.g., code, data, models), used in the paper, properly credited and are the license and terms of use explicitly mentioned and properly respected?
    \item[] Answer: \answerYes{} 
    \item[] Justification: We provide the related details in Appendix~\ref{appendix:dataset-details}.

\item {\bf New assets}
    \item[] Question: Are new assets introduced in the paper well documented and is the documentation provided alongside the assets?
    \item[] Answer: \answerNo{} 
    \item[] Justification: We did not introduce new assets in this paper.

\item {\bf Crowdsourcing and research with human subjects}
    \item[] Question: For crowdsourcing experiments and research with human subjects, does the paper include the full text of instructions given to participants and screenshots, if applicable, as well as details about compensation (if any)? 
    \item[] Answer: \answerNA{} 
    \item[] Justification: This paper does not involve human subjects.

\item {\bf Institutional review board (IRB) approvals or equivalent for research with human subjects}
    \item[] Question: Does the paper describe potential risks incurred by study participants, whether such risks were disclosed to the subjects, and whether Institutional Review Board (IRB) approvals (or an equivalent approval/review based on the requirements of your country or institution) were obtained?
    \item[] Answer: \answerNA{} 
    \item[] Justification: This paper does not involve human subjects.

\item {\bf Declaration of LLM usage}
    \item[] Question: Does the paper describe the usage of LLMs if it is an important, original, or non-standard component of the core methods in this research? Note that if the LLM is used only for writing, editing, or formatting purposes and does not impact the core methodology, scientific rigorousness, or originality of the research, declaration is not required.
    \item[] Answer: \answerNA{} 
    \item[] Justification: LLM is not an important part of the core methods in this paper.

\end{enumerate}

\newpage

\section*{\centering \textbf{\Large Appendix}}

The content of appendix is organized as follows:

\begin{enumerate}[leftmargin=1em]
\item \textbf{Algorithms}:
\begin{itemize}[leftmargin=1em]
    \item Appendix~\ref{appendix:linearfgw} talks about the details of LinearFGW~\citep{nguyen2023linear} that we omit in the main text. We summarize the overall procedure of LinearFGW in Algorithm~\ref{alg:linearfgw}.
    \item Appendix~\ref{appendix:gdd-computation} goes through the steps to compute Graph Dataset Distance (GDD). The entire procedure is included in Algorithm~\ref{alg:gdd-computation}.
    \item Appendix~\ref{appendix:great} summarizes the submodule \optmethodname\ used in our main algorithm (Algorithm~\ref{alg:alt-got-shrinkage}).
    \item 
    Appendix~\ref{appendix:gradate} summarizes our main algorithm \methodname\ (Algorithm~\ref{alg:overall-alg}).
\end{itemize}

\item \textbf{Proofs}:
\begin{itemize}[leftmargin=1.5em]
\item Appendix~\ref{appendix:all-proofs} provides the proofs for all the theorems in the main text.
\end{itemize}

\item \textbf{Additional Experiments}:
\begin{itemize}[leftmargin=1.5em]
\item Appendix~\ref{appendix:exp-table-1-size-old} compares \methodname\ with other data selection methods under \textit{graph size} shift with additional GNN backbones.
\item Appendix~\ref{appendix:exp-table-1-density-new} compares \methodname\ with other data selection methods under \textit{graph density} shift.
\item Appendix~\ref{appendix:exp-table-1-size-new} compares \methodname\ with other data selection methods under \textit{graph size} shift with additional GNN backbones.
\item Appendix~\ref{appendix:exp-table-2-size} compares the combination of \methodname\ and vanilla GNNs with other GDA methods under \textit{graph size} shift.
\item Appendix~\ref{appendix:exp-table-3-size} compares the combination of GDA methods and \methodname\ against other data selection methods under \textit{graph size} shift.
\item Appendix~\ref{appendix:exp-label-free} ablates on the validation-label-free setting.
\item Appendix~\ref{appendix:exp-additional-gnn} includes results on additional graph backbones.
\item Appendix~\ref{appendix:exp-additional-gda} includes results on additional GDA methods.

\end{itemize}

\item \textbf{Discussions}:
\begin{itemize}[leftmargin=1.5em]
\item Appendix~\ref{appendix:tmd-mmd} discusses FGW and Graph Dataset Distance (GDD) in relation to prior notions such as Tree-Mover Distance (TMD)~\citep{chuang2022tree} and Maximum Mean Discrepancy (MMD)~\citep{gretton2012kernel}.
\item Appendix~\ref{appendix:limit} discusses potential limitations and future direction of our work.
\end{itemize}
\item \textbf{Reproducibility}:
\begin{itemize}[leftmargin=1.5em]
\item Appendix~\ref{appendix:dataset-details} provides the dataset statistics and licenses used in this work.
\item Appendix~\ref{appendix:gnn-settings} introduces the overall settings of GNN to use for the graph data selection evaluation, including the model we select and the training protocols.
\item Appendix~\ref{appendix:gda-settings} includes GDA method-specific parameter settings, where we follow the default settings of the OpenGDA package~\cite{shi2023opengda}.
\end{itemize}
\item \textbf{Others}:
\begin{itemize}[leftmargin=1.5em]
\item Appendix~\ref{appendix:prelim-gda} includes problem definition of Graph Domain Adaptation (GDA).
\item Appendix~\ref{appendix:additional-related-work} provides additional related work.
\item Appendix~\ref{appendix:empirical-runtime} contains the empirical runtime of \methodname.
\item Appendix~\ref{appendix:ecdf-plots} includes the ECDF plots of graph properties across datasets.
\end{itemize}
\end{enumerate}

    
    

\section{Details of LinearFGW~(Algorithm \ref{alg:linearfgw})}
\label{appendix:linearfgw}
Formally, consider a set of $N$ graphs $\mathcal{D}=\{\mathcal{G}_i\}_{i=1}^{N}$, where each $\mathcal{G} = (\mathbf{A}, \mathbf{X}) \in \mathcal{D} $ represents an attributed graph with adjacency matrix $\mathbf{A} \in \mathbb{R}^{n \times n}$ and node feature matrix $\mathbf{X} \in \mathbb{R}^{n \times d}$. Note that $n$ is the number of nodes of $\mathcal{G}$ and $d$ is the dimension of node features. LinearFGW first requires a reference graph $\overline{\mathcal{G}} = (\overline{\mathbf{A}}, \overline{\mathbf{X}})$ where $\overline{\mathbf{A}} \in \mathbb{R}^{\bar{n} \times \bar{n}}$,  $\overline{\mathbf{X}} \in \mathbb{R}^{\bar{n}\times d}$, $\bar{n}$ is the number of nodes and $d$ is the dimension of node features. 
Typically, $\overline{\mathcal{G}}$ is obtained by solving an \textit{FGW barycenter problem}, which aims to find a ``center'' graph that has the minimum sum of pairwise graph distances over the entire graph set $\mathcal{D}$. 

Following the notation used in Section~\ref{prelim:fgw-distance}, we define the inter-graph distance matrix $\mathbf{M}_{(\mathcal{G}_1, \mathcal{G}_2)}$ between any pair of graphs (named as $\mathcal{G}_1 = \{  
\mathbf{A}_1 , \mathbf{X}_1, \mathbf{p}_1  \}$ and $\mathcal{G}_2 = \{ 
\mathbf{A}_2 , \mathbf{X}_2, \mathbf{p}_2  \}$) to be the pairwise Euclidean distance of node features. Namely, $\mathbf{M}_{(\mathcal{G}_1, \mathcal{G}_2)}  = [\| \mathbf{X}_1[i]-\mathbf{X}_2[j]  \|]_{ij}$. In addition, the intra-graph similarity matrix is chosen to be defined as their corresponding adjacency matrices (i.e., $\mathbf{C}_{\mathcal{G}_1} = \mathbf{A}_1$ and $\mathbf{C}_{\mathcal{G}_2} = \mathbf{A}_2$). Together with uniform distributions\footnote{Since we have no prior over the node importance in either graphs, the probability simplex will typically be set as uniform.} $\mathbf{p}_{\mathcal{G}_1} = \frac{\mathbf{1}_{n_1}}{n_1}$ and $\mathbf{p}_{\mathcal{G}_2} = \frac{\mathbf{1}_{n_2}}{n_2}$ over the nodes of $\mathcal{G}_1$ and $\mathcal{G}_2$ (with sizes $n_1$ and $n_2$), correspondingly, the \textit{FGW barycenter problem}\footnote{The optimization algorithm for solving this problem is omitted. Please refer to the original paper for more details.} can be formulated as follows:
\begin{align}
\label{eq:fgw-barycenter}
    \overline{\mathcal{G}}
    = \arg \min_\mathcal{G} 
    &\sum_{i=1}^N \text{FGW}(\mathbf{M}_{(\mathcal{G}_i , \mathcal{G})}, \mathbf{C}_{\mathcal{G}_i},
    \mathbf{C}_{\mathcal{G}},
    \mathbf{p}_{\mathcal{G}_i},
    \mathbf{p}_{\mathcal{G}},
    \alpha
    ), 
\end{align}
where $\alpha \in [0,1]$ is the pre-defined trade-off parameter.

After calculating the reference graph $\overline{\mathcal{G}}$, we then obtain $N$ optimal transport plans $\{\boldsymbol\pi_i \}_{i \in [n]}$ as the solutions by computing $\text{FGW}(\mathcal{G}, \overline{\mathcal{G}})$ for each $\mathcal{G} \in \mathcal{D}$ (via solving Equation~(\ref{eq:prelim-linearfgw})). Then, the barycentric projection~\citep{nguyen2023linear} of each graph's node edge with respect to $\overline{\mathcal{G}}$ can be written as
\begin{equation}
\label{eq:baryproj-node}
    \mathbf{T}_{\text{node}}(\boldsymbol\pi_i) = \bar{n}\cdot \boldsymbol\pi_i \mathbf{X}_i \in \mathbb{R}^{\bar{n}\times d},
\end{equation}
\begin{equation}
\label{eq:baryproj-edge}
    \mathbf{T}_{\text{edge}}(\boldsymbol\pi_i) = \bar{n}^2\cdot \boldsymbol\pi_i \mathbf{C}_i \boldsymbol\pi_i^\top \in \mathbb{R}^{\bar{n}\times \bar{n}}.
\end{equation}

Finally, we can define the LinearFGW distance based on these barycentric projections. Namely, for any pair of graphs $(\mathcal{G}_i, \mathcal{G}_j)$, we define a distance metric $d_{\text{LinearFGW}}(\cdot,\cdot)$ over the graph set  $\mathcal{D}$:
\begin{align}
\label{eq:method-linearfgw}
    d_{\text{LinearFGW}} &(\mathcal{G}_i, \mathcal{G}_j) \notag =\\
    &(1-\alpha) \|\mathbf{T}_\text{node}(\boldsymbol\pi_i)
    - \mathbf{T}_\text{node}(\boldsymbol\pi_j) \|_F^2 \notag \\
     & + \alpha \|\mathbf{T}_\text{edge}(\boldsymbol\pi_i)
    - \mathbf{T}_\text{edge}(\boldsymbol\pi_j) \|_F^2
    , 
\end{align}
for $ i, j \in [N]$. Note that $\|\cdot\|_F$ represents the Frobenius norm.

\begin{algorithm}[ht]
\caption{LinearFGW~\citep{nguyen2023linear}}
\label{alg:linearfgw}
\begin{algorithmic}[1]
   \STATE {\bfseries Input:} $N$ graphs $\mathcal{D}=\{\mathcal{G}_i\}_{i=1}^{N}$, trade-off parameter $\alpha$.
   \STATE Initialize pairwise distance matrix $\mathbf{D} \in \mathbb{R}^{N \times N}$
   \STATE Solve the FGW barycenter problem in Equation~(\ref{eq:fgw-barycenter}) and obtain the reference graph $\overline{\mathcal{G}}$;
   \FOR{graph $\mathcal{G}_i$ {\bfseries in} $\mathcal{D}$}
   \STATE Compute FGW($\mathcal{G}_i, \overline{\mathcal{G}}$) via solving Equation~(\ref{eq:prelim-linearfgw}) and obtain $\boldsymbol\pi_i$;
   \STATE Compute $\mathbf{T}_{\text{node}}(\boldsymbol\pi_i)$ and $\mathbf{T}_{\text{edge}}(\boldsymbol\pi_i)$ via Equation~(\ref{eq:baryproj-node})(\ref{eq:baryproj-edge});
   \ENDFOR 
   \FOR{$\mathcal{G}_i$ {\bfseries in} $\mathcal{D}$}
   \FOR{$\mathcal{G}_j$ {\bfseries in} $\mathcal{D}$}
   \STATE Compute $\mathbf{D}[i, j]$ = $d_{\text{LinearFGW}}(\mathcal{G}_i, \mathcal{G}_j)$ via Equation~(\ref{eq:method-linearfgw});
   \ENDFOR
   \ENDFOR
   \STATE \RETURN{ LinearFGW pairwise distance matrix $\mathbf{D}$}.
\end{algorithmic}
\end{algorithm}

\section{Summarization of GDD (Algorithm~\ref{alg:gdd-computation})}
\label{appendix:gdd-computation}

\begin{algorithm}[ht]
\caption{(Training-Validation) GDD Computation }
\label{alg:gdd-computation}
\begin{algorithmic}[1]
   \STATE {\bfseries Input:} labeled training graphs $\mathcal{D}^{\text{train}}=\{\mathcal{G}_i^{\text{train}}, y^\text{train}_i\}_{i=1}^{n}$, labeled validation graphs $\mathcal{D}^{\text{val}}=\{\mathcal{G}_i^{\text{val}}, y^\text{val}_i\}_{i=1}^{m}$, trade-off parameter $\alpha$, label signal strength $c \geq 0$, a shared label set $\mathcal{Y}$.
   \STATE Compute pairwise LinearFGW distance matrix $\mathbf{D} \in \mathbb{R}^{n \times m}$ via Algorithm~\ref{alg:linearfgw} with the graph set  
   $\mathcal{D}= \mathcal{D}^{\text{train}} \cup \mathcal{D}^{\text{val}}$ and 
 parameter $\alpha$;
   \STATE Initialize new pairwise distance matrix $\tilde{\mathbf{D}} = \mathbf{D}$;
   \STATE Initialize uniform empirical measures: \\
   $\mathbf{p}^\text{train} = \frac{1}{n} \sum_{i \in [n]} \delta_{(\mathcal{G}_i^{\text{train}}, y^\text{train}_i)}, 
   \mathbf{q}^\text{val} = \frac{1}{m} \sum_{j \in [m]} \delta_{(\mathcal{G}_j^{\text{val}}, y^\text{val}_j)}$;
   \FOR{training label $\ell_t$ {\bfseries in} $\mathcal{Y}$}
   \FOR{validation label $\ell_v$ {\bfseries in} $\mathcal{Y}$}
   \STATE Collect training index set with label $\ell_t$: \\
   $\mathbf{I}_{\ell_t} = \{i|y^\text{train}_i = \ell_t\}$;
   \STATE Collect validation data set with label $\ell_v$: \\
   $\mathbf{I}_{\ell_v} = \{j|y^\text{val}_j = \ell_v\}$;
   \STATE Compute \textit{graph-label distance} in Equation~(\ref{eq:method-label-distance}): \\
   $d(\ell_t,\ell_v)= \text{OT}(\mathbf{p}^{\text{train}}_{\ell_t}, \mathbf{q}^{\text{val}}_{\ell_v}, d_{\text{LinearFGW} })$;
   \STATE Update distance sub-matrix $\tilde{\mathbf{D}}[i \in \mathbf{I}_{\ell_t}, j \in \mathbf{I}_{\ell_v}] = \mathbf{D}[i \in \mathbf{I}_{\ell_t}, j \in \mathbf{I}_{\ell_v}] + c \cdot d(\ell_t,\ell_v)$;
   \ENDFOR 
   \ENDFOR 
   \STATE Compute $
   \text{OTDD}(\mathcal{D}^{\text{train}}, \mathcal{D}^{\text{val}})$= 
   $\text{OT}(\mathbf{p}^{\text{train}}, \mathbf{q}^{\text{val}}, \tilde{\mathbf{D}})$ via the equation in Section~\ref{prelim:otdd}.
   \STATE \RETURN{ GDD$(\mathcal{D}^{\text{train}}, \mathcal{D}^{\text{val}})$ = 
   $\text{OTDD}(\mathcal{D}^{\text{train}}, \mathcal{D}^{\text{val}})$.}
\end{algorithmic}
\end{algorithm}

\newpage
\section{Summarization of \optmethodname\ (Algorithm~\ref{alg:alt-got-shrinkage})}
\label{appendix:great}

Starting from a uniform training weight $\mathbf{w}$, \optmethodname\ alternates between two subroutines: (i) computes GDD between the two sets using pairwise distances \( \tilde{\mathbf{D}} \in \mathbb{R}^{{n\times m}} \) as the cost matrix (Line 4) and obtains the gradient \( \mathbf{g}_\mathbf{w} = \nabla_{\mathbf{w}} \text{GDD}(\mathbf{p}^\text{train} (\mathbf{w}), \mathbf{q}^\text{val}, \tilde{\mathbf{D}}) \) for updating \( \mathbf{w} \) (Line 5)
and (ii) gradually sparsifies \( \mathbf{w} \) by retaining only the top-$k$ entries followed by normalization to ensure $\mathbf{w}$ is on the probability simplex (Line 6-9). After $T$ iterations, we extract the non-zero entries from the resulting $\mathbf{w}$ and name this training index set as $\mathbf{S}$. 

\begin{algorithm}[ht]
\caption{\optmethodname}
\label{alg:alt-got-shrinkage}
\begin{algorithmic}[1]
   \STATE {\bfseries Input:} pairwise LinearFGW distance matrix $\tilde{\mathbf{D}} \in \mathbb{R}^{n \times m}$, selection ratio $\tau$, update step $T$, learning rate $\eta$. 
   \STATE Initialize uniform training weights: $\mathbf{w} = \frac{\mathbf{1}_n}{n}$;
   \FOR{$t=1$ {\bfseries to} $T-1$}
   \STATE Compute GDD($\mathbf{p}^\text{train}(\mathbf{w}), \mathbf{q}^\text{val}, \tilde{\mathbf{D}}$) via Algorithm~\ref{alg:gdd-computation};
   \STATE Compute $\mathbf{g}_{\mathbf{w}} = \nabla_{\mathbf{w}}\text{ GDD}(\mathbf{p}^\text{train}(\mathbf{w}), \mathbf{q}^\text{val}, \tilde{\mathbf{D}})$ via Theorem~\ref{thm:grad-gdd};
   \STATE Compute current sparsity level: \\
    $k=n\cdot \max(\tau, \frac{T-t+1}{T-1}+\frac{\tau t}{T-1})$;
   \STATE Update data weight: $\mathbf{w} = \max(\mathbf{w} - \eta \cdot \mathbf{g}_{\mathbf{w}}, \mathbf{0})$;
   \STATE Sparsify data weight: $\mathbf{w} = \mathbf{w}  \odot \text{Top-}k(\mathbf{w})$;
   \STATE Apply $\ell_1$-normalization: $\mathbf{w} = \mathbf{w} / \|\mathbf{w}\|_1$;
   \ENDFOR 
   
   \STATE \RETURN{ training data index set $\mathbf{S} = \text{nonzero}(\mathbf{w})$.}
\end{algorithmic}
\end{algorithm}

\section{Summarization of \methodname\ (Algorithm~\ref{alg:overall-alg})}
\label{appendix:gradate}

\begin{algorithm}[ht]
\caption{\methodname}
\label{alg:overall-alg}
\begin{algorithmic}[1]
   \STATE {\bfseries Input:} labeled training graphs $\mathcal{D}^{\text{train}}=\{\mathcal{G}_i^{\text{train}}, y^\text{train}_i\}_{i=1}^{n}$, labeled validation graphs $\mathcal{D}^{\text{val}}=\{\mathcal{G}_i^{\text{val}}, y^\text{val}_i\}_{i=1}^{m}$, trade-off parameter $\alpha$, label signal strength $c \geq 0$, selection ratio $\tau$, update step $T$, learning rate $\eta$.

   \STATE Compute pairwise LinearFGW distance matrix $\mathbf{D} \in \mathbb{R}^{n \times m}$ via Algorithm~\ref{alg:linearfgw} with the graph set  
   $\mathcal{D}^{\text{train}}, \mathcal{D}^{\text{val}}$ and parameter $\alpha$;
   \STATE Compute the (label-informed) pairwise distance matrix $\tilde{\mathbf{D}}$ with label signal $c$ (line 3-12) in Algorithm~\ref{alg:gdd-computation};
   \STATE Compute $\mathbf{S}$ = \optmethodname ($\tilde{\mathbf{D}}, \tau, T, \eta$);
   \STATE \RETURN{ selected training data index set $\mathbf{S}$.}
\end{algorithmic}
\end{algorithm}

\section{Discussions on FGW \& GDD and Previous Measures}
\label{appendix:tmd-mmd}

\paragraph{Comparison between FGW~\citep{vayer2020fused} and TMD~\citep{chuang2022tree}.}

Specifically, FGW has the following advantages over Tree Mover Distance (TMD)~\citep{chuang2022tree}. Firstly, Linear optimal transport theory~\citep{wang2013linear, nguyen2023linear} can be utilized to bring down the costs for pairwise FGW distance computation while TMD does not have similar technique. Secondly, a single-pair FGW computation (with time complexity $\mathcal{O}(|\mathcal{V}|^3)$) is cheaper than a single-pair TMD computation (with time complexity $\mathcal{O}(\mathcal{L}|\mathcal{V}|^4)$), where $\mathcal{V}$ is graph size and $\mathcal{L}$ is the depth of TMD. While cheaper, FGW can achieve similar theoretical results as TMD.

\paragraph{Comparison between GDD and MMD~\citep{gretton2012kernel}.}

GDD offers a more flexible and expressive notion of graph dataset similarity than Max Mean Discrepency (MMD)~\citep{gretton2012kernel}, which solely compares aggregated graph embeddings. To be more specific, GDD has the following three advantages over MMD. Firstly, unlike MMD, which often depends on model-specific representations (such as pre-trained encoder) or require training, GDD does not involve training and is model-free. This makes it broadly applicable across various graph-level datasets without the need for task-specific models. Secondly, GDD can optionally incorporate \textit{auxiliary label information} (when available), enabling more fine-grained and task-relevant comparisons between data distributions in classification settings. Finally, GDD is based on Wasserstein distance. Although not explicitly stated in our paper, this results in interpretable correspondences between data points across datasets that can directly be used for data selection or data re-weighting algorithms for domain adaptation applications.

\section{Proofs of Theorems}
\label{appendix:all-proofs}
\subsection{Proof of Theorem~\ref{thm:3.1-fgw}}
\label{appendix:thm3.1-fgw-proof}

In this section, we prove Theorem~\ref{thm:3.1-fgw}. We first focus on a simplified case with $k=1$, which implies that the underlying GNN has only one layer. Then, based on this result, we use induction to generalize the conclusion to any positive $k$, which represents multi-layer GNNs.

\subsubsection{Assumptions}\label{appendix:ass:fgw}
With a slight abuse of notation, let a graph \( \mathcal{G} \) denote its node set as well. We assume that \( f \) only uses one-hop information followed by a linear transformation. Specifically, for any graph \( \mathcal{G} = (\mathbf{A}, \mathbf{X}) \) and any node \( u \in \mathcal{G} \), the output \( f(\mathcal{G})_u \) depends only on the local neighborhood of \( u \), defined as \( \mathcal{N}_{\mathcal{G}}(u) := \{ \mathbf{A}[u,v], \mathbf{X}[v] \}_{v \in \mathcal{G}} \). This localized aggregation is first computed by a convolution function \( g \), and the result is then passed through a linear transformation with weights \( \mathbf{W} \) and bias \( \mathbf{b} \), giving:
\begin{align}
f(\mathcal{G})_u = \mathbf{W} \cdot g(\mathcal{N}_{\mathcal{G}}(u)) + \mathbf{b} = \mathbf{W} \cdot g\left(\left\{ \mathbf{A}[u,v], \mathbf{X}[v] \right\}_{v \in \mathcal{G}} \right) + \mathbf{b}.
\end{align}
We assume the convolution function $g$ is $C_{W}$-Lipschitz w.r.t.\ the following FGW distance $d_{W;\alpha}$: for any nodes $u_1\in\mathcal G_1$ and $u_2\in\mathcal G_2$,
\begin{align}
&d_{W;\alpha}(\mathcal N_{\mathcal G_1}(u_1),\mathcal N_{\mathcal G_2}(u_2)) \\
&:=\left(\inf_{\pi\in\Pi(\mu_1,\mu_2)}\underset{(v_1,v_2)\sim\pi}{\mathbb E}[(1-\alpha)\|\mathbf X_1[v_1]-\mathbf X_2[v_2]\|^r+\alpha|\mathbf A_1[u_1,v_1]-\mathbf A_2[u_2,v_2]|^r] \right)^{1/r},
\end{align}
where we use $\mu_1:=\mathsf{Unif}(\mathcal G_1)$ and $\mu_2:=\mathsf{Unif}(\mathcal G_2)$ in this work. 

\subsubsection{\texorpdfstring{Proof for $k=1$}{Proof for k=1}}

\begin{proof}
Let $\mu_1:=\mathsf{Unif}(\mathcal G_1)$, $\mu_2:=\mathsf{Unif}(\mathcal G_2)$. 

For any coupling $\pi\in\Pi(\mu_1,\mu_2)$, by Jensen's inequality w.r.t.\ the concave function $x\mapsto x^{1/r}$, 
\resizebox{\linewidth}{!}{
\begin{minipage}{\linewidth}
\begin{align}
&\underset{(u_1,u_2)\sim\pi}{\mathbb E}[\|f(\mathcal G_1)_{u_1}-f(\mathcal G_2)_{u_2}\|]\\
={}&\underset{(u_1,u_2)\sim\pi}{\mathbb E}\Big[\Big\| [\mathbf{W}\cdot g(\{(\mathbf A_1[u_1,v_1],\mathbf X_1[v_1])\}_{v_1\in\mathcal G_1}) +\mathbf{b}]- [\mathbf{W}\cdot g(\{(\mathbf A_2[u_2,v_2],\mathbf X_1[v_2])\}_{v_2\in\mathcal G_2})+\mathbf{b}] \Big\|\Big]\\
\le{}&\underset{(u_1,u_2)\sim\pi}{\mathbb E}\Big[C_{W} \|\mathbf{W}\|  \cdot d_{W;\alpha}(\{(\mathbf A_1[u_1,v_1],\mathbf X_1[v_1])\}_{v_1\in\mathcal G_1},\{(\mathbf A_2[u_2,v_2],\mathbf X_2[v_2])\}_{v_2\in\mathcal G_2})\Big]\\
={}&C_{W}\|\mathbf{W}\|\cdot\underset{(u_1,u_2)\sim\pi}{\mathbb E}\Big[\Big(\inf_{\pi' \in\Pi(\mu_1,\mu_2)}\underset{(v_1,v_2)\sim \pi'}{\mathbb E}[(1-\alpha)\|\mathbf X_1[v_1]-\mathbf X_2[v_2]\|^r+\\
&\quad \quad\quad\quad \alpha|\mathbf A_1[u_1,v_1]-\mathbf A_2[u_2,v_2]|^r]\Big)^{1/r}\Big]\notag\\
\le{}&C\cdot\underset{(u_1,u_2)\sim\pi}{\mathbb E}\Big[\Big(\underset{(v_1,v_2)\sim\pi}{\mathbb E}[(1-\alpha)\|\mathbf X_1[v_1]-\mathbf X_2[v_2]\|^r+\alpha|\mathbf A_1[u_1,v_1]-\mathbf A_2[u_2,v_2]|^r)]\Big)^{1/r}\Big]\\
\le{}&C\cdot\underset{(u_1,u_2)\sim\pi}{\mathbb E}\Big[\underset{(v_1,v_2)\sim\pi}{\mathbb E}[(1-\alpha)\|\mathbf X_1[v_1]-\mathbf X_2[v_2]\|^r+\alpha|\mathbf A_1[u_1,v_1]-\mathbf A_2[u_2,v_2]|^r)]\Big]^{1/r}
,
\end{align}
\end{minipage}
}
where $C_1 = C_W\|\mathbf{W}\|$.
We explain the inequalities as follows. The first inequality is from our smoothness assumption stated in the previous subsection. The second is by removing the infimum. The third is another use of Jensen's inequality.

Since the above inequality holds for any valid coupling $\pi$, we can take infimum on both side. Thus, it follows that $d_{W}(f(\mathcal G_1),f(\mathcal G_2))$ is at most
\begin{align}
&\inf_{\pi\in\Pi(\mu_1,\mu_2)}\underset{(u_1,u_2)\sim\pi}{\mathbb E}[\|f(\mathcal G_1)_{u_1}-f(\mathcal G_2)_{u_2}\|]
\\
\le{}&\inf_{\pi\in\Pi(\mu_1,\mu_2)}C_1\cdot\underset{(u_1,u_2)\sim\pi}{\mathbb E}\Big[\underset{(v_1,v_2)\sim\pi}{\mathbb E}[(1-\alpha)\|\mathbf X_1[v_1]-\mathbf X_2[v_2]\|^r+ \\
&\quad\quad\quad\quad \alpha|\mathbf A_1[u_1,v_1]-\mathbf A_2[u_2,v_2]|^r)]\Big]^{1/r}\notag
\\
={}&C_1\cdot\inf_{\pi\in\Pi(\mu_1,\mu_2)}\underset{(u_1,u_2)\sim\pi}{\mathbb E}\Big[\underset{(v_1,v_2)\sim\pi}{\mathbb E}[(1-\alpha)\|\mathbf X_1[v_1]-\mathbf X_2[v_2]\|^r+\\
&\quad\quad\quad\quad\alpha|\mathbf A_1[u_1,v_1]-\mathbf A_2[u_2,v_2]|^r)]\Big]^{1/r}
\\
={}&C_1\cdot \Big(\inf_{\pi\in\Pi(\mu_1,\mu_2)}\underset{(u_1,u_2)\sim\pi}{\mathbb E}\Big[\underset{(v_1,v_2)\sim\pi}{\mathbb E}[(1-\alpha)\|\mathbf X_1[v_1]-\mathbf X_2[v_2]\|^r+\\
&\quad\quad\quad\quad\alpha|\mathbf A_1[u_1,v_1]-\mathbf A_2[u_2,v_2]|^r)]\Big] \Big)^{1/r}
\\
={}&C_1\cdot\textnormal{FGW}_\alpha(\mathcal G_1,\mathcal G_2),
\end{align}
which completes the proof the case of $k=1$.
Note that the inequality is from our smoothness assumption stated in the previous section and the last equality is due to the definition of FGW distance with trade-off parameter $\beta=\alpha$.
\end{proof}

\subsubsection{\texorpdfstring{Proof for general $k>1$}{Proof for general k>1}}
\begin{proof}
For general $k>1$, we can iteratively apply similar logic as in the case of $k=1$ to bound the output distance with multi-layer GNNs. Specifically, we can write a $k$-layer GNN $f$ as a composite function that concatenates multiple convolution layer (i.e. $f_1, \cdots,f_k$)\footnote{We assume all these convolution functions $\{f_m\}_{1\le m\le k}$ satisfy the assumption we made in Section~\ref{appendix:ass:fgw} with constant $C_{W}$.} with ReLU activation functions (i.e. $\sigma_1, \cdots, \sigma_{k-1}$):
$
f = f_k \circ \sigma_{k-1} \circ f_{k-1} \circ \dotsm \circ \sigma_{1}\circ f_1,
\text{where}\;
\sigma_j=\mathrm{ReLU} (\cdot).$
For any $m \le k$, define $h_{m}:=f_{m}\circ \sigma_{m-1}\circ h_{m-1} = f_{m}\circ h'_{m-1}$, where $h'_{m-1} = \sigma_{m-1}\circ h_{m-1}$. Note that we have $f = h_k$. Then, for any coupling $\pi\!\in\!\Pi(\mu_1,\mu_2)$, we have:
\resizebox{\linewidth}{!}{
\begin{minipage}{\linewidth}
\begin{align}
&\underset{(u_1,u_2)\sim\pi}{\mathbb E}\bigl[\|f(\mathcal G_1)_{u_1}-f(\mathcal G_2)_{u_2}\|\bigr]  \\
=\;&
\underset{(u_1,u_2)\sim\pi}{\mathbb E}\bigl[\|  h_{k}(\mathcal G_1)_{u_1} - h_{k}(\mathcal G_2)_{u_2} \|\bigr]  \\
=\;&
\underset{(u_1,u_2)\sim\pi}{\mathbb E}
\bigl[\|(f_{k}\circ h'_{k-1})(\mathcal G_1)_{u_1}-(f_k\circ h'_{k-1})(\mathcal G_2)_{u_2}\|\bigr] \\
\le\;&
\underset{(u_1,u_2)\sim\pi}{\mathbb E}
\bigl[\|C_1\cdot  d_{W;\alpha} \bigl( h'_{k-1}(\mathcal G_1)_{u_1}, h'_{k-1} (\mathcal G_2)_{u_2})\bigr) \|\bigr] \\
=\;&
C_1\underset{(u_1,u_2)\sim\pi}{\mathbb E} \Big[ \Big(\inf_{\pi' \in \Pi(\mu_1,\mu_2)}\underset{(v_1,v_2)\sim\pi'}{\mathbb E} (1-\alpha)
\bigl[\|h'_{k-1}(\mathcal G_1)_{v_1}-h'_{k-1}(\mathcal G_2)_{v_2}\|^r\bigr] +
\\
&\quad\quad\quad\quad\alpha\bigl[|\mathbf A_1 [u_1, v_1] - \mathbf A_2 [u_2, v_2]|^r\bigr] \Big)^{1/r}\Big]\notag \\
\le\;&
C_1\underset{(u_1,u_2)\sim\pi}{\mathbb E} \Big[\underset{(v_1,v_2)\sim\pi}{\mathbb E} \Big(
(1-\alpha)
\bigl[\|h'_{k-1}(\mathcal G_1)_{v_1}-h'_{k-1}(\mathcal G_2)_{v_2}\|^r\bigr] + \\
&\quad\quad\quad\quad\alpha 
\bigl[|\mathbf A_1 [u_1, v_1] - \mathbf A_2 [u_2, v_2]|^r\bigr] \Big)^{1/r}\Big] \\
\le\;&
C_1\Bigg( 
(1-\alpha) \underset{(v_1,v_2)\sim\pi}{\mathbb E}\bigl[\|h'_{k-1}(\mathcal G_1)_{v_1}-h'_{k-1}(\mathcal G_2)_{v_2}\|^r\bigr] + 
\\
&\quad\quad\quad\quad\alpha
\underset{\substack{(u_1,u_2)\sim\pi \\ (v_1,v_2)\sim\pi}}{\mathbb E} \bigl[|\mathbf A_1 [u_1, v_1] - \mathbf A_2 [u_2, v_2]|^r\bigr]
\Bigg)^{1/r}\\
\le\;&
C_1\Bigg( 
(1-\alpha) \underset{(v_1,v_2)\sim\pi}{\mathbb E}\bigl[\|h_{k-1}(\mathcal G_1)_{v_1}-h_{k-1}(\mathcal G_2)_{v_2}\|^r\bigr] + 
\\
&\quad\quad\quad\alpha
\underset{\substack{(u_1,u_2)\sim\pi \\ (v_1,v_2)\sim\pi}}{\mathbb E} \bigl[|\mathbf A_1 [u_1, v_1] - \mathbf A_2 [u_2, v_2]|^r\bigr]
\Bigg)^{1/r}.
\end{align}
\end{minipage}
}

Note that the first inequality is from the smoothness assumption, the second is by removing infimum, the third is by Jensen's inequality and the fourth is because $\mathrm{ReLU}(\cdot)$ is a contraction function.

Here, we can iteratively apply the regularity assumption specified in Section~\ref{appendix:ass:fgw} to expand the term above: $\|h_{m-1}(\mathcal G_1)_{v_1}-h_{m-1}(\mathcal G_2)_{v_2}\|^r, \forall m \in \{k,\cdots,1\}$ to have the following deduction.

\begin{align}
&\underset{(u_1,u_2)\sim\pi}{\mathbb E}\bigl[\|f(\mathcal G_1)_{u_1}-f(\mathcal G_2)_{u_2}\|\bigr] \\
\le\;& C_1 \Bigg( 
C_1^{k-1} (1-\alpha)^k  \underset{(v_1,v_2)\sim\pi}{\mathbb E}\bigl[\|\mathbf X_1[v_1]-\mathbf X_2 [v_2]\|^r\bigr] + \\
&\alpha \sum_{m=0}^{k-1} [C_1(1-\alpha)]^m
\underset{\substack{(u_1,u_2)\sim\pi \\ (v_1,v_2)\sim\pi}}{\mathbb E} \bigl[|\mathbf A_1 [u_1, v_1] - \mathbf A_2 [u_2, v_2]|^r\bigr]
\Bigg)^{1/r} \notag \\
=\;&
C \cdot \Bigg( (1-\beta) \underset{(v_1,v_2)\sim\pi}{\mathbb E}\bigl[\|\mathbf X_1[v_1]-\mathbf X_2 [v_2]\|^r\bigr]  + \beta \underset{\substack{(u_1,u_2)\sim\pi \\ (v_1,v_2)\sim\pi}}{\mathbb E} \bigl[|\mathbf A_1 [u_1, v_1] - \mathbf A_2 [u_2, v_2]|^r\bigr] \Bigg)^{1/r},
\end{align}
where $C = C_1\frac{\alpha+(1-\alpha)^k(C_1)^{k-1}(1-C_1)}{1-C_1(1-\alpha)}$ and $\beta = \frac{\alpha(1-C_1^k(1-\alpha)^k)}{\alpha+(1-\alpha)^k(C_1)^{k-1}(1-C_1)}$.

Since the above equation holds for any coupling $\pi$, we can take infimum from both sides to get:
\begin{align}
&d_W (f(\mathcal{G}_1, f(\mathcal{G}_2)) \\
=\;& \inf_{\pi\in\Pi(\mu_1,\mu_2)}\underset{(u_1,u_2)\sim\pi}{\mathbb E}[\|f(\mathcal G_1)_{u_1}-f(\mathcal G_2)_{u_2}\|] \\
\le\;&  
C \cdot \inf_{\pi\in\Pi(\mu_1,\mu_2)}\Bigg( (1-\beta) \underset{(v_1,v_2)\sim\pi}{\mathbb E}\bigl[\|\mathbf X_1[v_1]-\mathbf X_2 [v_2]\|^r\bigr]  + \\
&\quad\quad\quad\quad\beta \underset{\substack{(u_1,u_2)\sim\pi \\ (v_1,v_2)\sim\pi}}{\mathbb E} \bigl[|\mathbf A_1 [u_1, v_1] - \mathbf A_2 [u_2, v_2]|^r\bigr] \Bigg)^{1/r} \\
=\;& C \cdot \textnormal{FGW}_\beta (\mathcal{G}_1, \mathcal{G}_2),
\end{align}
which completes the proof.
\end{proof}


\begin{remark}
    To justify the smoothness assumption on $g$, we note that it is an abstraction of GNN aggregation functions. For example, aggregation operations such as \textit{mean, max and sum} all satisfy our assumption.
\end{remark}

\begin{remark}
    Note that our technical assumption and the results of Theorem~\ref{thm:3.1-fgw} are independent. Firstly, the assumption on the convolution function $g$ is about the smoothness property \textit{between node representations within a single graph}; while the results of Theorem~\ref{thm:3.1-fgw} is bounding the FGW distance between \textit{sets of node representations between two graphs}.
\end{remark}

\subsection{Proof of Theorem~\ref{thm:3.3-gap}}
\label{appendix:proof:1}

For any coupling $\pi\in\Pi(\mathbf p^\text{train}(\mathbf w),\mathbf q^\text{val})$, by Jensen's inequality and the Lipschitzness assumption,
\begin{align}
&\Big|\underset{(\mathcal G,y)\sim\mathbf p^\textnormal{train}(\mathbf w)}{\mathbb E}[\mathcal{L} (f(\mathcal G), y)]-\underset{(\mathcal G,y)\sim\mathbf q^\textnormal{val}}{\mathbb E}[\mathcal{L} (f(\mathcal G), y)]\Big|
\\={}&\Big|\underset{(\mathcal G^\textnormal{train},y^\textnormal{train})\sim\mathbf p^\textnormal{train}(\mathbf w)}{\mathbb E}[\mathcal{L} (f(\mathcal G^\textnormal{train}), y^\textnormal{train})]-\underset{(\mathcal G^\textnormal{val},y^\textnormal{val})\sim\mathbf q^\textnormal{val}}{\mathbb E}[\mathcal{L} (f(\mathcal G^\textnormal{val}), y^\textnormal{val})]\Big|
\\={}&\Big|\underset{((\mathcal G^\textnormal{train},y^\textnormal{train}),(\mathcal G^\textnormal{val},y^\textnormal{val}))\sim\pi}{\mathbb E}[\mathcal{L} (f(\mathcal G^\textnormal{train}), y^\textnormal{train})]-\underset{((\mathcal G^\textnormal{train},y^\textnormal{train}),(\mathcal G^\textnormal{val},y^\textnormal{val}))\sim\pi}{\mathbb E}[\mathcal{L} (f(\mathcal G^\textnormal{val}), y^\textnormal{val})]\Big|
\\={}&\Big|\underset{((\mathcal G^\textnormal{train},y^\textnormal{train}),(\mathcal G^\textnormal{val},y^\textnormal{val}))\sim\pi}{\mathbb E}[\mathcal{L} (f(\mathcal G^\textnormal{train}), y^\textnormal{train})-\mathcal{L} (f(\mathcal G^\textnormal{val}), y^\textnormal{val})]\Big|
\\\le{}&\underset{((\mathcal G^\textnormal{train},y^\textnormal{train}),(\mathcal G^\textnormal{val},y^\textnormal{val}))\sim\pi}{\mathbb E}[|\mathcal{L} (f(\mathcal G^\textnormal{train}), y^\textnormal{train})-\mathcal{L} (f(\mathcal G^\textnormal{val}), y^\textnormal{val})|]
\\\le{}&\underset{((\mathcal G^\textnormal{train},y^\textnormal{train}),(\mathcal G^\textnormal{val},y^\textnormal{val}))\sim\pi}{\mathbb E}[C\cdot d^c_{g\mathcal Z}((\mathcal G^\textnormal{train}, y^\textnormal{train}),(\mathcal G^\textnormal{val}, y^\textnormal{val}))]
\\={}&C\cdot\underset{((\mathcal G^\textnormal{train},y^\textnormal{train}),(\mathcal G^\textnormal{val},y^\textnormal{val}))\sim\pi}{\mathbb E}[d^c_{g\mathcal Z}((\mathcal G^\textnormal{train}, y^\textnormal{train}),(\mathcal G^\textnormal{val}, y^\textnormal{val}))].
\end{align}
Since this holds for any coupling $\pi\in\Pi(\mathbf p^\text{train}(\mathbf w),\mathbf q^\text{val})$, then we have
\begin{align}
&\Big|\underset{(\mathcal G,y)\sim\mathbf p^\textnormal{train}(\mathbf w)}{\mathbb E}[\mathcal{L} (f(\mathcal G), y)]-\underset{(\mathcal G,y)\sim\mathbf q^\textnormal{val}}{\mathbb E}[\mathcal{L} (f(\mathcal G), y)]\Big|
\\\le{}&C\cdot\inf_{\pi\in\Pi(\mathbf p^\text{train}(\mathbf w),\mathbf q^\text{val})}\underset{((\mathcal G^\textnormal{train},y^\textnormal{train}),(\mathcal G^\textnormal{val},y^\textnormal{val}))\sim\pi}{\mathbb E}[d^c_{g\mathcal Z}((\mathcal G^\textnormal{train}, y^\textnormal{train}),(\mathcal G^\textnormal{val}, y^\textnormal{val}))]
\\
={}&C \cdot \textnormal{OT}(\mathbf p^\textnormal{train}(\mathbf w),\mathbf q^\textnormal{val},d^c_{g\mathcal Z})\\
={}&C\cdot\textnormal{GDD} 
(\mathcal{D^{\textnormal{train}}_\mathbf{w}, \mathcal{D}^\textnormal{val}}).
\end{align}
It follows that
\begin{align}
\underset{(\mathcal G,y)\sim\mathbf q^\textnormal{val}}{\mathbb E}[\mathcal{L} (f(\mathcal G), y)]\le\underset{(\mathcal G,y)\sim\mathbf p^\textnormal{train}(\mathbf w)}{\mathbb E}[\mathcal{L} (f(\mathcal G), y)]+C\cdot\textnormal{GDD} 
(\mathcal{D^{\textnormal{train}}_\mathbf{w}, \mathcal{D}^\textnormal{val}}),
\end{align}
which completes the proof.
Note that the first inequality follows from Jensen's inequality (w.r.t. the absolute function). The second and third inequalities are both due to the smoothness assumption stated in Theorem~\ref{thm:3.3-gap}.

\subsection{Proof of Theorem~\ref{thm:grad-gdd}}
\label{appendix:thm:3.5-proof}

\begin{theorem}[Gradient of GDD w.r.t. Training Weights; \citealp{just2023lava}]
Given a distance matrix $\mathbf{D}$, a validation empirical measure $\mathbf{q}^{\text{val}}$ and a training empirical measure $\mathbf{p}^{\text{train}}(\mathbf{w})$ based on the weight $\mathbf{w}$. Let $\boldsymbol{\beta}(\pi^*)$ be the dual variables with respect to $\mathbf{p}^{\text{train}}(\mathbf{w})$ for the \textnormal{GDD} problem defined in Equation~(\ref{eq:method-gdd-minimization}). 
The gradient of $\textnormal{GDD}(\mathbf{p}^\text{train}(\textbf{w}), \mathbf{q}^{\text{val}}, \mathbf{D})$ with respect to $\mathbf{w}$ can be computed as:
\begin{equation*}
    \nabla_\mathbf{w} \textnormal{GDD}(\mathbf{p}^\textnormal{train}(\textbf{w}), \mathbf{q}^{\textnormal{val}}, \mathbf{D}) = \boldsymbol{\beta}^* (\pi^*),
\end{equation*}
where $\boldsymbol{\beta}^*(\pi^*)$ is the optimal solution w.r.t. $\mathbf{p}^\text{train} (\mathbf{w})$ to the dual of the GDD problem.
\label{thm:grad-gdd}
\end{theorem}
\begin{proof}
    Omitted. Please see the Sensitivity Theorem stated by~\citet{bertsekas1997nonlinear}.
\end{proof}

\section{Additional Experiments}
\label{appendix:additional-exp}

\subsection{Comparing data selection methods for \textit{graph size} shift on GCN \& GIN}
\label{appendix:exp-table-1-size-old}

We conduct the same evaluation as Table~\ref{table:graph-select} on \textit{graph size} shift with GCN and GIN as backbone model in Table~\ref{appendix:table-graph-select-size-new}.

\begin{table*}[ht]
\centering
\resizebox{1.0\textwidth}{!}{%
\begin{tabular}{@{\extracolsep{\fill}}cc|cccc|cccc}
\toprule
\midrule
\multirow{2.4}{*}{\textbf{Dataset}}&   \multirow{1}{*}{\textbf{GNN Architecture} $\rightarrow$\!\!\!\! }& \multicolumn{4}{c|}{\makecell{\textbf{GCN}}}  & \multicolumn{4}{c}{\makecell{\textbf{GIN}}} \\
\cmidrule{2-10} 
&\textbf{Selection Method} $\downarrow$\!& $\tau=10\%$ & $\tau=20\%$ & $\tau=50\%$ & Full & $\tau=10\%$ & $\tau=20\%$ & $\tau=50\%$ & Full \\
\midrule
\multirow{4}{*}{\textsc{IMDB-BINARY}}
& Random &       
0.573\scriptsize{${\pm}$ 0.041} & 0.612\scriptsize{${\pm}$ 0.008} & 0.645\scriptsize{${\pm}$ 0.051} & 

\multirow{4}{*}{0.630\scriptsize{${\pm}$0.008}}
&    
0.620\scriptsize{${\pm}$ 0.007} & 0.582\scriptsize{${\pm}$ 0.009} & 0.605\scriptsize{${\pm}$ 0.019} & 

\multirow{4}{*}{0.602\scriptsize{${\pm}$0.010}}\\
&\textsc{KiDD}-\textsc{LR} &   
0.592\scriptsize{${\pm}$ 0.015} & 0.540\scriptsize{${\pm}$ 0.014} & 0.652\scriptsize{${\pm}$ 0.008} & 
& 
0.553\scriptsize{${\pm}$ 0.013} & 0.555\scriptsize{${\pm}$ 0.012} & 0.577\scriptsize{${\pm}$ 0.012} & 
  \\

& LAVA   &       
\underline{0.824}\scriptsize{${\pm}$ 0.008} & 
\underline{0.823}\scriptsize{${\pm}$ 0.019} & 
\textbf{0.837}\scriptsize{${\pm}$ 0.006} & 
 
&    
\underline{0.822}\scriptsize{${\pm}$ 0.005}  & 
\textbf{0.830}\scriptsize{${\pm}$ 0.011} & 
\textbf{0.848}\scriptsize{${\pm}$ 0.002}& 

   \\

\cmidrule{2-5} \cmidrule{7-9}
  

&\cellcolor[HTML]{D3D3D3}\methodname    &   
\cellcolor[HTML]{D3D3D3}\textbf{0.826}\scriptsize{${\pm}$ 0.009} & 
\cellcolor[HTML]{D3D3D3}\textbf{0.825}\scriptsize{${\pm}$ 0.018} & 
\cellcolor[HTML]{D3D3D3}\underline{0.830}\scriptsize{${\pm}$ 0.007} & 

& 
\cellcolor[HTML]{D3D3D3}\textbf{0.823}\scriptsize{${\pm}$ 0.002} & 
\cellcolor[HTML]{D3D3D3}\underline{0.820}\scriptsize{${\pm}$ 0.008} & 
\cellcolor[HTML]{D3D3D3}\underline{0.832}\scriptsize{${\pm}$ 0.008} & 

\\
\midrule
\multirow{4}{*}{\textsc{IMDB-MULTI}}
& Random &       
\underline{0.374}\scriptsize{${\pm}$ 0.031} & 
0.354\scriptsize{${\pm}$ 0.008} & 
0.366\scriptsize{${\pm}$ 0.008} & 
\multirow{4}{*}{0.386\scriptsize{${\pm}$0.006}}
&   
0.351\scriptsize{${\pm}$ 0.008} & 
0.372\scriptsize{${\pm}$ 0.039} & 
0.369\scriptsize{${\pm}$ 0.019} & 

\multirow{4}{*}{0.368\scriptsize{${\pm}$0.010}}\\
&\textsc{KiDD}-\textsc{LR} &   
0.329\scriptsize{${\pm}$ 0.010} & 0.416\scriptsize{${\pm}$ 0.064} & 0.432\scriptsize{${\pm}$ 0.010} & 
& 
\underline{0.346}\scriptsize{${\pm}$ 0.048} & 0.371\scriptsize{${\pm}$ 0.010} & 0.412\scriptsize{${\pm}$ 0.018} & 
   \\

&LAVA   &   
0.314\scriptsize{${\pm}$ 0.006} & 
\underline{0.426}\scriptsize{${\pm}$ 0.003} & 
\underline{0.600}\scriptsize{${\pm}$ 0.005} & 
  
& 
0.341\scriptsize{${\pm}$ 0.049} &
\underline{0.388}\scriptsize{${\pm}$ 0.018} &
\underline{0.563}\scriptsize{${\pm}$ 0.007} &
   \\
\cmidrule{2-5} \cmidrule{7-9}
& \cellcolor[HTML]{D3D3D3}\methodname   & 
\cellcolor[HTML]{D3D3D3}\textbf{0.353}\scriptsize{${\pm}$ 0.000} &  
\cellcolor[HTML]{D3D3D3}\textbf{0.524}\scriptsize{${\pm}$ 0.016} &  
\cellcolor[HTML]{D3D3D3}\textbf{0.602}\scriptsize{${\pm}$ 0.004}  & 
 
 &
 \cellcolor[HTML]{D3D3D3}\textbf{0.349}\scriptsize{${\pm}$ 0.046} &  
 \cellcolor[HTML]{D3D3D3}\textbf{0.497}\scriptsize{${\pm}$ 0.015} &  
 \cellcolor[HTML]{D3D3D3}\textbf{0.604}\scriptsize{${\pm}$ 0.006} & 
   \\
\midrule
\multirow{4}{*}{\textsc{MSRC}\_21}
& Random &       
0.450\scriptsize{${\pm}$ 0.008} & 0.497\scriptsize{${\pm}$ 0.011} & 0.781\scriptsize{${\pm}$ 0.019} & 
\multirow{4}{*}{0.816\scriptsize{${\pm}$0.026}}
&   
0.149\scriptsize{${\pm}$ 0.007} & 0.418\scriptsize{${\pm}$ 0.008} & 0.690\scriptsize{${\pm}$ 0.015} & 

\multirow{4}{*}{0.749\scriptsize{${\pm}$0.023}}\\
&\textsc{KiDD}-\textsc{LR} &   
\textbf{0.725}\scriptsize{${\pm}$ 0.017} & 0.819\scriptsize{${\pm}$ 0.015} & 0.857\scriptsize{${\pm}$ 0.008} & 
&
\textbf{0.649}\scriptsize{${\pm}$ 0.012} & 0.743\scriptsize{${\pm}$ 0.008} & 0.781\scriptsize{${\pm}$ 0.050} & 
\\

&LAVA   &   
0.617\scriptsize{${\pm}$ 0.015} & 
\underline{0.825}\scriptsize{${\pm}$ 0.014} & 
\underline{0.918}\scriptsize{${\pm}$ 0.018} & 

& 
0.617\scriptsize{${\pm}$ 0.008} & 
\underline{0.810}\scriptsize{${\pm}$ 0.004} & 
\underline{0.889}\scriptsize{${\pm}$ 0.011} & 
   \\
\cmidrule{2-5} \cmidrule{7-9}
& \cellcolor[HTML]{D3D3D3}\methodname   
& 
\cellcolor[HTML]{D3D3D3}\underline{0.670}\scriptsize{${\pm}$ 0.017} & 
\cellcolor[HTML]{D3D3D3}\textbf{0.836}\scriptsize{${\pm}$ 0.017} & 
\cellcolor[HTML]{D3D3D3}\textbf{0.953}\scriptsize{${\pm}$ 0.011} & 
&
\cellcolor[HTML]{D3D3D3}\underline{0.629}\scriptsize{${\pm}$ 0.011} &
\cellcolor[HTML]{D3D3D3}\textbf{0.813}\scriptsize{${\pm}$ 0.008} &
\cellcolor[HTML]{D3D3D3}\textbf{0.901}\scriptsize{${\pm}$ 0.008} & 
   \\
\midrule
\multirow{4}{*}{\texttt{ogbg-molbace}}
& Random &       
0.443\scriptsize{${\pm}$ 0.014} & 0.504\scriptsize{${\pm}$ 0.022} & 0.476\scriptsize{${\pm}$ 0.011} & 
\multirow{4}{*}{0.434\scriptsize{${\pm}$0.033}}
&   
0.479\scriptsize{${\pm}$ 0.070} & 0.471\scriptsize{${\pm}$ 0.092} & 0.578\scriptsize{${\pm}$ 0.030} & 

\multirow{4}{*}{0.548\scriptsize{${\pm}$0.028}}\\
&\textsc{KiDD}-\textsc{LR} &   
0.446\scriptsize{${\pm}$ 0.040} & 0.489\scriptsize{${\pm}$ 0.049} & 0.483\scriptsize{${\pm}$ 0.011} & 
&
0.547\scriptsize{${\pm}$ 0.080} & 0.523\scriptsize{${\pm}$ 0.060} & 0.571\scriptsize{${\pm}$ 0.013} & 
\\

&LAVA   &   
\underline{0.563}\scriptsize{${\pm}$ 0.045} &
\underline{0.574}\scriptsize{${\pm}$ 0.067} &
\underline{0.535}\scriptsize{${\pm}$ 0.044} & 

& 
\underline{0.645}\scriptsize{${\pm}$ 0.035}& 
\textbf{0.641}\scriptsize{${\pm}$ 0.027} & 
\textbf{0.648}\scriptsize{${\pm}$ 0.025} & 
   \\
\cmidrule{2-5} \cmidrule{7-9}
& \cellcolor[HTML]{D3D3D3}\methodname   & 
\cellcolor[HTML]{D3D3D3}\textbf{0.570}\scriptsize{${\pm}$ 0.080}  &
\cellcolor[HTML]{D3D3D3}\textbf{0.599}\scriptsize{${\pm}$ 0.037} & 
\cellcolor[HTML]{D3D3D3}\textbf{0.575}\scriptsize{${\pm}$ 0.056} & 
&
\cellcolor[HTML]{D3D3D3}\textbf{0.646}\scriptsize{${\pm}$ 0.033} & 
\cellcolor[HTML]{D3D3D3}\underline{0.618}\scriptsize{${\pm}$ 0.061} & 
\cellcolor[HTML]{D3D3D3}\underline{0.630}\scriptsize{${\pm}$ 0.020}   & 
   \\
\midrule
\multirow{4}{*}{\texttt{ogbg-molbbbp}}
& Random &       
0.499\scriptsize{${\pm}$ 0.041} & 
0.635\scriptsize{${\pm}$ 0.042} & 
0.648\scriptsize{${\pm}$ 0.031} & 
\multirow{4}{*}{0.618\scriptsize{${\pm}$0.037}}
&   
0.698\scriptsize{${\pm}$ 0.010} & 
0.633\scriptsize{${\pm}$ 0.043} & 
0.691\scriptsize{${\pm}$ 0.040} & 

\multirow{4}{*}{0.779\scriptsize{${\pm}$0.017}}\\
&\textsc{KiDD}-\textsc{LR} 
&   
0.639\scriptsize{${\pm}$ 0.025} & 
0.599\scriptsize{${\pm}$ 0.013} & 
0.611\scriptsize{${\pm}$ 0.023} & 
& 
0.546\scriptsize{${\pm}$ 0.105} & 
0.656\scriptsize{${\pm}$ 0.038} & 
0.609\scriptsize{${\pm}$ 0.081} & 
   \\

&LAVA   &   
\underline{0.667}\scriptsize{${\pm}$ 0.015}& 
\textbf{0.675}\scriptsize{${\pm}$ 0.013} & 
\textbf{0.691}\scriptsize{${\pm}$ 0.017} & 

& 
\underline{0.859}\scriptsize{${\pm}$ 0.019} & 
\underline{0.889}\scriptsize{${\pm}$ 0.016} & 
\underline{0.893}\scriptsize{${\pm}$ 0.011} &
   \\
\cmidrule{2-5} \cmidrule{7-9}
& 
\cellcolor[HTML]{D3D3D3}\methodname   & 
\cellcolor[HTML]{D3D3D3}\textbf{0.677}\scriptsize{${\pm}$ 0.007}  & 
\cellcolor[HTML]{D3D3D3}\underline{0.671}\scriptsize{${\pm}$ 0.015} & 
\cellcolor[HTML]{D3D3D3}\underline{0.673}\scriptsize{${\pm}$ 0.041} & 
&
\cellcolor[HTML]{D3D3D3}\textbf{0.866}\scriptsize{${\pm}$ 0.016}  &
\cellcolor[HTML]{D3D3D3}\textbf{0.890}\scriptsize{${\pm}$ 0.011} & 
\cellcolor[HTML]{D3D3D3}\textbf{0.895}\scriptsize{${\pm}$ 0.012}& 
   \\
\midrule
\multirow{4}{*}{\texttt{ogbg-molhiv}}
& Random &       
0.576\scriptsize{${\pm}$ 0.008} & 0.579\scriptsize{${\pm}$ 0.004} & 0.594\scriptsize{${\pm}$ 0.001} & 
\multirow{4}{*}{0.592\scriptsize{${\pm}$0.000}}
&   
0.613\scriptsize{${\pm}$ 0.004} & 0.617\scriptsize{${\pm}$ 0.045} & 0.624\scriptsize{${\pm}$ 0.015} & 

\multirow{4}{*}{0.664\scriptsize{${\pm}$0.027}}\\
&\textsc{KiDD}-\textsc{LR} &   
0.556\scriptsize{${\pm}$ 0.001} & 0.551\scriptsize{${\pm}$ 0.027} & 0.595\scriptsize{${\pm}$ 0.003} & 
& 
0.586\scriptsize{${\pm}$ 0.055} & 0.586\scriptsize{${\pm}$ 0.014} & 0.629\scriptsize{${\pm}$ 0.019} & 
   \\

&LAVA   &   
\textbf{0.669}\scriptsize{${\pm}$ 0.001}  & 
\textbf{0.683}\scriptsize{${\pm}$ 0.004} & 
\textbf{0.659}\scriptsize{${\pm}$ 0.002} & 

& 
\textbf{0.769}\scriptsize{${\pm}$ 0.014} & 
0.737\scriptsize{${\pm}$ 0.012} & 
0.796\scriptsize{${\pm}$ 0.025} & 
   \\
\cmidrule{2-5} \cmidrule{7-9}
& \cellcolor[HTML]{D3D3D3}\methodname   
& \cellcolor[HTML]{D3D3D3}\underline{0.640}\scriptsize{${\pm}$ 0.002}
& \cellcolor[HTML]{D3D3D3}\underline{0.638}\scriptsize{${\pm}$ 0.006}
& \cellcolor[HTML]{D3D3D3}\underline{0.629}\scriptsize{${\pm}$ 0.000}  & 
&
\cellcolor[HTML]{D3D3D3}\underline{0.731}\scriptsize{${\pm}$ 0.017}
&\cellcolor[HTML]{D3D3D3}\textbf{0.767}\scriptsize{${\pm}$ 0.004}
& \cellcolor[HTML]{D3D3D3}\textbf{0.805}\scriptsize{${\pm}$ 0.024} & 
   \\
\midrule
\bottomrule
\end{tabular}
}
\caption{Performance comparison across data selection methods for \textit{graph size} shift on GCN and GIN. We use \textbf{bold}/\underline{underline} to indicate the 1st/2nd best results.  \methodname\ achieves top-2 performance across all datasets and is the best-performer in most settings. }
\label{appendix:table-graph-select-size-new}
\vspace{-1em}
\end{table*}

\subsection{Comparing data selection methods for \textit{graph density} shift on GAT \& GraphSAGE}
\label{appendix:exp-table-1-density-new}

We conduct the same evaluation as Table~\ref{table:graph-select} on \textit{graph density} shift with GAT and GraphSAGE as backbone model in Table~\ref{appendix:table-graph-select-density-new}.

\begin{table*}[ht]
\centering
\resizebox{1.0\textwidth}{!}{%
\begin{tabular}{@{\extracolsep{\fill}}cc|cccc|cccc}
\toprule
\midrule
\multirow{2.4}{*}{\textbf{Dataset}}&   \multirow{1}{*}{\textbf{GNN Architecture} $\rightarrow$\!\!\!\! }& \multicolumn{4}{c|}{\makecell{\textbf{GAT}}}  & \multicolumn{4}{c}{\makecell{\textbf{GraphSAGE}}} \\
\cmidrule{2-10} 
&\textbf{Selection Method} $\downarrow$\!& $\tau=10\%$ & $\tau=20\%$ & $\tau=50\%$ & Full & $\tau=10\%$ & $\tau=20\%$ & $\tau=50\%$ & Full \\
\midrule
\multirow{4}{*}{\textsc{IMDB-BINARY}}
& Random &       
0.602\scriptsize{${\pm}$ 0.005} & 0.695\scriptsize{${\pm}$ 0.035} & 0.797\scriptsize{${\pm}$ 0.005} &

\multirow{4}{*}{0.807\scriptsize{${\pm}$0.033}}
&    
0.730\scriptsize{${\pm}$ 0.014} & 0.637\scriptsize{${\pm}$ 0.039} & 0.762\scriptsize{${\pm}$ 0.027} &

\multirow{4}{*}{0.823\scriptsize{${\pm}$0.009}}\\
&\textsc{KiDD}-\textsc{LR} &   
0.683\scriptsize{${\pm}$ 0.041} & 0.803\scriptsize{${\pm}$ 0.005} & 0.817\scriptsize{${\pm}$ 0.024} & 
& 
0.662\scriptsize{${\pm}$ 0.054} & 0.785\scriptsize{${\pm}$ 0.025} & 0.775\scriptsize{${\pm}$ 0.054} & 

  \\

& LAVA   &       
\underline{0.818}\scriptsize{${\pm}$ 0.010} & 
\underline{0.857}\scriptsize{${\pm}$ 0.009} & 
\underline{0.885}\scriptsize{${\pm}$ 0.018} & 
 
&    
\underline{0.827}\scriptsize{${\pm}$ 0.005} & 
\underline{0.840}\scriptsize{${\pm}$ 0.021} & 
\underline{0.883}\scriptsize{${\pm}$ 0.012}& 

   \\

\cmidrule{2-5} \cmidrule{7-9}
  

&\cellcolor[HTML]{D3D3D3}\methodname    &   
\cellcolor[HTML]{D3D3D3}\textbf{0.850}\scriptsize{${\pm}$ 0.023} & 
\cellcolor[HTML]{D3D3D3}\textbf{0.865}\scriptsize{${\pm}$ 0.008} & 
\cellcolor[HTML]{D3D3D3}\textbf{0.892}\scriptsize{${\pm}$ 0.012} & 

& 
\cellcolor[HTML]{D3D3D3}\textbf{0.835}\scriptsize{${\pm}$ 0.015} & 
\cellcolor[HTML]{D3D3D3}\textbf{0.852}\scriptsize{${\pm}$ 0.035} & 
\cellcolor[HTML]{D3D3D3}\textbf{0.907}\scriptsize{${\pm}$ 0.005} & 

\\
\midrule
\multirow{4}{*}{\textsc{IMDB-MULTI}}
& Random &       
0.087\scriptsize{${\pm}$ 0.014} & 0.071\scriptsize{${\pm}$ 0.006} & 0.076\scriptsize{${\pm}$ 0.003} & 

\multirow{4}{*}{0.080\scriptsize{${\pm}$0.000}}
&   
0.090\scriptsize{${\pm}$ 0.005} & 0.203\scriptsize{${\pm}$ 0.061} & 0.126\scriptsize{${\pm}$ 0.064} &

\multirow{4}{*}{0.097\scriptsize{${\pm}$0.024}}\\
&\textsc{KiDD}-\textsc{LR} &   
0.176\scriptsize{${\pm}$ 0.024} & 0.121\scriptsize{${\pm}$ 0.044} & 0.158\scriptsize{${\pm}$ 0.036} & 
& 
0.154\scriptsize{${\pm}$ 0.028} & 0.124\scriptsize{${\pm}$ 0.068} & 0.054\scriptsize{${\pm}$ 0.011} & 

   \\

&LAVA   &   
\underline{0.597}\scriptsize{${\pm}$ 0.273} &
\textbf{0.599}\scriptsize{${\pm}$ 0.294} & 
\underline{0.341}\scriptsize{${\pm}$ 0.049} & 
  
& 
\textbf{0.341}\scriptsize{${\pm}$ 0.049} &
\textbf{0.307}\scriptsize{${\pm}$ 0.164} &
\underline{0.328}\scriptsize{${\pm}$ 0.317} &
   \\
\cmidrule{2-5} \cmidrule{7-9}
& \cellcolor[HTML]{D3D3D3}\methodname   & 
\cellcolor[HTML]{D3D3D3}\textbf{0.790}\scriptsize{${\pm}$ 0.000} &  
\cellcolor[HTML]{D3D3D3}\underline{0.589}\scriptsize{${\pm}$ 0.287} &  
\cellcolor[HTML]{D3D3D3}\textbf{0.776}\scriptsize{${\pm}$ 0.039} &  
 
 &
 \cellcolor[HTML]{D3D3D3}\underline{0.306}\scriptsize{${\pm}$ 0.216}&  
 \cellcolor[HTML]{D3D3D3}\underline{0.299}\scriptsize{${\pm}$ 0.282} &  
 \cellcolor[HTML]{D3D3D3}\textbf{0.363}\scriptsize{${\pm}$ 0.238} & 
   \\
\midrule
\multirow{4}{*}{\textsc{MSRC}\_21}
& Random &       
0.462\scriptsize{${\pm}$ 0.029} & 0.763\scriptsize{${\pm}$ 0.007} & 0.857\scriptsize{${\pm}$ 0.018} & 

\multirow{4}{*}{0.860\scriptsize{${\pm}$0.007}}
&   
0.617\scriptsize{${\pm}$ 0.017} & 0.725\scriptsize{${\pm}$ 0.033} & 0.842\scriptsize{${\pm}$ 0.029} &

\multirow{4}{*}{0.874\scriptsize{${\pm}$0.004}}\\
&\textsc{KiDD}-\textsc{LR} &   
0.661\scriptsize{${\pm}$ 0.030} & 0.778\scriptsize{${\pm}$ 0.015} & 0.860\scriptsize{${\pm}$ 0.025} & 

&
0.681\scriptsize{${\pm}$ 0.073} & 0.787\scriptsize{${\pm}$ 0.025} & 0.857\scriptsize{${\pm}$ 0.004} & 

\\

&LAVA   &   
\underline{0.699}\scriptsize{${\pm}$ 0.047} & 
\underline{0.816}\scriptsize{${\pm}$ 0.037} & 
\underline{0.912}\scriptsize{${\pm}$ 0.007} & 

& 
\underline{0.766}\scriptsize{${\pm}$ 0.029} & 
\underline{0.857}\scriptsize{${\pm}$ 0.015}  & 
\underline{0.918}\scriptsize{${\pm}$ 0.011} & 
   \\
\cmidrule{2-5} \cmidrule{7-9}
& \cellcolor[HTML]{D3D3D3}\methodname   
& 
\cellcolor[HTML]{D3D3D3}\textbf{0.716}\scriptsize{${\pm}$ 0.017} & 
\cellcolor[HTML]{D3D3D3}\textbf{0.822}\scriptsize{${\pm}$ 0.004} & 
\cellcolor[HTML]{D3D3D3}\textbf{0.921}\scriptsize{${\pm}$ 0.007} & 
&
\cellcolor[HTML]{D3D3D3}\textbf{0.781}\scriptsize{${\pm}$ 0.026}  &
\cellcolor[HTML]{D3D3D3}\textbf{0.877}\scriptsize{${\pm}$ 0.026} &
\cellcolor[HTML]{D3D3D3}\textbf{0.944}\scriptsize{${\pm}$ 0.011}& 
   \\
\midrule
\multirow{4}{*}{\texttt{ogbg-molbace}}
& Random &       
0.480\scriptsize{${\pm}$ 0.040} & \textbf{0.606}\scriptsize{${\pm}$ 0.085} & 0.637\scriptsize{${\pm}$ 0.075} & 

\multirow{4}{*}{0.583\scriptsize{${\pm}$0.042}}
&   
0.459\scriptsize{${\pm}$ 0.149} & 0.478\scriptsize{${\pm}$ 0.097} & 0.503\scriptsize{${\pm}$ 0.034} & 

\multirow{4}{*}{0.622\scriptsize{${\pm}$0.119}}\\
&\textsc{KiDD}-\textsc{LR} &   
\underline{0.558}\scriptsize{${\pm}$ 0.012} & 0.443\scriptsize{${\pm}$ 0.029} & 0.628\scriptsize{${\pm}$ 0.023} & 

&
0.606\scriptsize{${\pm}$ 0.023} & \underline{0.596}\scriptsize{${\pm}$ 0.079} & \underline{0.607}\scriptsize{${\pm}$ 0.047} & 

\\

&LAVA   &   
\textbf{0.564}\scriptsize{${\pm}$ 0.097} &
0.519\scriptsize{${\pm}$ 0.007} &
\underline{0.696}\scriptsize{${\pm}$ 0.031}  & 

& 
\underline{0.620}\scriptsize{${\pm}$ 0.075}& 
\textbf{0.649}\scriptsize{${\pm}$ 0.004} & 
\textbf{0.651}\scriptsize{${\pm}$ 0.059}  & 
   \\
\cmidrule{2-5} \cmidrule{7-9}
& \cellcolor[HTML]{D3D3D3}\methodname   & 
\cellcolor[HTML]{D3D3D3}0.501\scriptsize{${\pm}$ 0.017}  &
\cellcolor[HTML]{D3D3D3}\underline{0.541}\scriptsize{${\pm}$ 0.048} & 
\cellcolor[HTML]{D3D3D3}\textbf{0.720}\scriptsize{${\pm}$ 0.004} & 
&
\cellcolor[HTML]{D3D3D3}\textbf{0.621}\scriptsize{${\pm}$ 0.067} & 
\cellcolor[HTML]{D3D3D3}0.587\scriptsize{${\pm}$ 0.078} & 
\cellcolor[HTML]{D3D3D3}0.568\scriptsize{${\pm}$ 0.126}   & 
   \\
\midrule
\multirow{4}{*}{\texttt{ogbg-molbbbp}}
& Random &       
0.511\scriptsize{${\pm}$ 0.034} & 0.529\scriptsize{${\pm}$ 0.027} & 0.513\scriptsize{${\pm}$ 0.018} & 
\multirow{4}{*}{0.569\scriptsize{${\pm}$0.030}}
&   
0.463\scriptsize{${\pm}$ 0.012} & 0.385\scriptsize{${\pm}$ 0.032} & 0.468\scriptsize{${\pm}$ 0.008} &

\multirow{4}{*}{0.447\scriptsize{${\pm}$0.008}}\\
&\textsc{KiDD}-\textsc{LR} 
&   
0.444\scriptsize{${\pm}$ 0.050} & 0.405\scriptsize{${\pm}$ 0.021} & 0.434\scriptsize{${\pm}$ 0.025} & 
& 
0.392\scriptsize{${\pm}$ 0.002} & 0.415\scriptsize{${\pm}$ 0.028} & 0.466\scriptsize{${\pm}$ 0.034} & 
   \\

&LAVA   &   
\underline{0.584}\scriptsize{${\pm}$ 0.054}& 
\underline{0.552}\scriptsize{${\pm}$ 0.018} & 
\underline{0.603}\scriptsize{${\pm}$ 0.021} & 

& 
\underline{0.526}\scriptsize{${\pm}$ 0.087} & 
\textbf{0.612}\scriptsize{${\pm}$ 0.005} & 
\underline{0.495}\scriptsize{${\pm}$ 0.029} &
   \\
\cmidrule{2-5} \cmidrule{7-9}
& 
\cellcolor[HTML]{D3D3D3}\methodname   & 
\cellcolor[HTML]{D3D3D3}\textbf{0.617}\scriptsize{${\pm}$ 0.038}  & 
\cellcolor[HTML]{D3D3D3}\textbf{0.578}\scriptsize{${\pm}$ 0.038} & 
\cellcolor[HTML]{D3D3D3}\textbf{0.632}\scriptsize{${\pm}$ 0.036} & 
&
\cellcolor[HTML]{D3D3D3}\textbf{0.580}\scriptsize{${\pm}$ 0.067}  &
\cellcolor[HTML]{D3D3D3}0.558\scriptsize{${\pm}$ 0.064} & 
\cellcolor[HTML]{D3D3D3}\textbf{0.528}\scriptsize{${\pm}$ 0.027}& 
   \\
\midrule
\multirow{4}{*}{\texttt{ogbg-molhiv}}
& Random &       
0.601\scriptsize{${\pm}$ 0.017} & 0.591\scriptsize{${\pm}$ 0.011} & 0.581\scriptsize{${\pm}$ 0.016} & 
\multirow{4}{*}{0.571\scriptsize{${\pm}$0.030}}
&   
0.577\scriptsize{${\pm}$ 0.016} & 0.591\scriptsize{${\pm}$ 0.007} & 0.594\scriptsize{${\pm}$ 0.005} & 

\multirow{4}{*}{0.588\scriptsize{${\pm}$0.003}}\\
&\textsc{KiDD}-\textsc{LR} &   
0.620\scriptsize{${\pm}$ 0.001} & 0.616\scriptsize{${\pm}$ 0.003} & 0.615\scriptsize{${\pm}$ 0.007} &
& 
\textbf{0.607}\scriptsize{${\pm}$ 0.008} & 0.534\scriptsize{${\pm}$ 0.057} & 0.603\scriptsize{${\pm}$ 0.018} & 
   \\

&LAVA   &   
\underline{0.621}\scriptsize{${\pm}$ 0.001}  & 
\textbf{0.631}\scriptsize{${\pm}$ 0.003}  & 
\textbf{0.624}\scriptsize{${\pm}$ 0.014} & 

& 
 0.575\scriptsize{${\pm}$ 0.012} & 
\textbf{0.607}\scriptsize{${\pm}$ 0.008} & 
0.608\scriptsize{${\pm}$ 0.007} & 
   \\
\cmidrule{2-5} \cmidrule{7-9}
& \cellcolor[HTML]{D3D3D3}\methodname   
& \cellcolor[HTML]{D3D3D3}\textbf{0.638}\scriptsize{${\pm}$ 0.001}
& \cellcolor[HTML]{D3D3D3}\underline{0.620}\scriptsize{${\pm}$ 0.002}
& \cellcolor[HTML]{D3D3D3}\underline{0.619}\scriptsize{${\pm}$ 0.004}   & 
&
\cellcolor[HTML]{D3D3D3}\underline{0.599}\scriptsize{${\pm}$ 0.021}
&\cellcolor[HTML]{D3D3D3}\underline{0.598}\scriptsize{${\pm}$ 0.009} 
& \cellcolor[HTML]{D3D3D3}\textbf{0.610}\scriptsize{${\pm}$ 0.006}  & 
   \\
\midrule
\bottomrule
\end{tabular}
}
\caption{Performance comparison across data selection methods for \textit{graph density} shift on GAT and GraphSAGE. We use \textbf{bold}/\underline{underline} to indicate the 1st/2nd best results.  \methodname\ is the best-performer in most settings.}
\label{appendix:table-graph-select-density-new}
\vspace{-1em}
\end{table*}

\subsection{Comparing data selection methods for \textit{graph size} shift on GAT \& GraphSAGE}
\label{appendix:exp-table-1-size-new}

We conduct the same evaluation as Table~\ref{table:graph-select} on \textit{graph size} shift with GAT and GraphSAGE as backbone model in Table~\ref{appendix:table-graph-select-size-new2}.

\begin{table*}[ht]
\centering
\resizebox{1.0\textwidth}{!}{%
\begin{tabular}{@{\extracolsep{\fill}}cc|cccc|cccc}
\toprule
\midrule
\multirow{2.4}{*}{\textbf{Dataset}}&   \multirow{1}{*}{\textbf{GNN Architecture} $\rightarrow$\!\!\!\! }& \multicolumn{4}{c|}{\makecell{\textbf{GAT}}}  & \multicolumn{4}{c}{\makecell{\textbf{GraphSAGE}}} \\
\cmidrule{2-10} 
&\textbf{Selection Method} $\downarrow$\!& $\tau=10\%$ & $\tau=20\%$ & $\tau=50\%$ & Full & $\tau=10\%$ & $\tau=20\%$ & $\tau=50\%$ & Full \\
\midrule
\multirow{4}{*}{\textsc{IMDB-BINARY}}
& Random &       
0.678\scriptsize{${\pm}$ 0.082} & 0.558\scriptsize{${\pm}$ 0.022} & 0.660\scriptsize{${\pm}$ 0.085} &

\multirow{4}{*}{0.595\scriptsize{${\pm}$0.007}}
&    
0.555\scriptsize{${\pm}$ 0.011} & 0.563\scriptsize{${\pm}$ 0.012} & 0.562\scriptsize{${\pm}$ 0.012} & 

\multirow{4}{*}{0.567\scriptsize{${\pm}$0.018}}\\
&\textsc{KiDD}-\textsc{LR} &   
0.683\scriptsize{${\pm}$ 0.071} & 0.587\scriptsize{${\pm}$ 0.081} & 0.665\scriptsize{${\pm}$ 0.098} & 

& 
0.663\scriptsize{${\pm}$ 0.035} & 0.558\scriptsize{${\pm}$ 0.013} & 0.595\scriptsize{${\pm}$ 0.007} & 

  \\

& LAVA   &       
\underline{0.808}\scriptsize{${\pm}$ 0.014} & 
\underline{0.830}\scriptsize{${\pm}$ 0.004} & 
\underline{0.835}\scriptsize{${\pm}$ 0.000} & 
 
&    
\underline{0.807}\scriptsize{${\pm}$ 0.026} & 
\underline{0.808}\scriptsize{${\pm}$ 0.027} & 
\underline{0.830}\scriptsize{${\pm}$ 0.004}& 

   \\

\cmidrule{2-5} \cmidrule{7-9}

&\cellcolor[HTML]{D3D3D3}\methodname    &   
\cellcolor[HTML]{D3D3D3}\textbf{0.835}\scriptsize{${\pm}$ 0.018} & 
\cellcolor[HTML]{D3D3D3}\textbf{0.833}\scriptsize{${\pm}$ 0.005} & 
\cellcolor[HTML]{D3D3D3}\textbf{0.837}\scriptsize{${\pm}$ 0.010} & 

& 
\cellcolor[HTML]{D3D3D3}\textbf{0.808}\scriptsize{${\pm}$ 0.016} & 
\cellcolor[HTML]{D3D3D3}\textbf{0.828}\scriptsize{${\pm}$ 0.024} & 
\cellcolor[HTML]{D3D3D3}\textbf{0.838}\scriptsize{${\pm}$ 0.016} & 

\\
\midrule
\multirow{4}{*}{\textsc{IMDB-MULTI}}
& Random &       
\textbf{0.384}\scriptsize{${\pm}$ 0.014} & 0.408\scriptsize{${\pm}$ 0.004} & 0.384\scriptsize{${\pm}$ 0.034} & 

\multirow{4}{*}{0.374\scriptsize{${\pm}$0.028}}
&   
0.336\scriptsize{${\pm}$ 0.010} & 0.357\scriptsize{${\pm}$ 0.005} & 0.381\scriptsize{${\pm}$ 0.026} &

\multirow{4}{*}{0.391\scriptsize{${\pm}$0.026}}\\
&\textsc{KiDD}-\textsc{LR} &   
\underline{0.366}\scriptsize{${\pm}$ 0.020} & \underline{0.434}\scriptsize{${\pm}$ 0.006} & 0.404\scriptsize{${\pm}$ 0.010} & 

& 
0.339\scriptsize{${\pm}$ 0.030} & \textbf{0.418}\scriptsize{${\pm}$ 0.030} & \underline{0.422}\scriptsize{${\pm}$ 0.011} & 

   \\

&LAVA   &   
0.333\scriptsize{${\pm}$ 0.058} & 
0.417\scriptsize{${\pm}$ 0.015} & 
\underline{0.577}\scriptsize{${\pm}$ 0.014} & 
  
& 
\underline{0.374}\scriptsize{${\pm}$ 0.021} &
\underline{0.389}\scriptsize{${\pm}$ 0.039} &
0.392\scriptsize{${\pm}$ 0.036} &
   \\
\cmidrule{2-5} \cmidrule{7-9}
& \cellcolor[HTML]{D3D3D3}\methodname   & 
\cellcolor[HTML]{D3D3D3}0.342\scriptsize{${\pm}$ 0.050} &
\cellcolor[HTML]{D3D3D3}\textbf{0.537}\scriptsize{${\pm}$ 0.003} &  
\cellcolor[HTML]{D3D3D3}\textbf{0.616}\scriptsize{${\pm}$ 0.002} & 
 
 &
 \cellcolor[HTML]{D3D3D3}\textbf{0.392}\scriptsize{${\pm}$ 0.032}&  
 \cellcolor[HTML]{D3D3D3}0.360\scriptsize{${\pm}$ 0.025}&  
 \cellcolor[HTML]{D3D3D3}\textbf{0.517}\scriptsize{${\pm}$ 0.071}& 
   \\
\midrule
\multirow{4}{*}{\textsc{MSRC}\_21}
& Random &       
0.284\scriptsize{${\pm}$ 0.025} & 0.614\scriptsize{${\pm}$ 0.056} & 0.731\scriptsize{${\pm}$ 0.018} & 

\multirow{4}{*}{0.787\scriptsize{${\pm}$0.004}}
&   
0.497\scriptsize{${\pm}$ 0.023} & 0.412\scriptsize{${\pm}$ 0.050} & 0.725\scriptsize{${\pm}$ 0.022} & 

\multirow{4}{*}{0.810\scriptsize{${\pm}$0.027}}\\
&\textsc{KiDD}-\textsc{LR} &   
\underline{0.626}\scriptsize{${\pm}$ 0.004} & 0.746\scriptsize{${\pm}$ 0.026} & 0.830\scriptsize{${\pm}$ 0.011} & 

&
\underline{0.722}\scriptsize{${\pm}$ 0.030} & 0.798\scriptsize{${\pm}$ 0.029} & 0.789\scriptsize{${\pm}$ 0.026} & 

\\

&LAVA   &   
0.620\scriptsize{${\pm}$ 0.015} & 
\underline{0.798}\scriptsize{${\pm}$ 0.012} & 
\underline{0.909}\scriptsize{${\pm}$ 0.018} & 

& 
0.693\scriptsize{${\pm}$ 0.029} & 
\underline{0.827}\scriptsize{${\pm}$ 0.015} & 
\underline{0.918}\scriptsize{${\pm}$ 0.008}& 
   \\
\cmidrule{2-5} \cmidrule{7-9}
& \cellcolor[HTML]{D3D3D3}\methodname   
& 
\cellcolor[HTML]{D3D3D3}\textbf{0.643}\scriptsize{${\pm}$ 0.039} & 
\cellcolor[HTML]{D3D3D3}\textbf{0.860}\scriptsize{${\pm}$ 0.021} & 
\cellcolor[HTML]{D3D3D3}\textbf{0.947}\scriptsize{${\pm}$ 0.014} & 
&
\cellcolor[HTML]{D3D3D3}\textbf{0.760}\scriptsize{${\pm}$ 0.023}  &
\cellcolor[HTML]{D3D3D3}\textbf{0.842}\scriptsize{${\pm}$ 0.007}  &
\cellcolor[HTML]{D3D3D3}\textbf{0.944}\scriptsize{${\pm}$ 0.004}  & 
   \\
\midrule
\multirow{4}{*}{\texttt{ogbg-molbace}}
& Random &       
0.515\scriptsize{${\pm}$ 0.006} & 0.464\scriptsize{${\pm}$ 0.056} & 0.488\scriptsize{${\pm}$ 0.005} & 

\multirow{4}{*}{0.463\scriptsize{${\pm}$0.004}}
&   
0.523\scriptsize{${\pm}$ 0.048} & 0.463\scriptsize{${\pm}$ 0.020} & 0.583\scriptsize{${\pm}$ 0.027} & 

\multirow{4}{*}{0.487\scriptsize{${\pm}$0.108}}\\
&\textsc{KiDD}-\textsc{LR} &   
0.480\scriptsize{${\pm}$ 0.011} & 0.452\scriptsize{${\pm}$ 0.016} & 0.467\scriptsize{${\pm}$ 0.022} & 

&
0.507\scriptsize{${\pm}$ 0.078} & 0.457\scriptsize{${\pm}$ 0.029} & 0.467\scriptsize{${\pm}$ 0.043} & 

\\

&LAVA   &   
\underline{0.524}\scriptsize{${\pm}$ 0.037}  &
\underline{0.545}\scriptsize{${\pm}$ 0.058}  &
\textbf{0.613}\scriptsize{${\pm}$ 0.080} & 

& 
\textbf{0.650}\scriptsize{${\pm}$ 0.011}& 
\underline{0.481}\scriptsize{${\pm}$ 0.007} & 
\underline{0.550}\scriptsize{${\pm}$ 0.060} & 
   \\
\cmidrule{2-5} \cmidrule{7-9}
& \cellcolor[HTML]{D3D3D3}\methodname   & 
\cellcolor[HTML]{D3D3D3}\textbf{0.529}\scriptsize{${\pm}$ 0.029}  &
\cellcolor[HTML]{D3D3D3}\textbf{0.570}\scriptsize{${\pm}$ 0.062}  & 
\cellcolor[HTML]{D3D3D3}\underline{0.509}\scriptsize{${\pm}$ 0.006}  & 
&
\cellcolor[HTML]{D3D3D3}\underline{0.556}\scriptsize{${\pm}$ 0.040}& 
\cellcolor[HTML]{D3D3D3}\textbf{0.561}\scriptsize{${\pm}$ 0.090} & 
\cellcolor[HTML]{D3D3D3}\textbf{0.564}\scriptsize{${\pm}$ 0.095}   & 
   \\
\midrule
\multirow{4}{*}{\texttt{ogbg-molbbbp}}
& Random &       
0.666\scriptsize{${\pm}$ 0.003} & 0.677\scriptsize{${\pm}$ 0.007} & 0.684\scriptsize{${\pm}$ 0.018} & 

\multirow{4}{*}{0.679\scriptsize{${\pm}$0.004}}
&   
\underline{0.634}\scriptsize{${\pm}$ 0.018} & \underline{0.648}\scriptsize{${\pm}$ 0.028} & 0.641\scriptsize{${\pm}$ 0.025} &

\multirow{4}{*}{0.680\scriptsize{${\pm}$0.010}}\\
&\textsc{KiDD}-\textsc{LR} 
&   
0.594\scriptsize{${\pm}$ 0.017} & 0.596\scriptsize{${\pm}$ 0.020} & 0.650\scriptsize{${\pm}$ 0.002} & 
& 
0.518\scriptsize{${\pm}$ 0.046} & 0.594\scriptsize{${\pm}$ 0.004} & 0.602\scriptsize{${\pm}$ 0.061} & 

   \\

&LAVA   &   
\underline{0.714}\scriptsize{${\pm}$ 0.011} & 
\textbf{0.731}\scriptsize{${\pm}$ 0.019} & 
\underline{0.710}\scriptsize{${\pm}$ 0.036} & 

& 
\textbf{0.639}\scriptsize{${\pm}$ 0.022} & 
0.602\scriptsize{${\pm}$ 0.030} & 
\underline{0.645}\scriptsize{${\pm}$ 0.026} &
   \\
\cmidrule{2-5} \cmidrule{7-9}
& 
\cellcolor[HTML]{D3D3D3}\methodname   & 
\cellcolor[HTML]{D3D3D3}\textbf{0.735}\scriptsize{${\pm}$ 0.021} & 
    \cellcolor[HTML]{D3D3D3}\underline{0.699}\scriptsize{${\pm}$ 0.027} & 
\cellcolor[HTML]{D3D3D3}\textbf{0.713}\scriptsize{${\pm}$ 0.005} & 
&
\cellcolor[HTML]{D3D3D3}0.623\scriptsize{${\pm}$ 0.041} &
\cellcolor[HTML]{D3D3D3}\textbf{0.650}\scriptsize{${\pm}$ 0.019} & 
\cellcolor[HTML]{D3D3D3}\textbf{0.652}\scriptsize{${\pm}$ 0.041} & 
   \\
\midrule
\multirow{4}{*}{\texttt{ogbg-molhiv}}
& Random &       
0.584\scriptsize{${\pm}$ 0.001} & 0.585\scriptsize{${\pm}$ 0.002} & 0.589\scriptsize{${\pm}$ 0.003} & 

\multirow{4}{*}{0.588\scriptsize{${\pm}$0.005}}
&   
0.610\scriptsize{${\pm}$ 0.024} & 0.491\scriptsize{${\pm}$ 0.031} & 0.603\scriptsize{${\pm}$ 0.001} &

\multirow{4}{*}{0.596\scriptsize{${\pm}$0.000}}\\
&\textsc{KiDD}-\textsc{LR} &   
0.586\scriptsize{${\pm}$ 0.001} & 0.584\scriptsize{${\pm}$ 0.001} & 0.584\scriptsize{${\pm}$ 0.004} & 
& 
0.549\scriptsize{${\pm}$ 0.071} & 0.588\scriptsize{${\pm}$ 0.008} & 0.564\scriptsize{${\pm}$ 0.011} & 

   \\

&LAVA   &   
\textbf{0.704}\scriptsize{${\pm}$ 0.005}  & 
\textbf{0.773}\scriptsize{${\pm}$ 0.007} & 
\textbf{0.759}\scriptsize{${\pm}$ 0.002} & 

& 
\textbf{0.721}\scriptsize{${\pm}$ 0.004} & 
\textbf{0.701}\scriptsize{${\pm}$ 0.008} & 
\textbf{0.686}\scriptsize{${\pm}$ 0.011}  & 
   \\
\cmidrule{2-5} \cmidrule{7-9}
& \cellcolor[HTML]{D3D3D3}\methodname   
& \cellcolor[HTML]{D3D3D3}\underline{0.663}\scriptsize{${\pm}$ 0.005}
& \cellcolor[HTML]{D3D3D3}\underline{0.655}\scriptsize{${\pm}$ 0.009}
& \cellcolor[HTML]{D3D3D3}\underline{0.660}\scriptsize{${\pm}$ 0.008}  & 
&
\cellcolor[HTML]{D3D3D3}\underline{0.637}\scriptsize{${\pm}$ 0.002}
&\cellcolor[HTML]{D3D3D3}\underline{0.646}\scriptsize{${\pm}$ 0.019}
& \cellcolor[HTML]{D3D3D3}\underline{0.640}\scriptsize{${\pm}$ 0.007} & 
   \\
\midrule
\bottomrule
\end{tabular}
}
\caption{Performance comparison across data selection methods for \textit{graph size} shift on GAT and GraphSAGE. We use \textbf{bold}/\underline{underline} to indicate the 1st/2nd best results.  \methodname\ achieves top-2 performance across most settings. The under-performance on \texttt{ogbg-molhiv} might due to the reason discussed in Section~\ref{ssec:final-discussion}.}
\label{appendix:table-graph-select-size-new2}
\vspace{-1em}
\end{table*}

\subsection{Comparing GDA and vanilla methods for \textit{graph size} shift}
\label{appendix:exp-table-2-size}

We conduct the same evaluation as Table~\ref{table:gda-vs-graph-select} on \textit{graph size} shift in Table~\ref{appendix:table-graph-vs-gda-size}.
\begin{table}[ht]
\centering
\scalebox{0.64}{
\begin{tabular}{@{\extracolsep{\fill}}ccccccccc}
\toprule
\midrule
&&&\multicolumn{6}{c}{\textbf{Dataset}}\\
\cmidrule{4-9}
\textbf{Type} & \textbf{Model} & \textbf{Data}  & \textsc{IMDB-BINARY}  & \textsc{IMDB-MULTI} & $\textsc{ MSRC\_21 }$ & \texttt{ogbg-molbace} & \texttt{ogbg-molbbbp} & \texttt{ogbg-molhiv} \\
\midrule
\multirow{4}{*}{GDA} &
AdaGCN  & Full &
0.593\scriptsize{${\pm}$ 0.012} &  0.362\scriptsize{${\pm}$ 0.017} & 0.202\scriptsize{${\pm}$ 0.075} &

0.513\scriptsize{${\pm}$ 0.018} &
0.625\scriptsize{${\pm}$ 0.137} &
0.412\scriptsize{${\pm}$ 0.011} \\
& GRADE  & Full &
0.648\scriptsize{${\pm}$ 0.105}&  0.390\scriptsize{${\pm}$ 0.019} & 0.696\scriptsize{${\pm}$ 0.008}&

0.403\scriptsize{${\pm}$ 0.018} &  0.669\scriptsize{${\pm}$ 0.005} & 0.599\scriptsize{${\pm}$ 0.005} \\

& ASN  & Full &
0.633\scriptsize{${\pm}$ 0.054}  &   0.372\scriptsize{${\pm}$ 0.009} & 0.734\scriptsize{${\pm}$ 0.015} & 

0.523\scriptsize{${\pm}$ 0.091} &   0.616\scriptsize{${\pm}$ 0.042} & 0.519\scriptsize{${\pm}$ 0.077} \\
& UDAGCN  & Full &
0.688\scriptsize{${\pm}$ 0.049} & 0.392\scriptsize{${\pm}$ 0.046} & 0.260\scriptsize{${\pm}$ 0.049} &

0.448\scriptsize{${\pm}$ 0.020} &  0.513\scriptsize{${\pm}$ 0.024} & 0.439\scriptsize{${\pm}$ 0.034}\\
\midrule
\multirow{16}{*}{Vanilla} &
\multirow{4}{*}{GCN}
& Random 20\% &
0.612\scriptsize{${\pm}$ 0.008} &   0.354\scriptsize{${\pm}$ 0.008} & 0.497\scriptsize{${\pm}$ 0.011}& 0.504\scriptsize{${\pm}$ 0.022} & 0.635\scriptsize{${\pm}$ 0.042} & 0.579\scriptsize{${\pm}$ 0.004}\\
& &LAVA 20\% & 
0.823\scriptsize{${\pm}$ 0.019}  &   0.426\scriptsize{${\pm}$ 0.003} &  0.825\scriptsize{${\pm}$ 0.014} & 0.574\scriptsize{${\pm}$ 0.067}  & 0.675\scriptsize{${\pm}$ 0.013} & 0.683\scriptsize{${\pm}$ 0.038}  \\
\cmidrule{3-9}
& & \cellcolor[HTML]{D3D3D3}\methodname\  20\%  &  
 \cellcolor[HTML]{D3D3D3}0.825\scriptsize{${\pm}$ 0.018} &   \cellcolor[HTML]{D3D3D3}\underline{0.524}\scriptsize{${\pm}$ 0.016} &  \cellcolor[HTML]{D3D3D3}0.836\scriptsize{${\pm}$ 0.017}  & \cellcolor[HTML]{D3D3D3}0.599\scriptsize{${\pm}$ 0.037} & \cellcolor[HTML]{D3D3D3}0.671\scriptsize{${\pm}$ 0.015}  & \cellcolor[HTML]{D3D3D3}0.638\scriptsize{${\pm}$ 0.006}  \\
\cmidrule{2-3} \cmidrule{4-9}
& \multirow{4}{*}{GIN}
& Random 20\% &      
0.582\scriptsize{${\pm}$ 0.009} & 0.372\scriptsize{${\pm}$ 0.039} & 0.418\scriptsize{${\pm}$ 0.008} &0.471\scriptsize{${\pm}$ 0.092} & 0.633\scriptsize{${\pm}$ 0.043} & 0.617\scriptsize{${\pm}$ 0.045}\\
& & LAVA 20\%&      
0.830\scriptsize{${\pm}$ 0.011} &  0.388\scriptsize{${\pm}$ 0.018} & 0.810\scriptsize{${\pm}$ 0.004} & \textbf{0.641}\scriptsize{${\pm}$ 0.027} & 
\underline{0.889}\scriptsize{${\pm}$ 0.016}  & 0.737\scriptsize{${\pm}$ 0.012} \\
\cmidrule{3-9}
& &  \cellcolor[HTML]{D3D3D3}\methodname\ 20\%&  
 \cellcolor[HTML]{D3D3D3}0.820\scriptsize{${\pm}$ 0.008} &   
 \cellcolor[HTML]{D3D3D3}0.497\scriptsize{${\pm}$ 0.015} & \cellcolor[HTML]{D3D3D3}0.813\scriptsize{${\pm}$ 0.008} & \cellcolor[HTML]{D3D3D3}\underline{0.618}\scriptsize{${\pm}$ 0.061} & \cellcolor[HTML]{D3D3D3}\textbf{0.890}\scriptsize{${\pm}$ 0.011}  & \cellcolor[HTML]{D3D3D3}\underline{0.767}\scriptsize{${\pm}$ 0.004} \\
 \cmidrule{2-3} \cmidrule{4-9}
& \multirow{4}{*}{GAT}
& Random 20\% &      
0.558\scriptsize{${\pm}$ 0.022} & 0.408\scriptsize{${\pm}$ 0.004} & 0.614\scriptsize{${\pm}$ 0.056} & 0.464\scriptsize{${\pm}$ 0.056} & 0.677\scriptsize{${\pm}$ 0.007} & 0.585\scriptsize{${\pm}$ 0.002} \\
& & LAVA 20\%&      
\underline{0.830}\scriptsize{${\pm}$ 0.004} & 0.417\scriptsize{${\pm}$ 0.015} & 0.798\scriptsize{${\pm}$ 0.012} & 0.545\scriptsize{${\pm}$ 0.058} & 0.731\scriptsize{${\pm}$ 0.019} & \textbf{0.773}\scriptsize{${\pm}$ 0.007} \\
\cmidrule{3-9}
& &  \cellcolor[HTML]{D3D3D3}\methodname\ 20\%&  
 \cellcolor[HTML]{D3D3D3}\textbf{0.833}\scriptsize{${\pm}$ 0.005} &   \cellcolor[HTML]{D3D3D3}\textbf{0.537}\scriptsize{${\pm}$ 0.003} & \cellcolor[HTML]{D3D3D3}\textbf{0.860}\scriptsize{${\pm}$ 0.021} & \cellcolor[HTML]{D3D3D3}0.570\scriptsize{${\pm}$ 0.062} & \cellcolor[HTML]{D3D3D3}0.699\scriptsize{${\pm}$ 0.027} & \cellcolor[HTML]{D3D3D3}0.655\scriptsize{${\pm}$ 0.009} \\
 \cmidrule{2-3} \cmidrule{4-9}
& \multirow{4}{*}{GraphSAGE}
& Random 20\% &     
0.563\scriptsize{${\pm}$ 0.012} & 0.357\scriptsize{${\pm}$ 0.005} &
0.412\scriptsize{${\pm}$ 0.050} & 0.463\scriptsize{${\pm}$ 0.020} & 0.648\scriptsize{${\pm}$ 0.028} & 0.491\scriptsize{${\pm}$ 0.031}\\
& & LAVA 20\%&      
0.808\scriptsize{${\pm}$ 0.027} & 0.389\scriptsize{${\pm}$ 0.039} & 0.827\scriptsize{${\pm}$ 0.015} & 0.481\scriptsize{${\pm}$ 0.007} & 0.602\scriptsize{${\pm}$ 0.030} & 0.701\scriptsize{${\pm}$ 0.008} \\
\cmidrule{3-9}
& &  \cellcolor[HTML]{D3D3D3}\methodname\ 20\%&  
 \cellcolor[HTML]{D3D3D3}0.828\scriptsize{${\pm}$ 0.024} &   \cellcolor[HTML]{D3D3D3}0.360\scriptsize{${\pm}$ 0.025} & \cellcolor[HTML]{D3D3D3}\underline{0.842}\scriptsize{${\pm}$ 0.007} & \cellcolor[HTML]{D3D3D3}0.561\scriptsize{${\pm}$ 0.090} & \cellcolor[HTML]{D3D3D3}0.650\scriptsize{${\pm}$ 0.019} & \cellcolor[HTML]{D3D3D3}0.646\scriptsize{${\pm}$ 0.019} \\
\midrule
\bottomrule
\end{tabular}
}
\vspace{0.2em}
\caption{Performance comparison across GDA and vanilla methods for \textit{graph size} shift. We use \textbf{bold}/\underline{underline} to indicate the 1st/2nd best results. \methodname\ can consistently achieve top-2 performance across all datasets and is the best performer in most settings.
}
\label{appendix:table-graph-vs-gda-size}
\end{table}

\subsection{Enhancing GDA methods for \textit{graph size} shift}

\label{appendix:exp-table-3-size}
We conduct the same evaluation as Table~\ref{table:graph-enhancer} on \textit{graph size} shift in Table~\ref{appendix:table-graph-enhancer-size}.

\begin{table*}[ht]
\centering
\resizebox{\textwidth}{!}{%
\begin{tabular}{@{\extracolsep{\fill}}cc|cccc|cccc}
\toprule
\midrule
\multirow{2.4}{*}{\textbf{Dataset}}&   \multirow{1}{*}{\textbf{GDA Method} $\rightarrow$\!\!\!\!\!\!} & \multicolumn{4}{c|}{\makecell{\textbf{AdaGCN}}}  & \multicolumn{4}{c}{\makecell{\textbf{GRADE}}} \\
\cmidrule{2-10} 
&\textbf{Selection Method} $\downarrow$\!\!& $\tau=10\%$ & $\tau=20\%$ & $\tau=50\%$ & Full & $\tau=10\%$ & $\tau=20\%$ & $\tau=50\%$ & Full \\
\midrule
\multirow{3}{*}{\textsc{IMDB-BINARY}}
& Random &       
0.582\scriptsize{${\pm}$ 0.091} & 0.520\scriptsize{${\pm}$ 0.103} & 0.455\scriptsize{${\pm}$ 0.120} & 
\multirow{3}{*}{0.593\scriptsize{${\pm}$ 0.012}}
&    
0.572\scriptsize{${\pm}$ 0.111} & 0.522\scriptsize{${\pm}$ 0.045} & 0.613\scriptsize{${\pm}$ 0.095} &

\multirow{3}{*}{0.648\scriptsize{${\pm}$ 0.105}}\\

& LAVA  &       
\underline{0.818}\scriptsize{${\pm}$ 0.012} &
\underline{0.815}\scriptsize{${\pm}$ 0.005} &
\underline{0.810}\scriptsize{${\pm}$ 0.013} &
 
&    
\underline{0.813}\scriptsize{${\pm}$ 0.007} &
\underline{0.814}\scriptsize{${\pm}$ 0.007} &
\underline{0.816}\scriptsize{${\pm}$ 0.005} &

\\

\cmidrule{2-5} \cmidrule{7-9}
&\cellcolor[HTML]{D3D3D3}\methodname   
&\cellcolor[HTML]{D3D3D3}\textbf{0.834}\scriptsize{${\pm}$ 0.014}
&\cellcolor[HTML]{D3D3D3}\textbf{0.830}\scriptsize{${\pm}$ 0.010}
&\cellcolor[HTML]{D3D3D3}\textbf{0.822}\scriptsize{${\pm}$ 0.022} 
& 

&\cellcolor[HTML]{D3D3D3}\textbf{0.814}\scriptsize{${\pm}$ 0.013}
&\cellcolor[HTML]{D3D3D3}\textbf{0.826}\scriptsize{${\pm}$ 0.007}
&\cellcolor[HTML]{D3D3D3}\textbf{0.827}\scriptsize{${\pm}$ 0.013} 
& 

\\
\midrule
\multirow{3}{*}{\textsc{IMDB-MULTI}}
& Random &       
0.261\scriptsize{${\pm}$ 0.064} & 0.247\scriptsize{${\pm}$ 0.052} & 0.252\scriptsize{${\pm}$ 0.059} &

\multirow{3}{*}{0.362\scriptsize{${\pm}$ 0.017}}
&    
0.312\scriptsize{${\pm}$ 0.034} & 0.280\scriptsize{${\pm}$ 0.000} & 0.282\scriptsize{${\pm}$ 0.030} & 

\multirow{3}{*}{0.390\scriptsize{${\pm}$ 0.019}}\\

& LAVA  &       
\underline{0.374}\scriptsize{${\pm}$ 0.055} & 
\underline{0.385}\scriptsize{${\pm}$ 0.080} & 
\underline{0.368}\scriptsize{${\pm}$ 0.098} & 

&    
\underline{0.386}\scriptsize{${\pm}$ 0.047} & 
\underline{0.407}\scriptsize{${\pm}$ 0.076} & 
\underline{0.411}\scriptsize{${\pm}$ 0.050} &

   \\

\cmidrule{2-5} \cmidrule{7-9}
&\cellcolor[HTML]{D3D3D3}\methodname   
&\cellcolor[HTML]{D3D3D3}\textbf{0.386}\scriptsize{${\pm}$ 0.053}
&\cellcolor[HTML]{D3D3D3}\textbf{0.442}\scriptsize{${\pm}$ 0.085} 
&\cellcolor[HTML]{D3D3D3}\textbf{0.509}\scriptsize{${\pm}$ 0.114}  & 

&\cellcolor[HTML]{D3D3D3}\textbf{0.401}\scriptsize{${\pm}$ 0.076} 
&\cellcolor[HTML]{D3D3D3}\textbf{0.411}\scriptsize{${\pm}$ 0.076} 
&\cellcolor[HTML]{D3D3D3}\textbf{0.503}\scriptsize{${\pm}$ 0.112} & 
\\
\midrule
\multirow{3}{*}{\textsc{MSRC}\_21}
& Random &       
0.084\scriptsize{${\pm}$ 0.028} & 0.079\scriptsize{${\pm}$ 0.022} & 0.114\scriptsize{${\pm}$ 0.030} &

\multirow{3}{*}{0.202\scriptsize{${\pm}$ 0.075}}
&    
0.137\scriptsize{${\pm}$ 0.012} & 0.379\scriptsize{${\pm}$ 0.053} & 0.667\scriptsize{${\pm}$ 0.055} &

\multirow{3}{*}{0.696\scriptsize{${\pm}$ 0.008}}\\

& LAVA  &       
\underline{0.377}\scriptsize{${\pm}$ 0.029} & 
\textbf{0.472}\scriptsize{${\pm}$ 0.030} & 
\underline{0.540}\scriptsize{${\pm}$ 0.103} & 

&    
\underline{0.532}\scriptsize{${\pm}$ 0.029} & 
\underline{0.728}\scriptsize{${\pm}$ 0.038} & 
\underline{0.854}\scriptsize{${\pm}$ 0.035} &

   \\

\cmidrule{2-5} \cmidrule{7-9}
&\cellcolor[HTML]{D3D3D3}\methodname   
&\cellcolor[HTML]{D3D3D3}\textbf{0.411}\scriptsize{${\pm}$ 0.075} 
&\cellcolor[HTML]{D3D3D3}\underline{0.465}\scriptsize{${\pm}$ 0.042} 
&\cellcolor[HTML]{D3D3D3}\textbf{0.593}\scriptsize{${\pm}$ 0.071} & 

&\cellcolor[HTML]{D3D3D3}\textbf{0.553}\scriptsize{${\pm}$ 0.043} 
&\cellcolor[HTML]{D3D3D3}\textbf{0.744}\scriptsize{${\pm}$ 0.044} 
&\cellcolor[HTML]{D3D3D3}\textbf{0.867}\scriptsize{${\pm}$ 0.013} & 

\\
\midrule
\multirow{3}{*}{\texttt{ogbg-molbace}}
& Random &       
0.498\scriptsize{${\pm}$ 0.091} & 0.477\scriptsize{${\pm}$ 0.038} & 0.498\scriptsize{${\pm}$ 0.063} &

\multirow{3}{*}{0.513\scriptsize{${\pm}$ 0.018}}
&    
0.478\scriptsize{${\pm}$ 0.005} & 0.468\scriptsize{${\pm}$ 0.074} & 0.475\scriptsize{${\pm}$ 0.046} &

\multirow{3}{*}{0.403\scriptsize{${\pm}$ 0.018}}\\

& LAVA  &       
\underline{0.510}\scriptsize{${\pm}$ 0.027} & 
\underline{0.523}\scriptsize{${\pm}$ 0.066}&
\underline{0.529}\scriptsize{${\pm}$ 0.056}& 

&    
\underline{0.509}\scriptsize{${\pm}$ 0.055}& 
\textbf{0.559}\scriptsize{${\pm}$ 0.026} &
\textbf{0.552}\scriptsize{${\pm}$ 0.033} &

   \\

\cmidrule{2-5} \cmidrule{7-9}
&\cellcolor[HTML]{D3D3D3}\methodname   
&\cellcolor[HTML]{D3D3D3}\textbf{0.524}\scriptsize{${\pm}$ 0.057}
&\cellcolor[HTML]{D3D3D3}\textbf{0.560}\scriptsize{${\pm}$ 0.053}
&\cellcolor[HTML]{D3D3D3}\textbf{0.550}\scriptsize{${\pm}$ 0.030}  & 

&\cellcolor[HTML]{D3D3D3}\textbf{0.556}\scriptsize{${\pm}$ 0.069} 
&\cellcolor[HTML]{D3D3D3}\underline{0.508}\scriptsize{${\pm}$ 0.025}
&\cellcolor[HTML]{D3D3D3}\underline{0.512}\scriptsize{${\pm}$ 0.022} 
& 

\\

\midrule
\multirow{3}{*}{\texttt{ogbg-molbbbp}}
& Random &       
0.539\scriptsize{${\pm}$ 0.020} & 0.600\scriptsize{${\pm}$ 0.044} & 0.538\scriptsize{${\pm}$ 0.083} &

\multirow{3}{*}{0.625\scriptsize{${\pm}$ 0.137}}
&    
0.632\scriptsize{${\pm}$ 0.006} & \textbf{0.648}\scriptsize{${\pm}$ 0.002} & \textbf{0.650}\scriptsize{${\pm}$ 0.002} &

\multirow{3}{*}{0.669\scriptsize{${\pm}$ 0.005} }\\

& LAVA  &       
\underline{0.583}\scriptsize{${\pm}$ 0.075} & 
\underline{0.653}\scriptsize{${\pm}$ 0.007} & 
\underline{0.657}\scriptsize{${\pm}$ 0.004} & 

&    
\underline{0.631}\scriptsize{${\pm}$ 0.011} & 
0.640\scriptsize{${\pm}$ 0.001} & 
0.637\scriptsize{${\pm}$ 0.003} &

   \\

\cmidrule{2-5} \cmidrule{7-9}
&\cellcolor[HTML]{D3D3D3}\methodname   
&\cellcolor[HTML]{D3D3D3}\textbf{0.662}\scriptsize{${\pm}$ 0.020}
&\cellcolor[HTML]{D3D3D3}\textbf{0.654}\scriptsize{${\pm}$ 0.004}
&\cellcolor[HTML]{D3D3D3}\textbf{0.664}\scriptsize{${\pm}$ 0.010} & 

&\cellcolor[HTML]{D3D3D3}\textbf{0.636}\scriptsize{${\pm}$ 0.010}
&\cellcolor[HTML]{D3D3D3}\underline{0.641}\scriptsize{${\pm}$ 0.005}
&\cellcolor[HTML]{D3D3D3}\underline{0.639}\scriptsize{${\pm}$ 0.013}& 

\\

\midrule
\multirow{3}{*}{\texttt{ogbg-molhiv}}
& Random &       
0.356\scriptsize{${\pm}$ 0.023} & 0.358\scriptsize{${\pm}$ 0.013} & 0.364\scriptsize{${\pm}$ 0.011} &

\multirow{3}{*}{0.412\scriptsize{${\pm}$ 0.011}}
&    
0.588\scriptsize{${\pm}$ 0.014} & 0.615\scriptsize{${\pm}$ 0.012} & 0.592\scriptsize{${\pm}$ 0.010} &

\multirow{3}{*}{0.599\scriptsize{${\pm}$ 0.005}}\\

& LAVA  &       
\underline{0.382}\scriptsize{${\pm}$ 0.035} & 
\textbf{0.403}\scriptsize{${\pm}$ 0.033} & 
\underline{0.384}\scriptsize{${\pm}$ 0.041} & 

&    
\textbf{0.673}\scriptsize{${\pm}$ 0.004} & 
\textbf{0.681}\scriptsize{${\pm}$ 0.002} & 
\textbf{0.668}\scriptsize{${\pm}$ 0.009} &

   \\

\cmidrule{2-5} \cmidrule{7-9}
&\cellcolor[HTML]{D3D3D3}\methodname   
&\cellcolor[HTML]{D3D3D3}\textbf{0.393}\scriptsize{${\pm}$ 0.040} 
&\cellcolor[HTML]{D3D3D3}\underline{0.387}\scriptsize{${\pm}$ 0.068}
&\cellcolor[HTML]{D3D3D3}\textbf{0.395}\scriptsize{${\pm}$ 0.040} & 

&\cellcolor[HTML]{D3D3D3}\underline{0.658}\scriptsize{${\pm}$ 0.004}
&\cellcolor[HTML]{D3D3D3}\underline{0.647}\scriptsize{${\pm}$ 0.005}
&\cellcolor[HTML]{D3D3D3}\underline{0.642}\scriptsize{${\pm}$ 0.005} & 

\\

\midrule
\midrule
\multirow{2.4}{*}{\textbf{Dataset}}&   \multirow{1}{*}{\textbf{GDA Method} $\rightarrow$\!\!\!\!\!\!} & \multicolumn{4}{c|}{\makecell{\textbf{ASN}}}  & \multicolumn{4}{c}{\makecell{\textbf{UDAGCN}}} \\
\cmidrule{2-10} 
&\textbf{Selection Method} $\downarrow$\!\!& $\tau=10\%$ & $\tau=20\%$ & $\tau=50\%$ & Full & $\tau=10\%$ & $\tau=20\%$ & $\tau=50\%$ & Full \\
\midrule
\multirow{3}{*}{\textsc{IMDB-BINARY}}
& Random &       
0.613\scriptsize{${\pm}$ 0.110} & 0.568\scriptsize{${\pm}$ 0.078} & 0.515\scriptsize{${\pm}$ 0.024} &

\multirow{3}{*}{0.633\scriptsize{${\pm}$ 0.054}}
&    
0.507\scriptsize{${\pm}$ 0.077} & 0.467\scriptsize{${\pm}$ 0.029} & 0.605\scriptsize{${\pm}$ 0.067} &

\multirow{3}{*}{0.688\scriptsize{${\pm}$ 0.049}}\\

& LAVA  &       
\underline{0.817}\scriptsize{${\pm}$ 0.012} & 
\underline{0.810}\scriptsize{${\pm}$ 0.012} & 
\textbf{0.84}7\scriptsize{${\pm}$ 0.017} &

&    
0.817\scriptsize{${\pm}$ 0.012} &
0.811\scriptsize{${\pm}$ 0.005} &
\textbf{0.837}\scriptsize{${\pm}$ 0.017} &

   \\

\cmidrule{2-5} \cmidrule{7-9}
&\cellcolor[HTML]{D3D3D3}\methodname   
&\cellcolor[HTML]{D3D3D3}\textbf{0.825}\scriptsize{${\pm}$ 0.007} 
&\cellcolor[HTML]{D3D3D3}\textbf{0.819}\scriptsize{${\pm}$ 0.024}
&\cellcolor[HTML]{D3D3D3}\underline{0.834}\scriptsize{${\pm}$ 0.012} & 

&\cellcolor[HTML]{D3D3D3}\textbf{0.840}\scriptsize{${\pm}$ 0.008}
&\cellcolor[HTML]{D3D3D3}\textbf{0.831}\scriptsize{${\pm}$ 0.009} 
&\cellcolor[HTML]{D3D3D3}\underline{0.816}\scriptsize{${\pm}$ 0.016} & 
\\
\midrule
\multirow{3}{*}{\textsc{IMDB-MULTI}}
& Random &       
0.126\scriptsize{${\pm}$ 0.013} & 0.101\scriptsize{${\pm}$ 0.058} & 0.156\scriptsize{${\pm}$ 0.039} & 

\multirow{3}{*}{0.372\scriptsize{${\pm}$ 0.009} }
&    
0.340\scriptsize{${\pm}$ 0.080} & 0.306\scriptsize{${\pm}$ 0.019} & 0.307\scriptsize{${\pm}$ 0.033} & 

\multirow{3}{*}{0.392\scriptsize{${\pm}$ 0.046}}\\

& LAVA  &       
\underline{0.379}\scriptsize{${\pm}$ 0.050} & 
\underline{0.445}\scriptsize{${\pm}$ 0.057} & 
\textbf{0.593}\scriptsize{${\pm}$ 0.004} & 

&    
\underline{0.348}\scriptsize{${\pm}$ 0.051} & 
\underline{0.387}\scriptsize{${\pm}$ 0.093} & 
\textbf{0.519}\scriptsize{${\pm}$ 0.120} & 
\\

\cmidrule{2-5} \cmidrule{7-9}
&\cellcolor[HTML]{D3D3D3}\methodname   
&\cellcolor[HTML]{D3D3D3}\textbf{0.425}\scriptsize{${\pm}$ 0.015}
&\cellcolor[HTML]{D3D3D3}\textbf{0.455}\scriptsize{${\pm}$ 0.097} 
&\cellcolor[HTML]{D3D3D3}\underline{0.577}\scriptsize{${\pm}$ 0.006} & 

&\cellcolor[HTML]{D3D3D3}\textbf{0.390}\scriptsize{${\pm}$ 0.055}  
&\cellcolor[HTML]{D3D3D3}\textbf{0.444}\scriptsize{${\pm}$ 0.089} 
&\cellcolor[HTML]{D3D3D3}\underline{0.451}\scriptsize{${\pm}$ 0.145} & 
\\
\midrule
\multirow{3}{*}{\textsc{MSRC\_21}}
& Random &       
0.481\scriptsize{${\pm}$ 0.071} & 0.277\scriptsize{${\pm}$ 0.039} & 0.556\scriptsize{${\pm}$ 0.012} & 

\multirow{3}{*}{0.734\scriptsize{${\pm}$ 0.015} }
&    
0.151\scriptsize{${\pm}$ 0.072} & 0.204\scriptsize{${\pm}$ 0.065} & 0.209\scriptsize{${\pm}$ 0.062} & 

\multirow{3}{*}{0.260\scriptsize{${\pm}$ 0.049}}\\

& LAVA  &       
\underline{0.661}\scriptsize{${\pm}$ 0.027} & 
\underline{0.779}\scriptsize{${\pm}$ 0.039} & 
\underline{0.867}\scriptsize{${\pm}$ 0.017} & 

&    
\underline{0.435}\scriptsize{${\pm}$ 0.024} & 
\textbf{0.498}\scriptsize{${\pm}$ 0.090} & 
\underline{0.563}\scriptsize{${\pm}$ 0.099} & 

\\

\cmidrule{2-5} \cmidrule{7-9}
&\cellcolor[HTML]{D3D3D3}\methodname   
&\cellcolor[HTML]{D3D3D3}\textbf{0.686}\scriptsize{${\pm}$ 0.022}  
&\cellcolor[HTML]{D3D3D3}\textbf{0.796}\scriptsize{${\pm}$ 0.020} 
&\cellcolor[HTML]{D3D3D3}\textbf{0.868}\scriptsize{${\pm}$ 0.034} & 

&\cellcolor[HTML]{D3D3D3}\textbf{0.465}\scriptsize{${\pm}$ 0.051} 
&\cellcolor[HTML]{D3D3D3}\underline{0.470}\scriptsize{${\pm}$ 0.097}
&\cellcolor[HTML]{D3D3D3}\textbf{0.616}\scriptsize{${\pm}$ 0.055} & 

\\

\midrule
\multirow{3}{*}{\texttt{ogbg-molbace}}
& Random &       
0.465\scriptsize{${\pm}$ 0.048} & 0.440\scriptsize{${\pm}$ 0.049} & 0.466\scriptsize{${\pm}$ 0.060} &

\multirow{3}{*}{0.523\scriptsize{${\pm}$ 0.091}}
&    
0.485\scriptsize{${\pm}$ 0.017} & 0.503\scriptsize{${\pm}$ 0.044} & 0.544\scriptsize{${\pm}$ 0.011} &

\multirow{3}{*}{0.448\scriptsize{${\pm}$ 0.020}}\\

& LAVA  &       
\underline{0.496}\scriptsize{${\pm}$ 0.077} & 
\underline{0.560}\scriptsize{${\pm}$ 0.032} & 
\textbf{0.596}\scriptsize{${\pm}$ 0.052} & 

&    
\underline{0.499}\scriptsize{${\pm}$ 0.036} & 
\textbf{0.553}\scriptsize{${\pm}$ 0.041} & 
\underline{0.517}\scriptsize{${\pm}$ 0.012} &

   \\

\cmidrule{2-5} \cmidrule{7-9}
&\cellcolor[HTML]{D3D3D3}\methodname   
&\cellcolor[HTML]{D3D3D3}\textbf{0.565}\scriptsize{${\pm}$ 0.073}
&\cellcolor[HTML]{D3D3D3}\textbf{0.596}\scriptsize{${\pm}$ 0.053}
&\cellcolor[HTML]{D3D3D3}\underline{0.546}\scriptsize{${\pm}$ 0.023} & 

&\cellcolor[HTML]{D3D3D3}\textbf{0.521}\scriptsize{${\pm}$ 0.002}
&\cellcolor[HTML]{D3D3D3}\underline{0.519}\scriptsize{${\pm}$ 0.022}
&\cellcolor[HTML]{D3D3D3}\textbf{0.555}\scriptsize{${\pm}$ 0.024} & 
\\

\midrule
\multirow{3}{*}{\texttt{ogbg-molbbbp}}
& Random &       
0.537\scriptsize{${\pm}$ 0.091} & 0.530\scriptsize{${\pm}$ 0.076} & 0.545\scriptsize{${\pm}$ 0.062} & 

\multirow{3}{*}{0.616\scriptsize{${\pm}$ 0.042}}
&    
0.549\scriptsize{${\pm}$ 0.031} & 0.568\scriptsize{${\pm}$ 0.043} & 0.536\scriptsize{${\pm}$ 0.008} & 

\multirow{3}{*}{0.513\scriptsize{${\pm}$ 0.024} }\\

& LAVA  &       
\underline{0.606}\scriptsize{${\pm}$ 0.024} & 
\underline{0.635}\scriptsize{${\pm}$ 0.008} & 
\underline{0.646}\scriptsize{${\pm}$ 0.000} & 

&    
\underline{0.655}\scriptsize{${\pm}$ 0.005} & 
\underline{0.649}\scriptsize{${\pm}$ 0.003} & 
\underline{0.673}\scriptsize{${\pm}$ 0.011}&

   \\

\cmidrule{2-5} \cmidrule{7-9}
&\cellcolor[HTML]{D3D3D3}\methodname   
&\cellcolor[HTML]{D3D3D3}\textbf{0.621}\scriptsize{${\pm}$ 0.016}
&\cellcolor[HTML]{D3D3D3}\textbf{0.640}\scriptsize{${\pm}$ 0.019}
&\cellcolor[HTML]{D3D3D3}\textbf{0.650}\scriptsize{${\pm}$ 0.017}& 

&\cellcolor[HTML]{D3D3D3}\textbf{0.660}\scriptsize{${\pm}$ 0.008}
&\cellcolor[HTML]{D3D3D3}\textbf{0.674}\scriptsize{${\pm}$ 0.011} 
&\cellcolor[HTML]{D3D3D3}\textbf{0.677}\scriptsize{${\pm}$ 0.027} & 

\\

\midrule
\multirow{3}{*}{\texttt{ogbg-molhiv}}
& Random &       
0.385\scriptsize{${\pm}$ 0.023} & \underline{0.459}\scriptsize{${\pm}$ 0.086} & 0.397\scriptsize{${\pm}$ 0.070} & 

\multirow{3}{*}{0.519\scriptsize{${\pm}$ 0.077} }
&    
\textbf{0.446}\scriptsize{${\pm}$ 0.041} & 0.412\scriptsize{${\pm}$ 0.021} & 0.409\scriptsize{${\pm}$ 0.014} & 

\multirow{3}{*}{0.439\scriptsize{${\pm}$ 0.034}}\\

& LAVA  &       
\textbf{0.449}\scriptsize{${\pm}$ 0.058} & 
\textbf{0.465}\scriptsize{${\pm}$ 0.088}  & 
\underline{0.399}\scriptsize{${\pm}$ 0.074} & 

&    
0.426\scriptsize{${\pm}$ 0.021} & 
\underline{0.431}\scriptsize{${\pm}$ 0.042}  & 
\textbf{0.433}\scriptsize{${\pm}$ 0.017} &

   \\

\cmidrule{2-5} \cmidrule{7-9}
&\cellcolor[HTML]{D3D3D3}\methodname   
&\cellcolor[HTML]{D3D3D3}\underline{0.435}\scriptsize{${\pm}$ 0.044}
&\cellcolor[HTML]{D3D3D3}0.423\scriptsize{${\pm}$ 0.096}
&\cellcolor[HTML]{D3D3D3}\textbf{0.474}\scriptsize{${\pm}$ 0.094} & 

&\cellcolor[HTML]{D3D3D3}\underline{0.433}\scriptsize{${\pm}$ 0.020}
&\cellcolor[HTML]{D3D3D3}\textbf{0.434}\scriptsize{${\pm}$ 0.008} 
&\cellcolor[HTML]{D3D3D3}\underline{0.395}\scriptsize{${\pm}$ 0.015} & 

\\
\midrule
\bottomrule
\end{tabular}
}
\caption{Performance comparison across combinations of GDA methods and data selection methods for \textit{graph size} shift. We use \textbf{bold}/\underline{underline} to indicate the 1st/2nd best results. \methodname\ achieves the best performance in most settings.} 
\label{appendix:table-graph-enhancer-size}
\end{table*}


\subsection{Validation-label-free setting}
\label{appendix:exp-label-free}

While we originally consider $c$ as tunable parameter, we acknowledge that the existence of validation labels will be implicitly required when $c \neq 0$, which might not be practical under some real-world scenarios. Here, we pick two GNN backbones (i.e. GCN~\cite{kipf2016semi} and GIN~\cite{xu2018powerful}) evaluate the effectiveness of \methodname\ with $c=0$ (i.e. validation-label-free). The result is shown in Table~\ref{appendix:table-label-free}. For simplicity, we only report relative performance improvement (\%) of \methodname\ over the strongest baseline under all settings. We can observe that the advantage of \methodname\ remains comparable to what is reported in the main text.

\begin{table}[t]
\centering
\begin{tabular}{ccccccc}
\toprule
\midrule
Dataset &
\makecell{GCN \\ ($\tau$=10\%)} &
\makecell{GCN \\ ($\tau$=20\%)} &
\makecell{GCN \\ ($\tau$=50\%)} &
\makecell{GIN \\ ($\tau$=10\%)} &
\makecell{GIN \\ ($\tau$=20\%)} &
\makecell{GIN \\ ($\tau$=50\%)} \\
\midrule
\textsc{IMDB-BINARY}      & +11.13 & +9.92  & +2.76  & +4.63  & +5.16  & +13.37 \\
\textsc{IMDB-MULTI}       & +221.31 & +201.09 & +189.61 & +45.78 & +95.01 & +87.04 \\
\textsc{MSRC\_21}         & +1.99  & +1.36  & +3.01  & +2.36  & +0.09  & +2.88  \\
\texttt{ogbg-molbace}     & +1.32  & -4.39  & +13.17 & -16.32 & +2.84  & +7.48  \\
\texttt{ogbg-molbbbp}     & +0.00  & -14.46 & +0.00  & -0.08  & +2.18  & +2.22  \\
\texttt{ogbg-molhiv}      & +3.05  & +0.75  & +0.48  & -1.24  & +0.35  & +2.52  \\
\midrule
\bottomrule
\end{tabular}
\vspace{0.2em}
\caption{Relative improvement (\%) of \methodname\ over the strongest baseline under different settings.}
\label{appendix:table-label-free}
\end{table}

\subsection{Additional backbones}
\label{appendix:exp-additional-gnn}

In addition to typical GNN backbones, we also conduct experiments on a wider range of graph algorithms, including SGFormer~\cite{wu2023improving} and APPNP~\cite{gasteiger2018combining}. In Table~\ref{appendix:table-more-gnn}, we report relative performance improvement (\%) of \methodname\ over the strongest baseline under all settings. We can observe that \methodname\ still outperforms other baselines mostly.

\begin{table}[t]
\centering
\renewcommand{\arraystretch}{1.05}
\begin{tabular}{ccccccc}
\toprule
\midrule
Dataset &
\makecell{SGFormer \\ ($\tau$=10\%)} &
\makecell{SGFormer \\ ($\tau$=20\%)} &
\makecell{SGFormer \\ ($\tau$=50\%)} &
\makecell{APPNP \\ ($\tau$=10\%)} &
\makecell{APPNP \\ ($\tau$=20\%)} &
\makecell{APPNP \\ ($\tau$=50\%)} \\
\midrule
\textsc{IMDB-BINARY}   & +9.82  & +0.59  & -0.34  & +7.28  & +3.01  & -0.35  \\
\textsc{IMDB-MULTI}    & +6.46  & +57.40 & +1.65  & -13.46 & -2.59  & +13.71 \\
\textsc{MSRC\_21}      & +14.26 & +3.58  & +2.13  & +9.54  & +3.18  & +1.27  \\
\texttt{ogbg-molbace}  & -9.66  & +0.05  & +1.90  & +4.03  & -2.31  & +2.84  \\
\texttt{ogbg-molbbbp}  & +2.93  & +5.47  & -3.71  & +22.80 & +6.33  & -4.58  \\
\texttt{ogbg-molhiv}   & +1.51  & -1.79  & +1.65  & -12.00 & +6.07  & +1.88  \\
\midrule
\bottomrule
\end{tabular}
\vspace{0.2em}
\caption{Relative improvement (\%) of \methodname\ over the strongest baseline under different settings.}
\label{appendix:table-more-gnn}
\end{table}

\subsection{Additional GDA methods}
\label{appendix:exp-additional-gda}
We add additional experiments based on two variants of A2GNN~\cite{liu2024rethinking} with different losses and TDSS~\cite{chen2025smoothness}. In Table~\ref{appendix:table-more-gda}, we report relative performance improvement (\%) of \methodname\ over the strongest baseline under all settings. For simplicity, we report results with selection ratio equals to $20\%$. It can be observed that \methodname\ consistently provides the most significant enhancements for the three newly evaluated GDA methods compared to other selection baselines.

\begin{table}[t]
\centering
\begin{tabular}{cccc}
\toprule
\midrule
Dataset &
\makecell{\textbf{A2GNN-ADV} \\ ($\tau$=20\%)} &
\makecell{\textbf{A2GNN-MMD} \\ ($\tau$=20\%)} &
\makecell{\textbf{TDSS} \\ ($\tau$=20\%)} \\
\midrule
\textsc{IMDB-BINARY} & +31.12 & +11.40 & +3.12  \\
\textsc{IMDB-MULTI}  & +26.54 & +58.88 & -42.10 \\
\textsc{MSRC\_21}    & +20.63 & -4.40  & +25.09 \\
\texttt{ogbg-molbace} & +1.46  & -2.57  & +2.88  \\
\texttt{ogbg-molbbbp} & +17.34 & +5.85  & -5.71  \\
\texttt{ogbg-molhiv}  & +0.65  & +0.81  & +0.55  \\
\midrule
\bottomrule
\end{tabular}
\vspace{0.2em}
\caption{Relative improvement (\%) of \methodname\ over three additionally added GDA methods under different datasets.}
\label{appendix:table-more-gda}
\end{table}


\section{Datasets}
\label{appendix:dataset-details}
In Table~\ref{table:dataset-details}, we provide details of datasets used in this work. For \textsc{\# Nodes} and \textsc{\# Edges}, we report the mean sizes across all graphs in the dataset.

\begin{table}[ht]
\centering
\scalebox{0.8}{
\begin{tabular}{cccccccc}
\toprule
\midrule
 \textsc{Dataset} & \textsc{\# Graphs} & \textsc{\# Nodes} & \textsc{\# Edges} & \textsc{\#Features} & \textsc{\# Class} & \textsc{Data Source}& \textsc{ License} \\
\midrule
\textsc{IMDB-BINARY} & 1000 & 19.77& 96.53 & None& 2& PyG~\citep{fey2019fast}& MIT License\\
\textsc{IMDB-MULTI} & 1500 & 12.74& 53.88 & None& 3 & PyG~\citep{fey2019fast}& MIT License\\
\textsc{MSRC\_21} & 563 & 77.52& 198.32& None& 20& PyG~\citep{fey2019fast}& MIT License\\
\texttt{ogbg-molbace} & 1513 & 34.08 & 36.85& 9& 2& OGB~\citep{hu2020ogb}& MIT License\\
\texttt{ogbg-molbbbp} & 2039 & 24.06 & 25.95& 9& 2& OGB~\citep{hu2020ogb}& MIT License\\
\texttt{ogbg-molhiv} & 41127 & 25.51 & 27.46 & 9& 2& OGB~\citep{hu2020ogb}& MIT License\\
\midrule
\bottomrule
\end{tabular}
}
\vspace{0.2em}
\caption{Dataset Statistics and Licenses.}
\label{table:dataset-details}
\end{table}

\section{Backbone GNN Settings for Graph Selection Evaluation}
\label{appendix:gnn-settings}
\paragraph{GNN Models.}
We consider four widely used graph neural network architecture, GCN~\citep{kipf2016semi}, GIN~\citep{xu2018powerful}, GAT~\citep{velivckovic2017graph} and GraphSAGE~\citep{hamilton2017inductive}. The detailed model architectures are described as follows: (i) For GCN, we use three GCN layers with number of hidden dimensions equal to 32. ReLU is used between layers and a global mean pooling layer is set as the readout layer to generate graph-level embedding. A dropout layer with probability $p=0.5$ is applied after the GCN layers. Finally, a linear layer with softmax is placed at the end for graph class prediction. (ii) For GIN, we use three-layer GIN with 32 hidden dimensions. We use ReLU between layers and global mean pooling for readout. A dropout layer with probability $0.5$ is placed after GIN layers and finally a linear layer with softmax for prediction. (iii) For GAT, we use two-layer GAT layers with four heads with global mean pooling for readout. A dropout layer with probability $0.5$ is placed after GIN layers and finally a linear layer with softmax for prediction. (iv) For GraphSAGE, we use two GraphSAGE layers with mean aggregation operation. The hidden dimension is set to 32. A dropout layer with probability $p=0.5$ is applied after the GCN layers. Finally, a linear layer with softmax is placed at the end for graph class prediction.

\paragraph{Experiment Details.} We perform all our methods in Python and GNN models are built-in modules of PyTorch Geometric~\cite{fey2019fast}. The learning rate is set to $10^{-2}$ with weight decay $5\cdot 10^{-4}$. We train 200 epochs for datasets \textsc{IMDB-BINARY, IMDB-MULTI, MSRC\_21} and 100 epochs for datasets \texttt{ogbg-molbace, ogbg-molbbbp, ogbg-molhiv}  with early stopping, evaluating the test set on the model checkpoint that achieves the highest validation performance during training. For each combination of data and model, we report the mean and standard deviation of classification performance over 3-5 random trials. For TUDatasets, we use accuracy as the performance metric; for OGB datasets, we use AUCROC as the performance metric. The computation is performed on Linux with an NVIDIA Tesla V100-SXM2-32GB GPU. For graphs without node features, we also follow \citet{zenggraph} that generates degree-specific one-hot features for each node in the graphs.

\section{GDA Method-Specific Settings}
\label{appendix:gda-settings}
We follow the default parameter settings in the code repository of OpenGDA~\cite{shi2023opengda}. We train 200 epochs for datasets \textsc{IMDB-BINARY, IMDB-MULTI, MSRC\_21} and 100 epochs for datasets \texttt{ogbg-molbace, ogbg-molbbbp, ogbg-molhiv}  with early stopping, evaluating the test set on the model checkpoint that achieves the highest validation performance during training.
\begin{itemize}[leftmargin=1em]
\item AdaGCN~\cite{dai2022graph}: We set the learning rate to $10^{-3}$ with regularization coefficient equal to $10^{-4}$. Dropout rate is $0.3$ and $\lambda_b=1, \lambda_{gp}=5$.
\item ASN~\cite{zhang2021adversarial}: We set the learning rate to 
$10^{-3}$ with regularization coefficient equal to $10^{-4}$. The dropout rate is $0.5$. The difference loss coefficient, domain loss coefficient and the reconstruction loss coefficient is set to $10^{-6}, 0.1, 0.5$. 
\item GRADE~\cite{wu2023non}: We set the learning rate to $10^{-3}$ with regularization coefficient equal to $10^{-4}$. Dropout rate is set to $0.1$.
\item UDAGCN~\cite{wu2020unsupervised}: We set the learning rate to $10^{-3}$ with regularization coefficient equal to $10^{-4}$. The domain loss weight equals to 1. 
\end{itemize}

\section{Additional Preliminary: Graph Domain Adaptation (GDA)}
\label{appendix:prelim-gda}

\label{appendix:ssec:graph-da}

Consider a source domain $\mathcal{D}_s = {(\mathcal{G}_i^s, y_i^s)}_{i=1}^{n_s}$ and a target domain $\mathcal{D}_t = {(\mathcal{G}_i^t, y_i^t)}_{i=1}^{n_t}$, where each $\mathcal{G} = (\mathbf{A}, \mathbf{X})$ represents an attributed graph with the adjacency matrix 
$\mathbf{A} \in \mathbb{R}^{n \times n}$ and the node feature matrix $\mathbf{X} \in \mathbb{R}^{n \times d}$, where $n$ is the number of nodes and $d$ is the dimension of node features. With a shared label set $\mathcal{Y}$, the graphs in the source domain are labeled with $y_i^s \in \mathcal{Y}$.
The two domains are drawn from shifted joint distributions of graph and label space, i.e., $P_s(\mathcal{G}, y) \neq P_t(\mathcal{G}, y)$.
The goal of GDA is to learn a classifier $f: \mathcal{G} \rightarrow \mathcal{Y}$ with the  source domain data that minimizes the expected risk on the target domain:
$\mathbb{E}_{(\mathcal{G}, y) \sim P_t}[\mathcal{L}(f(\mathcal{G}), y)]$,
where $\mathcal{L}$ is a task-specific loss function.

\section{Additional Related Work}
\label{appendix:additional-related-work}

In the era of big data and AI~\citep{liu2024class,liu2023topological,qiu2025saffron,qiu2025graph,qiu2025ask,qiu2025efficient,qiu2022dimes,xu2024discrete,yoo2024ensuring,yoo2025embracing,yoo2025generalizable,xu2024language,li2023metadata}, efficient data utilization has become paramount 
for scalable machine learning. Beyond the data selection and domain adaptation 
methods discussed in the main text, our work also connects to several related 
research directions that provide complementary perspectives on graph learning 
and distribution matching. Specifically, there is a rich body of work at the intersection of optimal transport 
theory and graph data~\citep{jain2025subsampling,
yu2025joint,zeng2023parrot, zeng2023generative, zeng2024hierarchical, zeng2025pave}, which are closely related to our main methodology.

\section{Empirical Runtime of \methodname}
\label{appendix:empirical-runtime}
In Table~\ref{table:empirical-runtime-1} and Table~\ref{table:empirical-runtime-2}, we provide the empirical runtime on  datasets (\textsc{IMDB-BINARY}, \textsc{IMDB-MULTI} and \textsc{MSRC\_21}) and datasets (\texttt{ogbg-molbace}, \texttt{ogbg-molbbbp} and \texttt{ogbg-molhiv}), respectively. We observe that the on-line runtime is insignificant compared to typical GNN training time. And the off-line computation is only run once, which can be pre-computed. Furthermore, we can achieve much better accuracy compared to LAVA with nearly no additional runtime.

\begin{table}[ht]
\centering
\scalebox{0.85}{
\begin{tabular}{ccccc}
\toprule
\midrule
 & \textbf{Procedure / Dataset} & \textsc{IMDB-BINARY} & \textsc{IMDB-MULTI} & \textsc{MSRC\_21} \\

\midrule
\multirow{2}{*}{\makecell{Off-line Computation \\ (GDD Computation)}} 
& FGW Pairwise distance & 7.41 & 9.61 & 18.18 \\
& Label-informed pairwise distance & 0.04 & 0.06 & 0.24 \\
\midrule
\multirow{2}{*}{On-line Computation} 
& \optmethodname\ (Algorithm~\ref{alg:alt-got-shrinkage}) & 0.28 & 0.52 & 0.11 \\
& LAVA & 0.09 & 0.14 & 0.03 \\
\midrule
\multirow{3}{*}{GNN Training Time} 
& GCN  (w/ 10\% data) & 13.45 & 16.36 & 9.59 \\
& GCN  (w/ 20\% data) & 17.64 & 21.40 & 13.85 \\
& GCN  (w/ 50\% data) & 29.92 & 45.82 & 19.57 \\
\midrule
\bottomrule
\end{tabular}
}
\vspace{0.1em}
\caption{\textbf{Empirical run-time behavior for TUDatasets (in seconds)}. We can observe that the off-line procedures can be run comparable to a single GNN training time and the on-line procedure has a negligible runtime compared to GNN training. In addition, we can achieve significantly better performance compared to LAVA with slight additional on-line runtime.}
\label{table:empirical-runtime-1}
\end{table}

\begin{table}[ht]
\centering
\scalebox{0.85}{
\begin{tabular}{ccccc}
\toprule
\midrule

  & \textbf{Procedure / Dataset} & \texttt{ogbg-molbace} & \texttt{ogbg-molbbbp} & \texttt{ogbg-molhiv} \\
\midrule
\multirow{2}{*}{\makecell{Off-line Computation \\ (GDD Computation)}} 
& FGW Pairwise distance & 15.44 & 19.67 & 283.99 \\
& Label-informed pairwise distance & 0.12 & 0.17 & 53.82 \\
\midrule
\multirow{2}{*}{On-line Computation} 
& \optmethodname\ (Algorithm~\ref{alg:alt-got-shrinkage}) & 0.43 & 0.84 & 295.42 \\
& LAVA & 0.06 & 0.13 & 39.45 \\
\midrule
\multirow{3}{*}{GNN Training Time} 
& GCN  (w/ 10\% data) & 12.76 & 13.91 & 190.34 \\
& GCN  (w/ 20\% data) & 13.08 & 22.59 & 312.08 \\
& GCN  (w/ 50\% data) & 22.21 & 30.48 & 845.46 \\
\midrule
\bottomrule
\end{tabular}
}
\vspace{0.1em}
\caption{\textbf{Empirical run-time behavior for OGB datasets (in seconds)}. We can observe that the off-line procedures can be run comparable to a single GNN training time and the on-line procedure has a negligible runtime compared to GNN training. In addition, we can achieve significantly better performance compared to LAVA with slight additional on-line runtime.}
\label{table:empirical-runtime-2}
\end{table}

\vspace{-2em}
\section{Limitations and Outlook}
\label{appendix:limit}

\begin{enumerate}
    \item \textit{How to eliminate the dependence on validation data?} Although it is common in machine learning research to assume we have some validation set that represents the data distribution on the target set (or ``statistically closer'' to the target set), it might not be always available under certain extreme scenarios. Thus, our interesting future direction is to extend our framework to  no-validation-data or test-time adaptation settings.
    \item \textit{Can our proposed method scale to extremely large settings?} When we have millions of \textit{large} graphs in both training and validation set, the efficiency of \methodname\ might be a concern. However, most of the computationally intensive sub-procedure of our method can be done off-line (see complexity analysis in Section~\ref{ssec:graph-data-selection} and empirical runtime in Appendix~\ref{appendix:empirical-runtime}) and the online runtime is ignorable compared to typical GNN training on full dataset. One possible mitigation is to do data  clustering using FGW distance before running our main algorithm \methodname\ to avoid computational overhead.
    \item \textit{How to select the optimal amount of data?} We demonstrate that in Section~\ref{ssec:final-discussion}, the relationship between selection ratio and GNN adaptation performance is not always trivial across different settings. To ease the comparison pipeline, we fix to some target selection ratios (i.e. $10\%,20\%,50\%$) for our main experiments, but we acknowledge that these ratios might not yield the best adaptation performance or serve as the best indicator to comparison across different methods. Thus, one potential extension of our method is to automate the process of selecting the optimal selection ratios when dealing different levels of domain shifts.
    \item \textit{How to extend \methodname\ to node-level graph domain adaption setting?} One possible extension to solve node-level tasks is as follows: decomposing source/target graphs into set of ego-graphs where each of these ego-graphs represent the local topology of each node, and applying \methodname\ directly to select the optimal nodes for adaptation. Future directions that require more investigation include (i) how to decide the size of local vicinity for each ego-graph, (ii) how to co-consider node/edge selection for optimal adaptation and (iii) how to mitigate the information loss when extracting ego-graphs. 
\end{enumerate}

\section{Impact Statement}
\label{appendix:impact}
This paper discusses the advancement of the field of Graph Machine Learning. While there are potential societal consequence of our work, none of which we feel must be hightlighted .

\newpage
\section{ECDF Plots of Different Covariate Shift Settings}
\label{appendix:ecdf-plots}

\begin{figure}[ht]
    \centering
    \subfigure[\textsc{IMDB-BINARY} (Density)]{
        \includegraphics[width=0.45\linewidth]{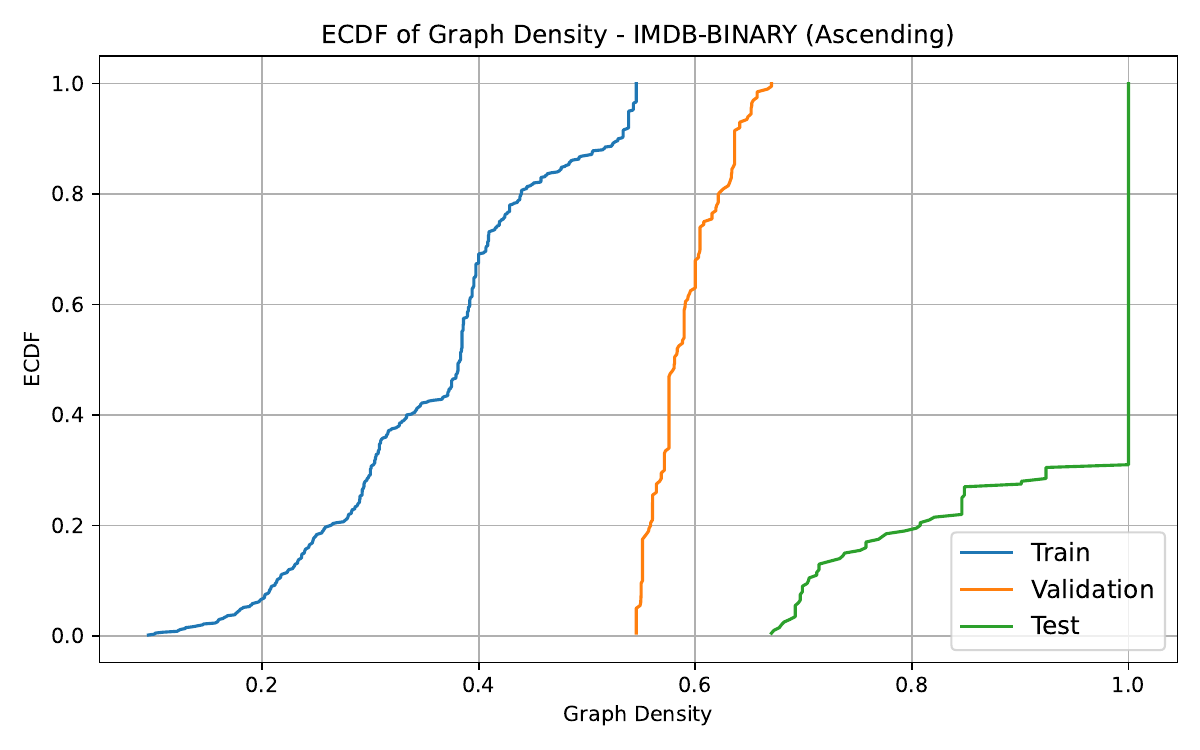}
    }
    \subfigure[\textsc{IMDB-BINARY} (Size)]{
        \includegraphics[width=0.45\linewidth]{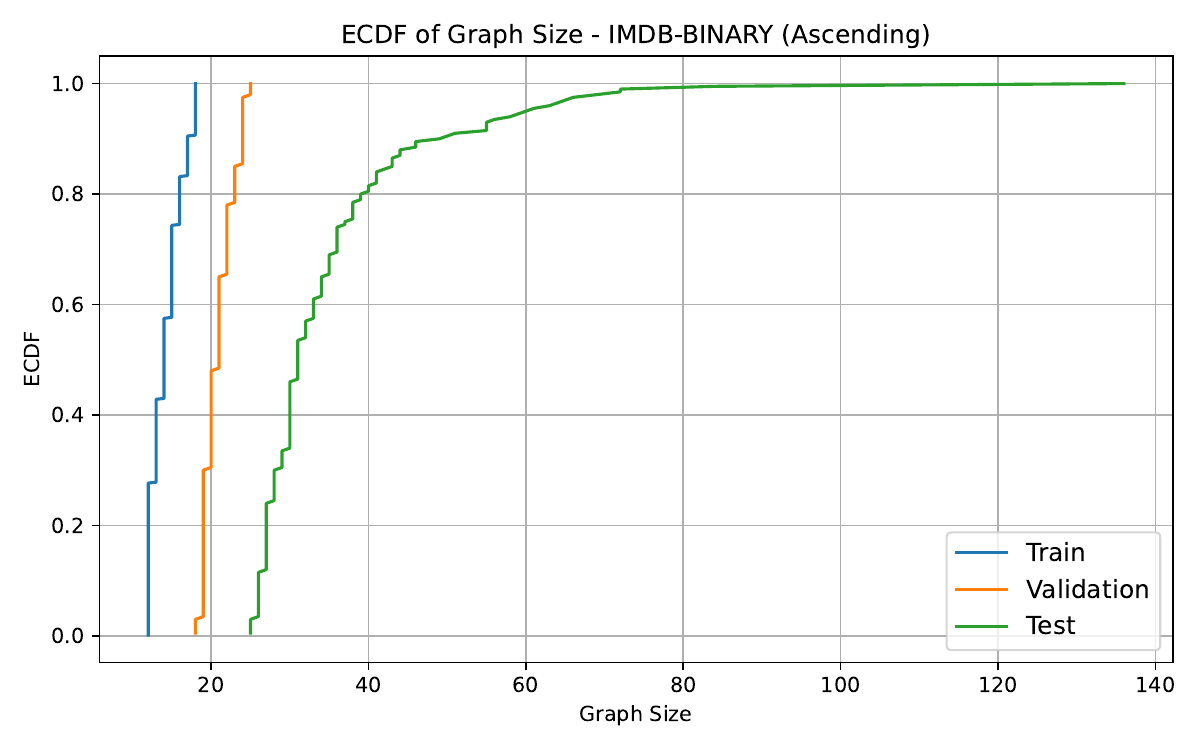}
    }

    \subfigure[\textsc{IMDB-MULTI} (Density)]{
        \includegraphics[width=0.45\linewidth]{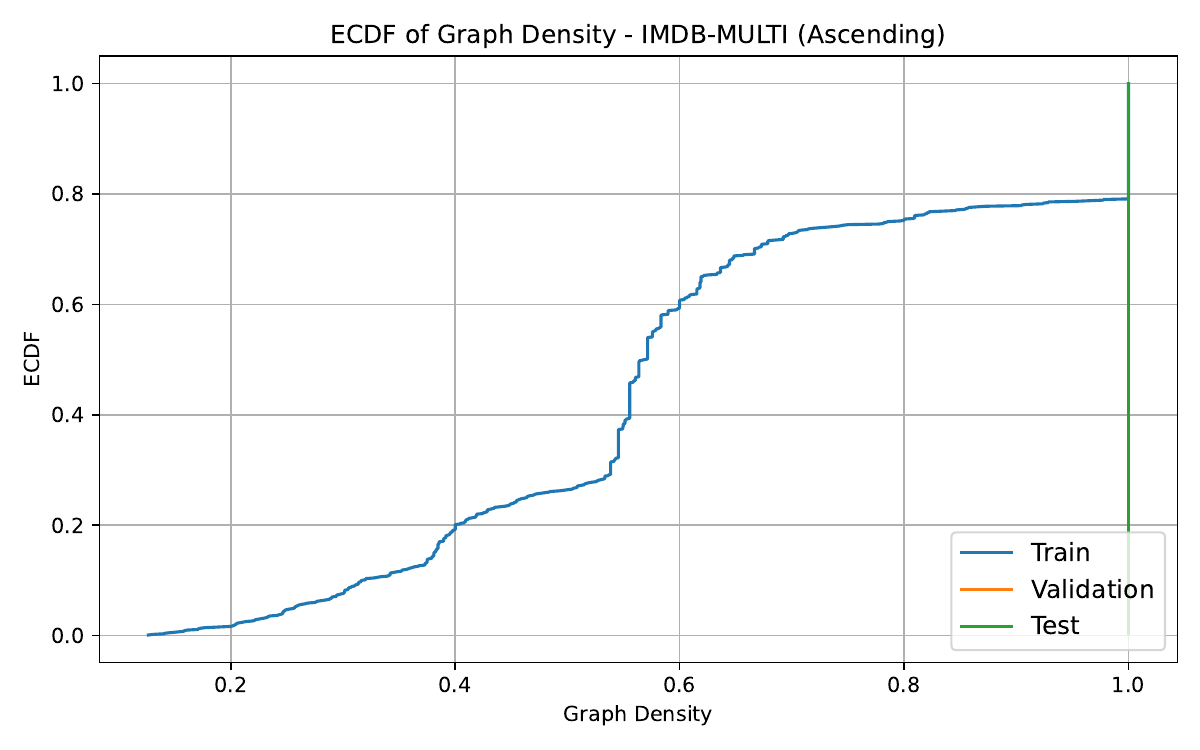}
    }
    \subfigure[\textsc{IMDB-MULTI} (Size)]{
        \includegraphics[width=0.45\linewidth]{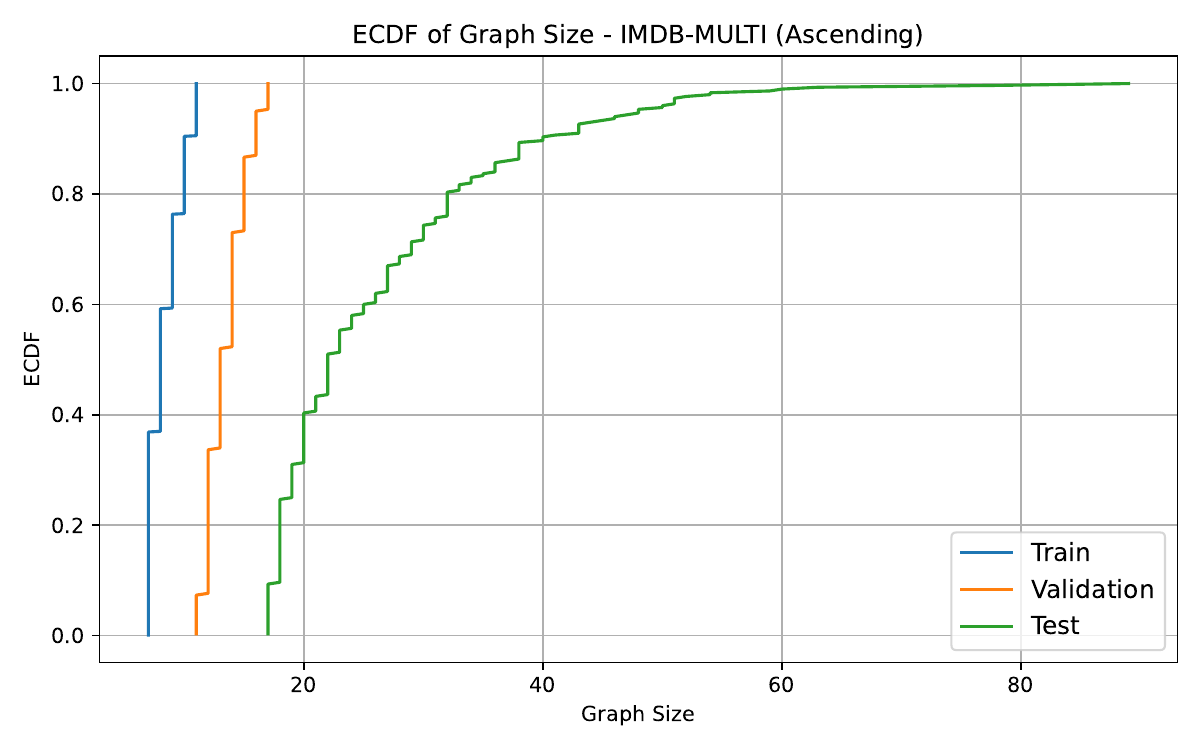}
    }

    \subfigure[\textsc{MSRC\_21} (Density)]{
        \includegraphics[width=0.45\linewidth]{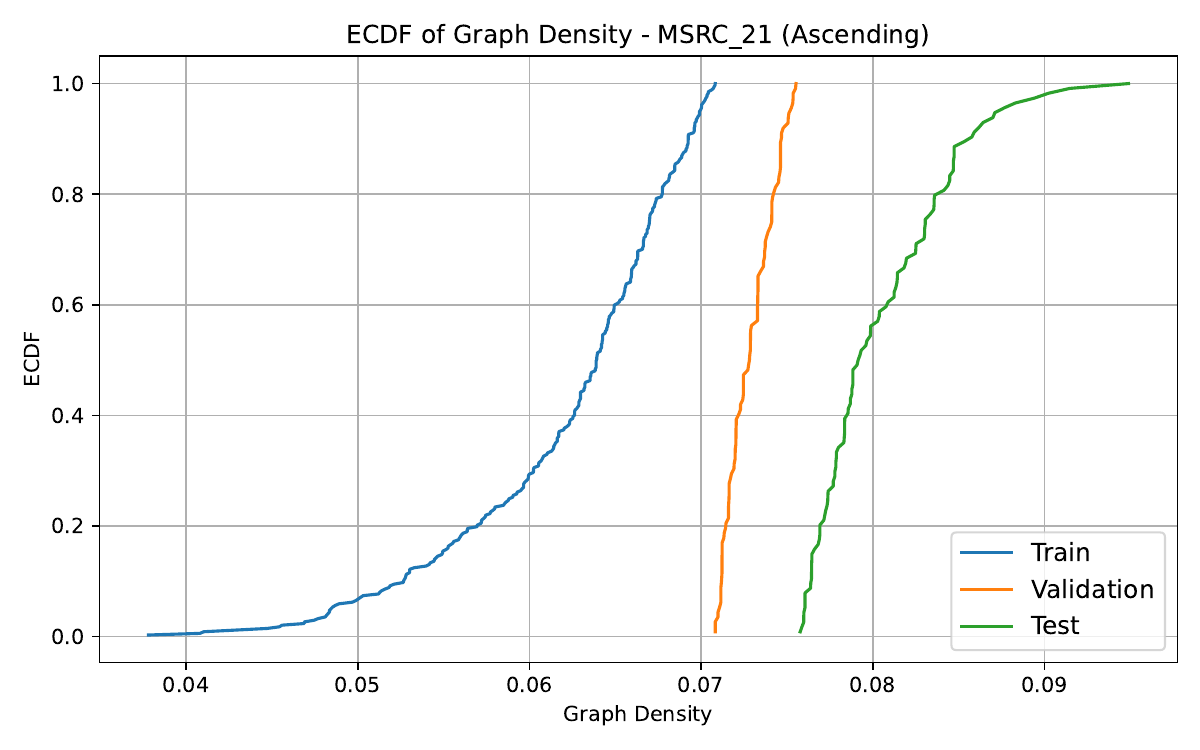}
    }
    \subfigure[\textsc{MSRC\_21} (Size)]{
        \includegraphics[width=0.45\linewidth]{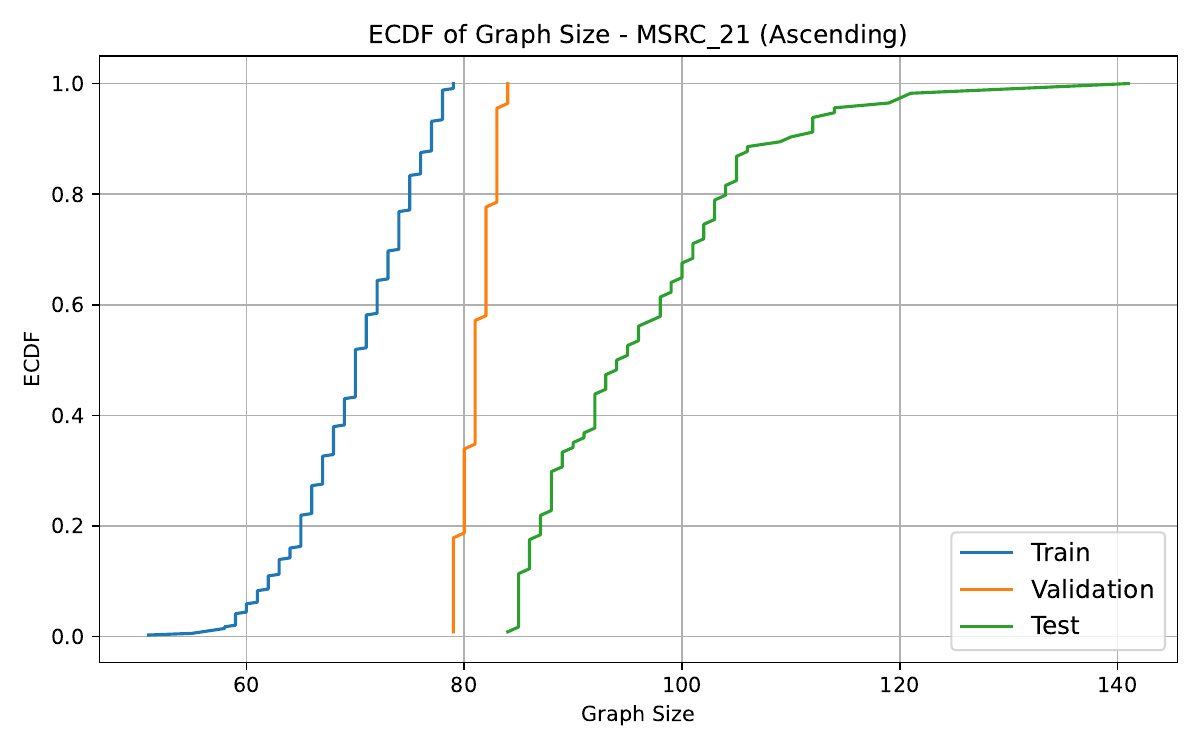}
    }

    \caption{\textbf{ECDF plots of graph density and size for \textsc{IMDB-BINARY}, \textsc{IMDB-MULTI}, and \textsc{MSRC\_21} datasets}. The Blue, Orange, and Green curves represent the distributions of the training, validation, and test splits, respectively. Graphs are sorted in the ascending order by the specified shift (density or size).}
    \label{fig:ecdf_part1}
\end{figure}


\begin{figure}[ht]
    \centering
    \subfigure[\texttt{ogbg-molbace} (Density)]{
        \includegraphics[width=0.45\linewidth]{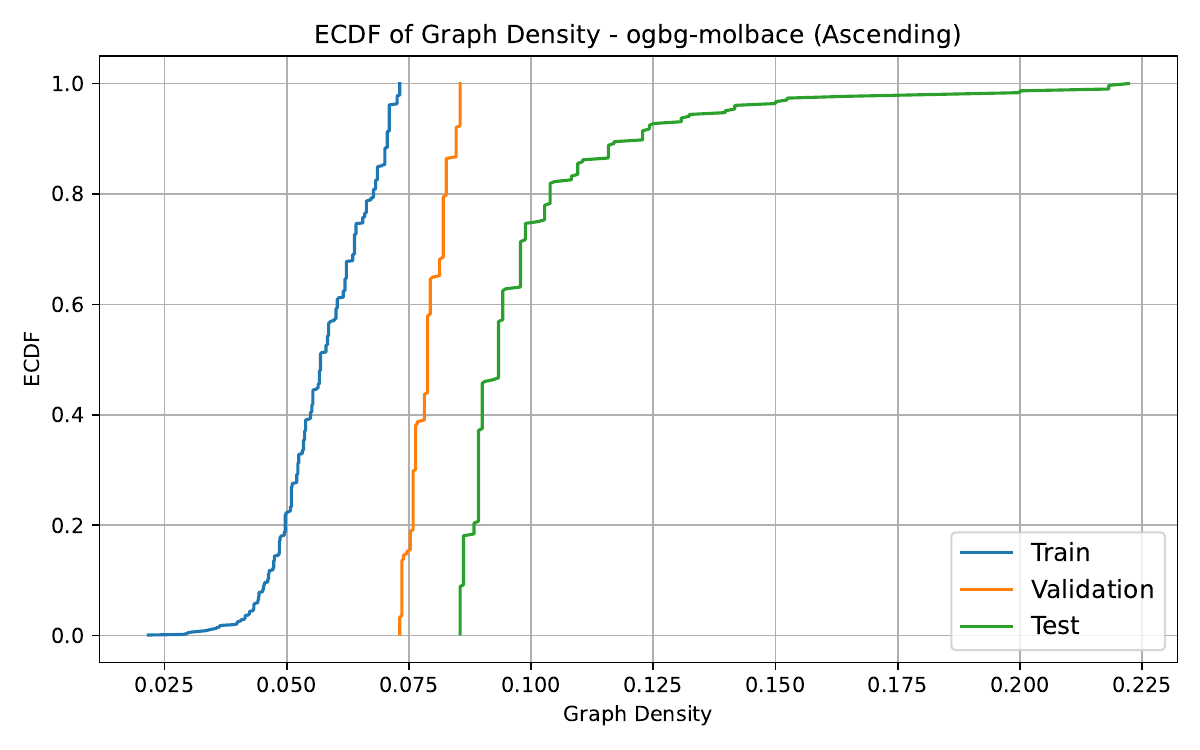}
    }
    \subfigure[\texttt{ogbg-molbace} (Size)]{
        \includegraphics[width=0.45\linewidth]{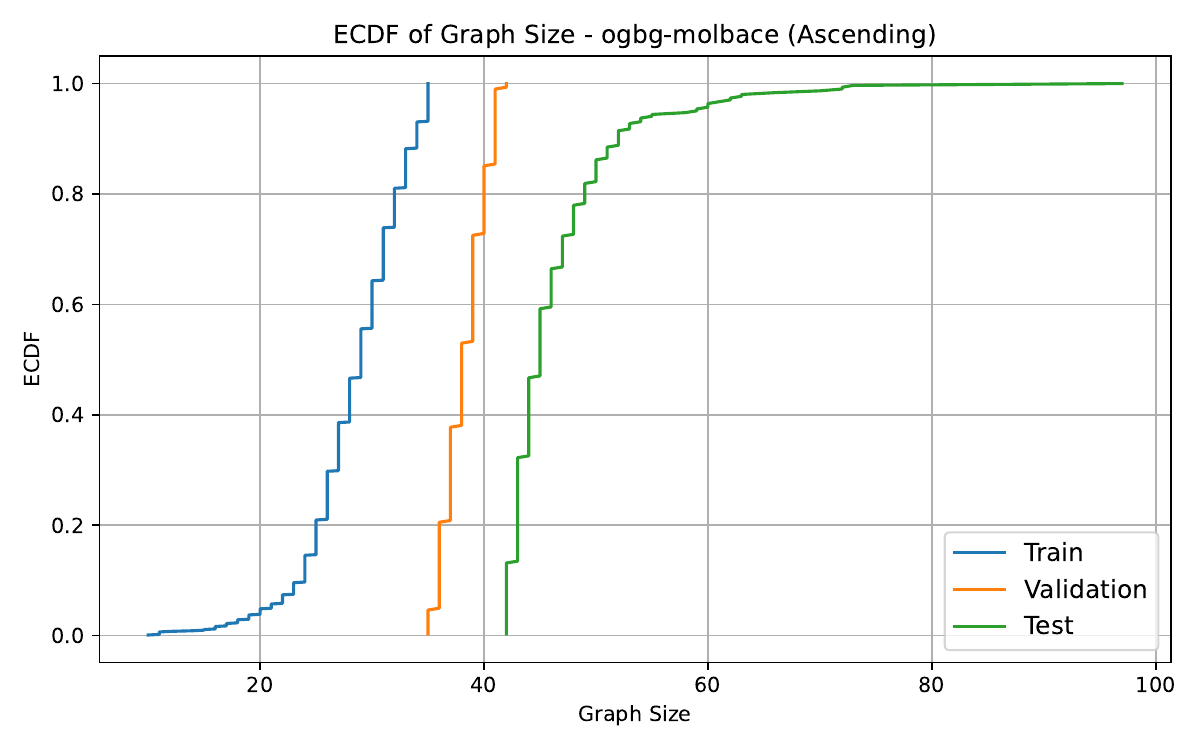}
    }

    \subfigure[\texttt{ogbg-molbbbp} (Density)]{
        \includegraphics[width=0.45\linewidth]{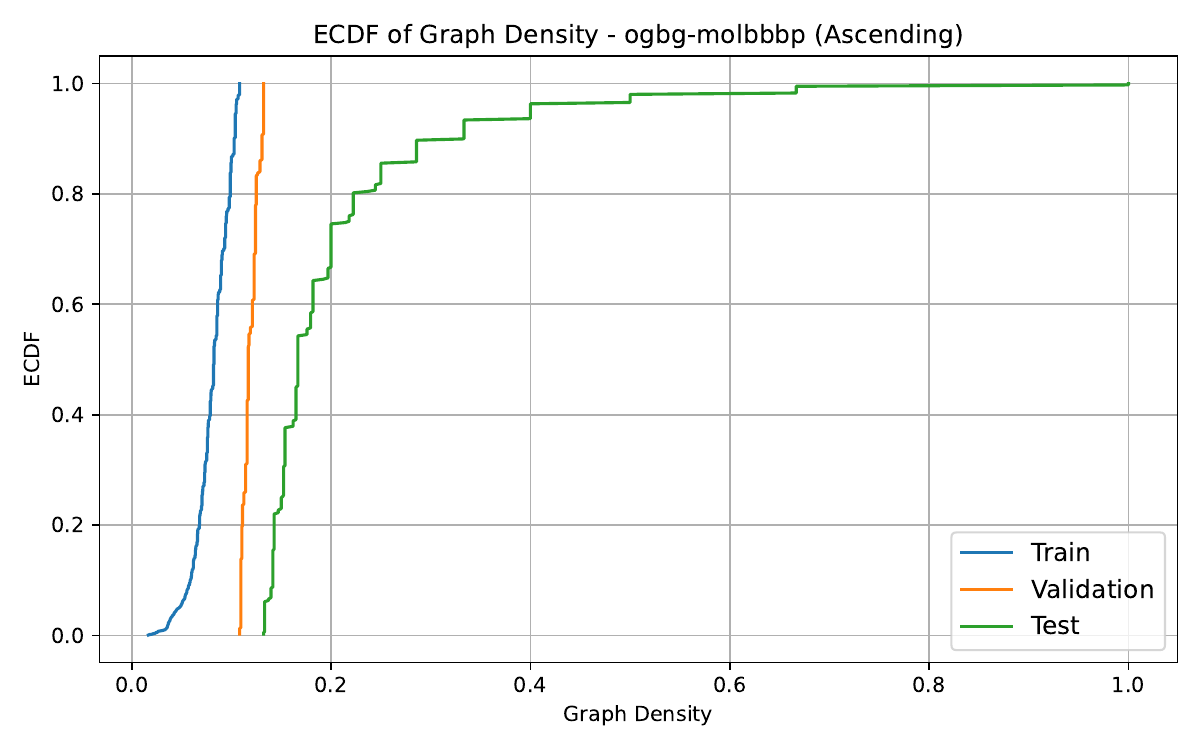}
    }
    \subfigure[\texttt{ogbg-molbbbp} (Size)]{
        \includegraphics[width=0.45\linewidth]{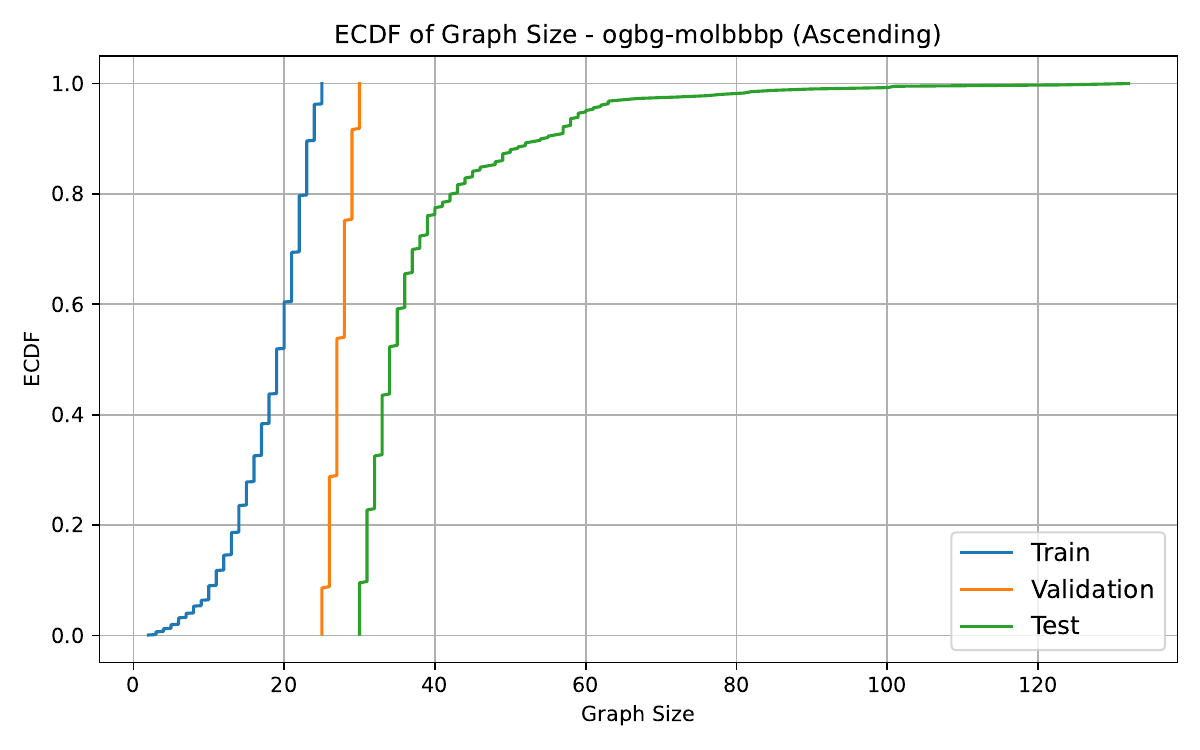}
    }

    \subfigure[\texttt{ogbg-molhiv} (Density)]{
        \includegraphics[width=0.45\linewidth]{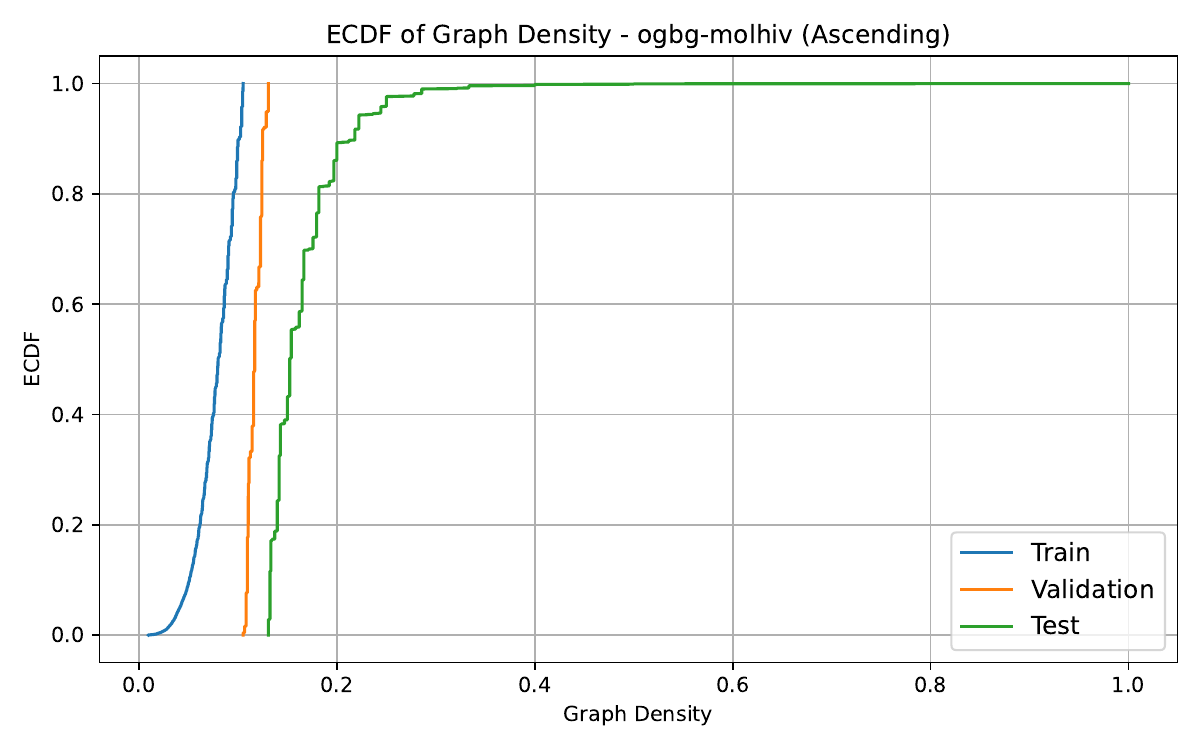}
    }
    \subfigure[\texttt{ogbg-molhiv} (Size)]{
        \includegraphics[width=0.45\linewidth]{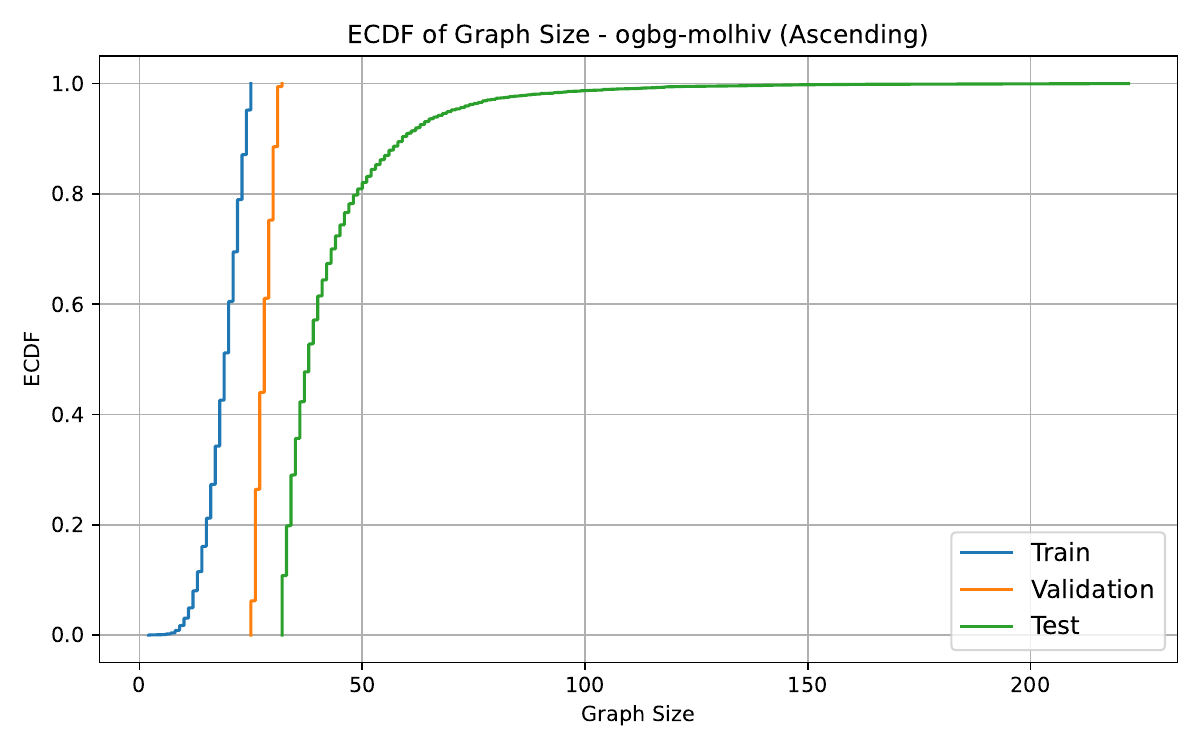}
    }

    \caption{\textbf{ECDF plots of graph density and size for \texttt{ogbg-molbbbp}, \texttt{ogbg-molbace}, and \texttt{ogbg-molhiv} datasets}. The Blue, Orange, and Green curves represent the distributions of the training, validation, and test splits, respectively. Graphs are sorted in ascending order by the specified shift (density or size).}
    \label{fig:ecdf_part2}
\end{figure}

\end{document}